%% file: main.tex
\pdfoutput=1
\documentclass[letterpaper]{article} % DO NOT CHANGE THIS
\usepackage[arxiv]{aaai24}  % DO NOT CHANGE THIS
\usepackage{times}  % DO NOT CHANGE THIS
\usepackage{helvet}  % DO NOT CHANGE THIS
\usepackage{courier}  % DO NOT CHANGE THIS
\usepackage[hyphens]{url}  % DO NOT CHANGE THIS
\usepackage{graphicx} % DO NOT CHANGE THIS
\usepackage{natbib}  % DO NOT CHANGE THIS AND DO NOT ADD ANY OPTIONS TO IT
\usepackage{caption} % DO NOT CHANGE THIS AND DO NOT ADD ANY OPTIONS TO IT
\usepackage[utf8]{inputenc}
\usepackage{acronym}
\usepackage{algorithm}
\usepackage[noend]{algpseudocode}
\algrenewcommand{\algorithmicindent}{1.15em}
\usepackage{subcaption}
\usepackage{tikz}
\usepackage{amsmath}
\usepackage{amssymb}
\usepackage{amsthm}
\usepackage{booktabs}
\usepackage{csquotes}
\usepackage[capitalise]{cleveref}
\usepackage{cuted}

\usetikzlibrary{arrows.meta, positioning, shapes}
\newtheorem{corollary}{Corollary}
\newtheorem{definition}{Definition}
\newtheorem{example}{Example}

\newtheorem{theorem}{Theorem}

\input{./defs.tex}

\title{Colour Passing Revisited: Lifted Model Construction with Commutative Factors\thanks{Extended
version of paper accepted to the Proceedings of the 38th AAAI
Conference on Artificial Intelligence (AAAI-24).}}

\author{%
	Malte Luttermann\textsuperscript{\rm 1,\rm 2},
	Tanya Braun\textsuperscript{\rm 3},
	Ralf Möller\textsuperscript{\rm 1,\rm 2},
	Marcel Gehrke\textsuperscript{\rm 2}
}

\affiliations{%
	\textsuperscript{\rm 1}German Research Center for Artificial Intelligence (DFKI), Lübeck, Germany \\
	\textsuperscript{\rm 2}Institute of Information Systems, University of Lübeck, Germany \\
	\textsuperscript{\rm 3}Data Science Group, University of Münster, Germany \\
	malte.luttermann@dfki.de, tanya.braun@uni-muenster.de, \{moeller, gehrke\}@ifis.uni-luebeck.de
}

\begin{document}

\maketitle

\begin{abstract}
	Lifted probabilistic inference exploits symmetries in a probabilistic model to allow for tractable probabilistic inference with respect to domain sizes.
	To apply lifted inference, a lifted representation has to be obtained, and to do so, the so-called \acl{cp} algorithm is the state of the art.
	The \acl{cp} algorithm, however, is bound to a specific inference algorithm and we found that it ignores commutativity of factors while constructing a lifted representation.
	We contribute a modified version of the \acl{cp} algorithm that uses \aclp{lv} to construct a lifted representation independent of a specific inference algorithm while at the same time exploiting commutativity of factors during an offline-step.
	Our proposed algorithm efficiently detects more symmetries than the state of the art and thereby drastically increases compression, yielding significantly faster online query times for probabilistic inference when the resulting model is applied.
\end{abstract}
\acresetall

\section{Introduction} \label{sec:intro}
\Acp{pfg} combine probabilistic modeling with first-order logic by introducing \acp{lv}, allowing for reasoning under uncertainty about individuals and their relationships.
A fundamental task using \acp{pfg} is to perform probabilistic inference, i.e., to compute marginal distributions of \acp{rv} given observations for other \acp{rv}.
To allow for tractable probabilistic inference (e.g., inference requiring polynomial time) with respect to domain sizes of \acp{lv}, \acp{pfg} with corresponding inference algorithms have been developed, where the main idea is to use a representative for indistinguishable individuals for computations.
To perform lifted inference, a lifted representation has to be constructed first.
In this paper, we study the problem of lifted model construction---that is, we aim to obtain a lifted representation, equivalent to a given propositional (ground) model, which can then be used for lifted probabilistic inference.
In particular, we consider the problem of lifted model construction independent of a specific inference algorithm and under consideration of commutative factors, i.e., factors that map to identical values regardless of the order of their arguments.

Over the last two decades, there has been a considerable amount of research dealing with lifted probabilistic inference using \acp{pfg}.
Lifted inference exploits symmetries in a relational model to compute marginal distributions more efficiently while maintaining exact answers~\citep{Niepert2014a}.
\citet{Poole2003a} first introduces \acp{pfg} and \ac{lve} as an algorithm to perform lifted inference in \acp{pfg}.
\Ac{lve} has then been refined and developed further by many researchers to reach its current form~\citep{DeSalvoBraz2005a,DeSalvoBraz2006a,Milch2008a,Kisynski2009a,Taghipour2013a,Braun2018a}.
Other inference algorithms operating on \acp{pfg} include the \ac{ljt} algorithm, which is designed to handle sets of queries~\citep{Braun2016a}.
The well-known \ac{cp} algorithm (originally named \enquote{CompressFactorGraph}) introduced by \citet{Kersting2009a} builds on the work by \citet{Singla2008a} and is commonly used to construct a lifted representation from a given \ac{fg}.
The \ac{cp} algorithm incorporates a colour passing procedure to detect symmetries in a graph similar to the \acl{wl} algorithm~\citep{Weisfeiler1968a}, which is commonly used to test for graph isomorphism.
Even though the \ac{cp} algorithm is technically able to construct a \ac{pfg}, \ac{cp} in its current form is used as an intermediate step for lifted belief propagation and hence, \ac{cp} is bound to this specific algorithm and does not introduce \acp{lv} to output a valid \ac{pfg}.
Furthermore, the \ac{cp} algorithm does not handle commutative factors (i.e., factors that map to a unique output value regardless of the order of some input values) and is dependent on the order of the factors' argument lists (i.e., identical factors where the arguments of one factor are permuted are not recognised).
Both impose significant limitations for practical applications as taking commutative factors into account and finding identical factors independent of the order of their argument lists result in more compressed models and thus allow for an additional speedup during inference.
Other works establish a connection between automorphism groups and coloured graphs~\citep{Niepert2012a,Bui2013a,Holtzen2020a} by searching for symmetries in a full joint probability distribution without exploiting a factorisation of the distribution.
However, these works do not introduce \acp{lv} and hence, their lifted representation is dependent on a specific inference algorithm as well.

To overcome the limitations of neglecting commutative factors and relying on fixed argument orders to detect identical factors in \ac{cp}, we contribute the \ac{cpr} algorithm which is a modification of \ac{cp} that also handles commutative factors and finds identical factors independent of argument orders, resulting in more compact models while maintaining equivalent model semantics.
In addition, \ac{cpr} is independent of a specific inference algorithm.
\Ac{cpr} uses so-called \acp{crv} to compactly encode commutative factors and we transfer the idea of using histograms from \acp{crv} to allow for order-independent identification of identical factors.
More specifically, we (i) exploit symmetries \emph{within} a factor, and (ii) make use of symmetries \emph{between} factors, where potentials are identical although they have not been recognised as such before.
An additional off\-line step allows us to tackle both (i) and (ii), which contribute to a more compact model to drastically reduce the time needed to perform online inference.
\Acp{crv} are already used in lifted inference algorithms but are not yet used during learning a \ac{pfg} and thus, using \acp{crv} to encode a \ac{pfg} vastly compresses the model.
We also show how \acp{lv} are introduced to obtain a fully-fledged pipeline from propositional \ac{fg} to \ac{pfg} allowing for tractable probabilistic inference with respect to domain sizes in a \ac{pfg} independent of a specific inference algorithm.

The remaining part of this paper is structured as follows.
\Cref{sec:prelim} introduces background information and notations.
Afterwards, in \cref{sec:commutative}, we present solutions to efficiently handle both symmetries within factors and symmetries between factors.
Following these solutions, in \cref{sec:cpr}, we introduce the \ac{cpr} algorithm, which builds on the \ac{cp} algorithm, to transform an input \ac{fg} to a valid \ac{pfg} under consideration of commutative factors and independent of factors' argument orders.
We show how \acp{lv} are introduced by \ac{cpr} to obtain the \ac{pfg} as a lifted representation independent of a specific inference algorithm.
In \cref{sec:eval}, we provide experiments confirming that \ac{cpr} yields significantly faster inference times compared to the state of the art.

\section{Background} \label{sec:prelim}
We begin by recapitulating \acp{fg} as propositional probabilistic models and then move on to define \acp{pfg}.
An \ac{fg} is an undirected graphical model to represent a full joint probability distribution~\citep{Kschischang2001a}.
\begin{definition}[Factor Graph]
	An \emph{\ac{fg}} $G = (\boldsymbol V, \boldsymbol E)$ is a bipartite graph with $\boldsymbol V = \boldsymbol R \cup \boldsymbol F$ where $\boldsymbol R = \{R_1, \ldots, R_n\}$ is a set of variable nodes and $\boldsymbol F = \{f_1, \ldots, f_m\}$ is a set of factor nodes, and there is an edge between a variable node $R$ and a factor node $f$ in $\boldsymbol E \subseteq \boldsymbol R \times \boldsymbol F$ if $R$ appears in the argument list of $f$.
	A factor is a function that maps its arguments to a positive real number (called potential).
	The semantics of an \ac{fg} can be expressed by $P(R_1, \ldots, R_n) = \frac{1}{Z} \prod_{f \in \boldsymbol F} f$ with $Z$ being the normalisation constant.
\end{definition}
\Cref{fig:example_fg} shows a toy example of an \ac{fg} with five variable nodes $ComA$, $ComB$, $Rev$, $SalA$, and $SalB$ and five factor nodes $f_1, \ldots, f_5$.
The \ac{fg} describes the relationships between a company's revenue ($Rev$) and its employee's competences and salaries: There are two employees $Alice$ ($A$) and $Bob$ ($B$), their competences are denoted as $ComA$, $ComB$, respectively, and their salaries are given by $SalA$, $SalB$, respectively.
The input-output pairs of the factors are omitted for brevity.
We next define \acp{pfg}, first introduced by \citet{Poole2003a}, based on the definitions given by \citet{Gehrke2020a}.
\Acp{pfg} combine first-order logic with probabilistic models, using \acp{lv} as parameters in \acp{rv} to represent sets of indistinguishable \acp{rv}, forming \acp{prv}.
\begin{definition}[Logvar, PRV, Event]
	Let $\boldsymbol{R}$ be a set of \ac{rv} names, $\boldsymbol{L}$ a set of \ac{lv} names, $\Phi$ a set of factor names, and $\boldsymbol{D}$ a set of constants.
	All sets are finite.
	Each \ac{lv} $L$ has a domain $\mathcal{D}(L) \subseteq \boldsymbol{D}$.
	A \emph{constraint} is a tuple $(\mathcal{X}, C_{\mathcal{X}})$ of a sequence of \acp{lv} $\mathcal{X} = (X_1, \ldots, X_n)$ and a set $C_{\mathcal{X}} \subseteq \times_{i = 1}^n\mathcal{D}(X_i)$.
	The symbol $\top$ for $C$ marks that no restrictions apply, i.e., $C_{\mathcal{X}} = \times_{i = 1}^n\mathcal{D}(X_i)$.
	A \emph{\ac{prv}} $R(L_1, \ldots, L_n)$, $n \geq 0$, is a syntactical construct of a \ac{rv} $R \in \boldsymbol{R}$ possibly combined with \acp{lv} $L_1, \ldots, L_n \in \boldsymbol{L}$ to represent a set of \acp{rv}.
	If $n = 0$, the \ac{prv} is parameterless and forms a propositional \ac{rv}.
	A \ac{prv} $A$ (or \ac{lv} $L$) under constraint $C$ is given by $A_{|C}$ ($L_{|C}$), respectively.
	We may omit $|\top$ in $A_{|\top}$ or $L_{|\top}$.
	The term $\mathcal{R}(A)$ denotes the possible values (range) of a \ac{prv} $A$. 
	An \emph{event} $A = a$ denotes the occurrence of \ac{prv} $A$ with range value $a \in \mathcal{R}(A)$ and we call a set of events $\boldsymbol E = \{A_1 = a_1, \ldots, A_k = a_k\}$ \emph{evidence}.
\end{definition}
As an example, consider $\boldsymbol{R} = \{Com, Rev, Sal\}$ for competence, revenue, and salary, respectively, and $\boldsymbol{L} = \{E\}$ with $\mathcal{D}(E) = \{Alice, Bob\}$ (employees), combined into Boolean \acp{prv} $Com(E)$, $Rev$, and $Sal(E)$.

A \ac{pf} describes a function, mapping argument values to positive real numbers (called potentials), of which at least one is non-zero.
\begin{figure}[t]
	\centering
	\begin{subfigure}{0.45\linewidth}
		\centering
		\begin{tikzpicture}
			\node[rv, draw] (ca) {$ComA$};
			\node[rv, draw, right = 0.2cm of ca] (cb) {$ComB$};
		
			\factor{above}{ca}{0.3cm}{180}{$f_1$}{f1}
			\factor{above}{cb}{0.3cm}{0}{$f_2$}{f2}
		
			\factor{below left}{cb}{0.4cm and 0.25cm}{90}{$f_3$}{f3}
		
			\node[rv, draw, below = 0.3cm of f3] (rev) {$Rev$};
			\factor{below left}{rev}{0.4cm and 0.45cm}{180}{$f_4$}{f4}
			\factor{below right}{rev}{0.4cm and 0.45cm}{0}{$f_5$}{f5}
		
			\node[rv, draw, below = 0.3cm of f4] (sa) {$SalA$};
			\node[rv, draw, below = 0.3cm of f5] (sb) {$SalB$};
		
			\draw (f1) -- (ca);
			\draw (f2) -- (cb);
		
			\draw (ca) -- (f3);
			\draw (cb) -- (f3);
		
			\draw (f3) -- (rev);
		
			\draw (rev) -- (f4);
			\draw (rev) -- (f5);
			\draw (ca) -- (f4);
			\draw (cb) -- (f5);
		
			\draw (f4) -- (sa);
			\draw (f5) -- (sb);
		\end{tikzpicture}
		\caption{}
		\label{fig:example_fg}
	\end{subfigure}
	\begin{subfigure}{0.5\linewidth}
		\centering
		\begin{tikzpicture}
			\node[rv, inner sep = 1.8pt] (C) {$Com(E)$};
		
			\pfs{above}{C}{0.5cm}{0}{$g_1$}{G1a}{G1}{G1b}
			\factor{below left}{C}{0.6cm and 0.45cm}{180}{$g_2$}{G2}
			\pfs{below right}{C}{0.6cm and 0.45cm}{0}{$g_3$}{G3a}{G3}{G3b}
		
			\node[rv, below = 0.7cm of G2] (R) {$Rev$};
			\node[rv, inner sep = 1.8pt, below = 0.7cm of G3] (S) {$Sal(E)$};
		
			\begin{pgfonlayer}{bg}
				\draw (C) -- (G1);
				\draw (C) -- (G2);
				\draw (C) -- (G3);
				\draw (R) -- (G2);
				\draw (R) -- (G3);
				\draw (S) -- (G3);
			\end{pgfonlayer}
		\end{tikzpicture}
		\caption{}
		\label{fig:example_pm}
		\end{subfigure}
	\caption{(a) An \ac{fg} describing the relationships between a company's revenue and its employee's competences and salaries, (b) a \ac{pfg} corresponding to the lifted representation of the \ac{fg} shown in (a). The mappings of argument values to potentials of the (par)factors are omitted for brevity.}
	\label{fig:example_fg_pm}
\end{figure}
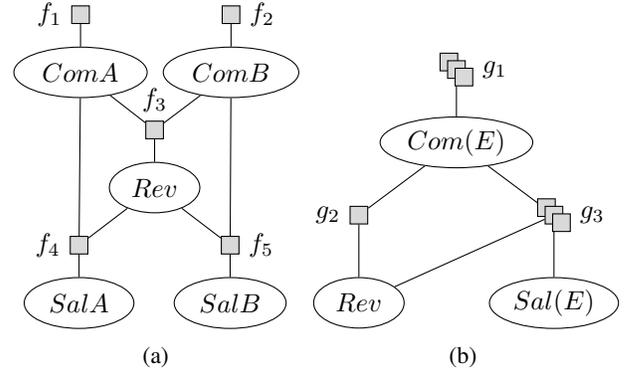
\begin{definition}[Parfactor, Model, Semantics]
	We denote a \emph{\ac{pf}} $g$ by $\phi(\mathcal{A})_{| C}$ with $\mathcal{A} = (A_1, \ldots, A_n)$ a sequence of \acp{prv}, $\phi$$:$ $\times_{i = 1}^n \mathcal{R}(A_i) \mapsto \mathbb{R}^+$ a function with name $\phi \in \Phi$ mapping argument values to a positive real number called \emph{potential}, and $C$ a constraint on the \acp{lv} of $\mathcal{A}$. %, identical for all possible groundings of $\mathcal{A}$ w.r.t.\ $C$.
	We may omit $|\top$ in $\phi(\mathcal{A})_{|\top}$.
	The term $lv(Y)$ refers to the \acp{lv} in some element $Y$, a \ac{prv}, a \ac{pf}, or sets thereof.
	The term $gr(Y_{| C})$ denotes the set of all instances (groundings) of $Y$ w.r.t.\ constraint $C$.
	A set of \acp{pf} $\{g_i\}_{i=1}^n$ forms a \emph{\ac{pfg}} $G$.
	The semantics of $G$ is given by grounding and building a full joint distribution.
	With $Z$ as the normalisation constant, $G$ represents $P_G = \frac{1}{Z} \prod_{f \in gr(G)} f$.
\end{definition}
Before we take a look at an example, we need the concept of a \ac{crv}, introduced by \citet{Milch2008a}.
\begin{definition}[\Ac{crv}]
	$\#_{X}[R(\mathcal{X})]$ denotes a \emph{\ac{crv}}, where $lv(\mathcal{X}) = \{X\}$ (other inputs constant).
	Its range is the space of possible histograms.
	A histogram $h$ is a set of tuples $\{(v_i, n_i)\}_{i = 1}^m$, $v_i \in \mathcal R(R(\mathcal{X}))$, $n_i \in \mathbb{N}$, $m = |\mathcal R(R(\mathcal{X}))|$, and $\sum_i n_i = |gr(X_{| C})|$ for some constraint $C$ over $\mathcal{X}$.
	A shorthand notation is $[n_1, \ldots, n_m]$.
	$h(v_i)$ returns $n_i$.
	Since counting binds \ac{lv} $X$, $lv(\#_{X} [R(\mathcal{X})]) = \mathcal{X} \setminus \{X\}$.
\end{definition}
For example, the two groundings of a Boolean \ac{prv} $R(X)$ with $|\mathcal D(X)| = 2$ can be assigned $\mathrm{true}$ values ($[2,0]$), one $\mathrm{true}$ and one $\mathrm{false}$ value ($[1,1]$), or $\mathrm{false}$ values ($[0,2]$).
\Cref{fig:example_pm} shows a \ac{pfg} $G = \{g_i\}^3_{i=1}$ with $g_1 = \phi_1(Com(E))_{|\top}$, $g_2 = \phi_2(\#_E[Com(E)], Rev)_{|\top}$, and $g_3 = \phi_3(Com(E), Rev, Sal(E))_{|\top}$ where $\phi_2$ contains a \ac{crv} $\#_E[Com(E)]$.
$G$ is a lifted representation of the \ac{fg} shown in \cref{fig:example_fg}.
The definition of a \ac{pfg} also implies that every \ac{fg} is a \ac{pfg} containing only parameterless \acp{rv}.

The state of the art algorithm to transform a (propositional) \ac{fg} into a lifted representation is the \ac{cp} algorithm~\citep{Kersting2009a,Ahmadi2013a}, which we briefly recap in the following.
\ac{cp} tries to find symmetries in an \ac{fg} based on potentials of factors, on ranges and evidence of \acp{rv}, as well as on the graph structure.
Each \ac{rv} is assigned a colour such that \acp{rv} with identical ranges and identical evidence get the same colour, and each factor is assigned a colour such that factors with the same potentials get the same colour.
The colours are then passed from every \ac{rv} to its neighbouring factors and vice versa.
After each colour passing step, colours are reassigned depending on the received colours and a node's own colour.
Factor nodes also include the position of a \ac{rv} in their argument list into their message.
In the end, all \acp{rv} and factors, respectively, are grouped together based on their colours and the procedure is iterated until groupings do not change anymore.
A formal description of \ac{cp} can be found in \cref{appendix:cp_algo}.
In its current form, \ac{cp} does not handle commutative factors and is dependent on the order of the factors' argument lists.
Therefore, we next explore the problem of lifting an \ac{fg} to obtain a \ac{pfg} taking into account commutative factors and afterwards tackle the problem of finding identical factors in an \ac{fg} independent of the order of their argument lists.

\section{Colour Passing Revisited} \label{sec:commutative}
Consider again the example shown in \cref{fig:example_fg_pm}.
Intuitively, one might expect the \ac{cp} algorithm to output the groupings corresponding to the \ac{pfg} shown in \cref{fig:example_pm} if $f_1$ and $f_2$ as well as $f_4$ and $f_5$ share the same potentials.
The \ac{cp} algorithm, however, ends up without grouping anything if it is provided with the \ac{fg} from \cref{fig:example_fg} as input because $f_3$ sends different messages to $ComA$ and $ComB$ due to different positions of $ComA$ and $ComB$ in $f_3$'s argument list.

In the upcoming subsection, we show that in cases where a factor $f$ is commutative, the position of a \ac{rv} in $f$'s argument list is not relevant for the colour passing procedure and hence can be omitted.
Afterwards, we transfer the idea of using histograms from \acp{crv} to efficiently detect identical factors independent of the order of their arguments.
We demonstrate that instead of scanning two tables from top to bottom and comparing their values one by one, we can build a set of potential values for each possible histogram and compare the sets pairwise to guarantee order-independence without introducing additional computational overhead.

\subsection{Symmetries within Factors} \label{sec:intra_factor_symmetries}
A simplified version of the situation regarding $f_3$ in \cref{fig:example_fg} is depicted in \cref{fig:crv_example_fg}. %, where we omit all unnecessary details to focus on the relevant parts.
Here, again, \ac{cp} does not group anything because $A$ and $B$ have different positions in $\phi_1$.
However, in this example, $\phi_1$ encodes a symmetric function (that is, a function returning the same value independent of the order of its arguments) because $\phi_1(\mathrm{true}, \mathrm{false}) = \phi_1(\mathrm{false}, \mathrm{true}) = \varphi_2$ and we could use the symmetries within $\phi_1$ to group $A$ and $B$ using a \ac{crv}, as shown in \cref{fig:crv_example_pm}.
Although lifted inference algorithms such as \ac{lve} use \acp{crv} by count converting \acp{prv} during the inference task~\citep{Taghipour2013a}, \acp{crv} have, to the best of our knowledge, not been used to learn a valid \ac{pfg} so far.
In consequence, we gain a significant speedup for inference algorithms when using \acp{crv} to model a \ac{pfg}.

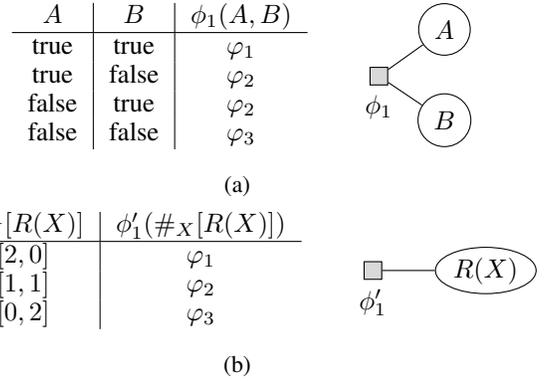
\begin{figure}[t]
	\centering
	\begin{subfigure}[t]{\linewidth}
		\centering
		\begin{tikzpicture}
			\node[circle, draw] (A) {$A$};
			\node[circle, draw] (B) [below = 0.5cm of A] {$B$};
			\factor{below left}{A}{0.25cm and 0.5cm}{270}{$\phi_1$}{f1}
		
			\node[left = 0.7cm of f1] (tab_f1) {
				\begin{tabular}{c|c|c}
					$A$   & $B$   & $\phi_1(A,B)$ \\ \hline
					true  & true  & $\varphi_1$ \\
					true  & false & $\varphi_2$ \\
					false & true  & $\varphi_2$ \\
					false & false & $\varphi_3$ \\
				\end{tabular}
			};
		
			\draw (A) -- (f1);
			\draw (B) -- (f1);
		\end{tikzpicture}
		\caption{}
		\label{fig:crv_example_fg}
	\end{subfigure}

	\begin{subfigure}[t]{\linewidth}
		\centering
		\begin{tikzpicture}
			\node[ellipse, inner sep = 1.2pt, draw] (AB) {$R(X)$};
			\factor{left}{AB}{0.7cm}{270}{$\phi'_1$}{f1}
		
			\node[left = 0.7cm of f1] (tab_f1) {
				\begin{tabular}{c|c}
					$\#_X[R(X)]$  & $\phi'_1(\#_X[R(X)])$ \\ \hline
					$[2,0]$ & $\varphi_1$ \\
					$[1,1]$ & $\varphi_2$ \\
					$[0,2]$ & $\varphi_3$ \\
				\end{tabular}
			};
		
			\begin{pgfonlayer}{bg}
				\draw (AB) -- (f1);
			\end{pgfonlayer}
		\end{tikzpicture}
		\caption{}
		\label{fig:crv_example_pm}
	\end{subfigure}
	\caption{(a) An \ac{fg} containing a commutative factor $\phi_1$, (b) a \ac{pfg} entailing equivalent semantics as the \ac{fg} shown in (a). To obtain the \ac{pfg} in (b), $\phi_1$ is mapped to $\phi'_1$ using a \ac{crv} $\#_X[R(X)]$ counting over a \ac{lv} $X$ with $|\mathcal D(X)| = 2$.}
	\label{fig:crv_example}
\end{figure}

We now explain how to make use of the symmetries within $\phi_1$ in \cref{fig:crv_example_fg} to group together $A$ and $B$.
Even though we choose Boolean \acp{rv} to keep the example small, the idea also applies to \acp{rv} with arbitrary ranges.
For two Boolean \acp{rv} (here $A$ and $B$), there are three possible histograms $[2,0]$, $[1,1]$, and $[0,2]$.
As $\phi_1$ outputs a unique value for all of the three possible histograms, $\phi_1$ is commutative and thus, $\phi_1$ can be represented by a factor $\phi'_1$ taking as input a \ac{crv} that counts over a \ac{lv} $X$ with $\abs{\mathcal D(X)} = 2$.
\Cref{fig:crv_example_pm} visualises the resulting \ac{pfg} using a \ac{crv} to obtain a lifted representation equivalent to the \ac{fg} shown in \cref{fig:crv_example_fg}, where the size of the table has been reduced from exponential to polynomial in the number of arguments of the factor.
$\phi'_1$ now maps each possible histogram to a potential, i.e., $\phi'_1$ outputs $\varphi_1$ for the histogram $[2,0]$, $\varphi_2$ for $[1,1]$, and $\varphi_3$ for $[0,2]$.
These mappings of the histograms capture exactly the semantics of $\phi_1$ from \cref{fig:crv_example_fg}.
Consequently, we use \acp{crv} to compactly encode symmetric functions and thus allow for additional groupings of \acp{rv} and factors.
We incorporate this idea into the \ac{cpr} algorithm by omitting the position of a \ac{rv} in a factor's argument list in a message if the factor is commutative.
To profit from the usage of \acp{crv}, it is even sufficient for a factor to be \emph{partially commutative}, defined as follows. 
\begin{definition}[Partially Commutative Factor] \label{def:partial_commutative_factor}
	A factor $\phi$ with argument list $R_1, \ldots, R_n$ is called \emph{partially commutative} if there exists a non-empty subset $\boldsymbol S \subseteq \{R_1, \ldots, R_n\}$ with $\abs{\boldsymbol S} > 1$ such that $\phi$ is commutative with respect to the subset $\boldsymbol S$, i.e., for all events $r_1, \ldots, r_n \in \times_{i=1}^n \mathcal R(R_i)$ it holds that $\phi(r_1, \ldots, r_n) = \phi(r_{\pi(1)}, \ldots, r_{\pi(n)})$ for all permutations $\pi$ of $\{1, \ldots, n\}$ with $\pi(i) = i$ for all $r_i \notin \boldsymbol S$.
\end{definition}
In many practical settings, factors are partially commutative, for example when individuals are indistinguishable and only the number of individuals having a certain property is of interest (e.g., in the employee example the number of competent employees determines the revenue of the company while it does not matter which specific employees are competent).
To check whether a factor $\phi(R_1, \ldots, R_n)$ is commutative, let w.l.o.g.\ $\{R_1, \ldots, R_m\} \subseteq \{R_1, \ldots, R_n\}$ be a maximal subset of \acp{rv} with $\mathcal R(R_1) = \ldots = \mathcal R(R_m)$ that satisfies the condition given in \cref{def:partial_commutative_factor}.
If no such subset exists, $\phi$ is not commutative, else $\phi$ can be mapped to a new (par)factor $\phi'$ by replacing the \acp{rv} $R_1, \ldots, R_m$ by a \ac{crv} counting over a \ac{lv} $X$ with $|\mathcal D(X)| = m$.
The remaining \acp{rv} occurring in the argument list of $\phi$ but not in $R_1, \ldots, R_m$ are transferred unchanged to the argument list of $\phi'$.
Each combination of histogram and possible values for the remaining \acp{rv} is then mapped to a unique value by $\phi'$.
\begin{example}
	Consider again the factor $f_3$ from \cref{fig:example_fg}, which takes the three arguments $ComA$, $ComB$, and $Rev$ as input.
	Since the structure of $Rev$ differs from the structure of $ComA$ and $ComB$, we do not intend to count over the complete set of input \acp{rv} but only over a subset of the set of input \acp{rv}.
	The subset should be as large as possible to obtain the most compression for our lifted representation.
	Moreover, to group $ComA$ and $ComB$ into $Com(E)$, they have to behave identically with respect to the potentials of $f_3$, that is, counting over them has to result in identical potentials for each histogram.
	In particular, assuming that the order of $f_3$'s arguments is $ComA$, $ComB$, $Rev$, it must hold that $f_3(\mathrm{true}, \mathrm{false}, \mathrm{true}) = f_3(\mathrm{false}, \mathrm{true}, \mathrm{true})$ and $f_3(\mathrm{true}, \mathrm{false}, \mathrm{false}) = f_3(\mathrm{false}, \mathrm{true}, \mathrm{false})$.
\end{example}
For a factor $\phi$ with $n$ arguments $R_1, \ldots, R_n$, the number of possible subsets we could possibly count over is in $O(2^n)$.
The number of candidates can be reduced by only considering subsets $\{R_1, \ldots, R_m\} \subseteq \{R_1, \ldots, R_n\}$ where $R_1, \ldots, R_m$ have the same number of neighbours in the graph because \acp{rv} with different numbers of neighbouring factors receive different messages during colour passing and hence cannot be grouped together.
Even though the number of subsets to check in the worst case remains in $O(2^n)$, the computational effort is often reduced by the optimisation of only considering subsets consisting of \acp{rv} which have the same number of neighbours.
Moreover, as we aim to compress an \ac{fg} by exploiting symmetries, we inherently assume that there are at least some symmetries in the \ac{fg} (otherwise, we would not intend to run \ac{cp}).
Thus, symmetries within factors are likely to be found fast in practical applications, which we also confirm in our experiments.

In the worst case, checking $O(2^n)$ subsets for commutativity is infeasible for large $n$ but we argue that for practical applications, we can assume that $n$ is reasonably small: A factor $\phi(R_1, \ldots, R_n)$ defines $2^n$ mappings (rows in its table) if all \acp{rv} $R_1, \ldots, R_n$ are Boolean and hence, storing the table of input-output pairs requires $O(2^n)$ space (for larger ranges of the $R_i$, there are even more mappings).
Consequently, the table cannot even be stored for large $n$, implying that the number of arguments of each factor is limited to small values in practical applications.
As there are at least as many rows as there are subsets of $\{R_1, \ldots, R_n\}$, the current version of \ac{cp} needs exponential time in $n$ even without checking for commutativity of factors.
Therefore, handling symmetries within factors requires no additional costs while at the same time allowing us to drastically increase compression and thus speed up online inference.
The number of rows is reduced from exponential in $n$ to polynomial in $n$~\citep{Milch2008a} and \acp{rv} as well as factors are grouped together that could not be grouped together before.

So far, we applied the idea of using histograms to find symmetries \emph{within} a factor.
Next, we aim to find symmetries \emph{between} factors independent of the order of their arguments before gathering both ideas into the \ac{cpr} algorithm.
When checking for symmetries \emph{between} multiple factors, the same histograms that are used to check for symmetries \emph{within} factors are computed, allowing us to reuse these histograms without additional computational effort.

\subsection{Symmetries between Factors}
We now deploy histograms to detect symmetries between different factors.
To illustrate this point, have a look at \cref{fig:input_order_fg}.
Considering the factor $\phi_2$ in \cref{fig:input_order_fg}, $C$ has position two and $B$ position one in the argument list of $\phi_2$.
Swapping the positions of $C$ and $B$ in $\phi_2$ results in a table where the positions of $\varphi_2$ and $\varphi_3$ are swapped because $\phi_2(C = \mathrm{true}, B = \mathrm{false}) = \varphi_2$ and $\phi_2(C = \mathrm{false}, B = \mathrm{true}) = \varphi_3$.
In summary, $\phi_2$ still entails the same semantics after rearranging its arguments if its mappings (rows in the table of $\phi_2$) are swapped accordingly at the same time.

\begin{figure}[t]
	\centering
	\begin{tikzpicture}
		\node[circle, draw] (A) {$A$};
		\node[circle, draw] (B) [below = 0.5cm of A] {$B$};
		\node[circle, draw] (C) [below = 0.5cm of B] {$C$};
		\factor{below right}{A}{0.25cm and 0.5cm}{270}{$\phi_1$}{f1}
		\factor{below right}{B}{0.25cm and 0.5cm}{270}{$\phi_2$}{f2}
	
		\node[right = 0.5cm of f1, yshift=4.5mm] (tab_f1) {
			\begin{tabular}{c|c|c}
				$A$   & $B$   & $\phi_1(A,B)$ \\ \hline
				true  & true  & $\varphi_1$ \\
				true  & false & $\varphi_2$ \\
				false & true  & $\varphi_3$ \\
				false & false & $\varphi_4$ \\
			\end{tabular}
		};
	
		\node[right = 0.5cm of f2, yshift=-4.5mm] (tab_f2) {
			\begin{tabular}{c|c|c}
				$B$   & $C$   & $\phi_2(B,C)$ \\ \hline
				true  & true  & $\varphi_1$ \\
				true  & false & $\varphi_3$ \\
				false & true  & $\varphi_2$ \\
				false & false & $\varphi_4$ \\
			\end{tabular}
		};
	
		\draw (A) -- (f1);
		\draw (B) -- (f1);
		\draw (B) -- (f2);
		\draw (C) -- (f2);
	\end{tikzpicture}
	\caption{An exemplary \ac{fg} where the input order of the arguments of $\phi_2$ (or $\phi_1$) can be rearranged such that the potentials of $\phi_1$ and $\phi_2$ are identical when comparing their tables.}
	\label{fig:input_order_fg}
\end{figure}
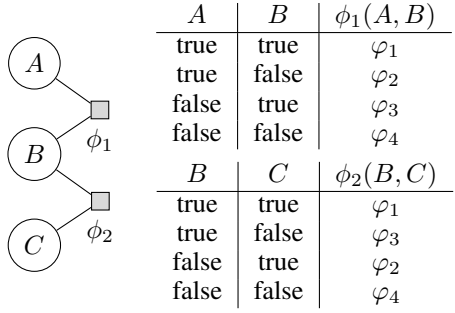

Thus, the potentials of $\phi_1$ and $\phi_2$ are actually identical.
Running \ac{cp} in its current form on the \ac{fg} from \cref{fig:input_order_fg}, however, results in no groups at all for two reasons.
First, depending on the way \ac{cp} checks for identical potentials, it might assign different colours to the factors $\phi_1$ and $\phi_2$ because comparing their tables row by row would lead to different colours for $\phi_1$ and $\phi_2$\footnote{The original \ac{cp} paper does not exactly specify how the check for identical potentials should be implemented.}.
Second, $A$ and $C$ are located at different positions in their respective factor and hence receive different messages during message passing.
However, requiring indistinguishable individuals to appear at the same position in each factor is a massive restriction for practical applications.
On the other hand, naively comparing all $O(n!)$ permutations of an argument list of length $n$ requires a lot of computational effort.
Using histograms to ensure necessary conditions when searching for identical potentials avoids having to try all permutations in advance.
\begin{theorem}[Necessary Conditions for Identical Potentials] \label{th:identical_potentials}
	For two factors $\phi_1(R_1, \ldots, R_n)$ and $\phi_2(R'_1, \ldots, R'_n)$ to be able to represent equivalent potentials, the following two conditions are required to hold:
	\begin{enumerate}
		\item\label{itm:bijection} A bijection $\tau$$:$ $\{R_1, \ldots, R_n\} \to \{R'_1, \ldots, R'_n\}$ exists that maps each $R_i$ to a $R'_j$ such that $\mathcal R(R_i) = \mathcal R(R'_j)$,
		\item\label{itm:identical_histograms} for each histogram $\{(v_i, n_i)\}_{i=1}^{m}$ with $v_i \in \bigcup_{i=1}^{n} \mathcal R(R_i)$, $n_i \in \mathbb N$, $m = |\bigcup_{i=1}^{n} \mathcal R(R_i)|$, and $\sum_i n_i = n$, the multiset of all potentials to which the histogram is mapped is identical for $\phi_1$ and $\phi_2$.
	\end{enumerate}
\end{theorem}
\begin{proof}
	If \cref{itm:bijection} is not satisfied, there must be at least one pair of arguments $R_i$ and $R'_j$ such that $\mathcal R(R_i) \neq \mathcal R(R'_j)$, implying that $\phi_1$ and $\phi_2$ are defined over different function domains and hence cannot have identical potentials.
	Regarding \cref{itm:identical_histograms}, we know that the set of possible histograms specifies all possible inputs for a factor if we neglect the order of its arguments.
	Consequently, if there exists a histogram that is mapped to different multisets of values by $\phi_1$ and $\phi_2$, there is no possibility to permute the arguments such that both tables read identical values from top to bottom.
\end{proof}
In other words, \cref{itm:bijection} ensures that the function domains of the two factors are identical when neglecting the order of their arguments, which implies that both factors entail the same set of possible histograms ($\tau$ does not have to be unique).
\Cref{itm:identical_histograms} demands that the mapping from order-independent inputs to outputs is equivalent for both factors---otherwise, there exists no permutation of the arguments such that both factors have identical potentials.
\begin{example}
	Applying \cref{th:identical_potentials} to our example from \cref{fig:input_order_fg}, there are three different possible histograms, which are identical for $\phi_1$ and $\phi_2$ as their arguments satisfy \cref{itm:bijection}.
	We can also verify that all three histograms yield identical multisets of mapped values for $\phi_1$ and $\phi_2$$:$ $[2,0] \mapsto \{\varphi_1\}$, $[1,1] \mapsto \{\varphi_2, \varphi_3\}$, and $[0,2] \mapsto \{\varphi_4\}$.
\end{example}
Note that we use a multiset instead of a set for each histogram because it is possible for a factor to map a histogram to the same value multiple times.
For example, in a scenario with two factors $\phi'_1$ and $\phi'_2$ where $\phi'_1$ maps some histogram $[i,j] \mapsto \{\varphi'_1, \varphi'_1, \varphi'_2\}$ and $\phi'_2$ maps $[i,j] \mapsto \{\varphi'_1, \varphi'_2, \varphi'_2\}$, the factors cannot represent identical potentials and a set with unique elements $\{\varphi'_1, \varphi'_2\}$ is not able to detect such a situation.
The histograms allow us to avoid a naive check of all possible permutations, reducing the required computational effort as we only need to take a look at permutations if the initial histogram-check is passed successfully.
\begin{corollary} \label{cor:identical_potentials}
	If two factors $\phi_1$ and $\phi_2$ do not satisfy \cref{th:identical_potentials}, they cannot represent equivalent potentials.
\end{corollary}
In case two factors satisfy \cref{th:identical_potentials}, a sufficient condition for them to represent identical potentials is that there exists a permutation of the arguments of one factor such that their tables read identical values from top to bottom.
\begin{corollary}[Sufficient Condition for Identical Potentials] \label{th:identical_potentials_sc}
	Two factors $\phi_1(R_1, \ldots, R_n)$ and $\phi_2(R'_1, \ldots, R'_n)$ represent equivalent potentials if and only if there exists a permutation $\pi$ of $\{1, \ldots, n\}$ such that for all $r_1, \ldots, r_n \in \times_{i=1}^n \mathcal R(R_i)$ it holds that $\phi_1(r_1, \ldots, r_n) = \phi_2(r_{\pi(1)}, \ldots, r_{\pi(n)})$.
\end{corollary}
After rearranging the order of arguments and the order of potentials of a factor accordingly to match the potentials of another factor, the \ac{cp} algorithm can be applied without further changes.
Running \ac{cp} on the \ac{fg} shown in \cref{fig:input_order_fg} after rearrangement, both $A$ and $C$ as well as $\phi_1$ and $\phi_2$ are grouped together.
Using histograms to filter candidates before comparing permutations keeps the computational effort as small as possible and enables us to reuse the histograms for the commutativity check \emph{within} a factor over the whole argument list.
Thus, performing an initial histogram-check does not require additional computational costs and avoids naively comparing all permutations for order independence.
Next, we gather the presented solutions to handle symmetries within and between factors into the \ac{cpr} algorithm and show how \acp{lv} are introduced to obtain a valid \ac{pfg}.

\section{Advanced Colour Passing} \label{sec:cpr}
In this section, we combine the insights from the previous section into a modified version of the \ac{cp} algorithm, called \ac{cpr}.
\Cref{alg:cp_revisited} presents the entire \ac{cpr} algorithm, which is explained in more detail in the following.

\begin{algorithm}[t]
	\caption{Advanced Colour Passing}
	\label{alg:cp_revisited}
	\alginput{An \ac{fg} $G$ with \acp{rv} $\boldsymbol R = \{R_1, \ldots, R_n\}$, \\\hspace*{\algorithmicindent} factors $\boldsymbol \Phi = \{\phi_1, \ldots, \phi_m\}$, and evidence \\\hspace*{\algorithmicindent} $\boldsymbol E = \{R_1 = r_1, \ldots, R_k = r_k\}$.} \\
	\algoutput{A lifted representation $G'$ in form of a \ac{pfg} \\\hspace*{\algorithmicindent} entailing equivalent semantics as $G$.}
	\begin{algorithmic}[1]
		\State Assign each $R_i$ a colour according to $\mathcal R(R_i)$ and $\boldsymbol E$\;
		\State Assign each $\phi_i$ a colour according to order-independent potentials and rearrange arguments accordingly\;
		\Repeat
			\For{each factor $\phi \in \boldsymbol \Phi$}
				\State $signature_{\phi} \gets [\,]$\;
				\For{each \ac{rv} $R \in neighbours(G, \phi)$}
					\State\Comment{In order of appearance in $\phi$}\;
					\State $append(signature_{\phi}, R.colour)$\;
				\EndFor
				\State $append(signature_{\phi}, \phi.colour)$\;
			\EndFor
			\State Group together all $\phi$s with the same signature\;
			\State Assign each such cluster a unique colour\;
			\State Set $\phi.colour$ correspondingly for all $\phi$s\;
			\For{each \ac{rv} $R \in \boldsymbol R$}
				\State $signature_{R} \gets [\,]$\;
				\For{each factor $\phi \in neighbours(G, R)$}
					\If{$\phi$ is commutative w.r.t.\ $\boldsymbol S$ and $R \in \boldsymbol S$}
						\State $append(signature_{R}, (\phi.colour, 0))$\;
					\Else
						\State $append(signature_{R}, (\phi.colour, p(R, \phi)))$\;
					\EndIf
				\EndFor
				\State Sort $signature_{R}$ according to colour\;
				\State $append(signature_{R}, R.colour)$\;
			\EndFor
			\State Group together all $R$s with the same signature\;
			\State Assign each such cluster a unique colour\;
			\State Set $R.colour$ correspondingly for all $R$s\;
		\Until{grouping does not change}
		\State $G' \gets$ construct \acs{pfg} from groupings\;
	\end{algorithmic}
\end{algorithm}

\Ac{cpr} begins with the colour assignment to variable nodes, meaning that all \acp{rv} that have the same range and observed event are assigned the same colour.
Thereafter, \ac{cpr} assigns colours to factor nodes such that factors representing identical potentials are assigned the same colour.
Two factors represent identical potentials if they satisfy the conditions given in \cref{th:identical_potentials} and there exists a rearrangement of one of the factor's arguments such that both factors have identical tables of potentials when comparing them row by row.
As shown in the previous section, \ac{cpr} uses histograms to detect factors with identical potentials regardless of the order of their arguments.
In case the arguments of a factor have to be rearranged to obtain identical tables during the comparison, \ac{cpr} uses the positions of the arguments after the rearrangement throughout the message passing procedure afterwards.
\Ac{cpr} rearranges each factor's arguments at most once.
The message passing in \ac{cpr} differs from the message passing in \ac{cp} in the sense that every factor $\phi(R_1, \ldots, R_j)$ that is commutative with respect to a subset of its arguments $\boldsymbol S \subseteq \{R_1, \ldots, R_j\}$ passes the position $p(R, \phi)$ of a \ac{rv} $R$ in $\phi$'s argument list only to \acp{rv} $R \notin \boldsymbol S$.
Every factor $\phi$ passes zero instead of the actual position $p(R, \phi)$ to all \acp{rv} $R \in \boldsymbol S$ to mark commutativity.
$\boldsymbol S$ is a maximal subset for which $\abs{\boldsymbol S} > 1$ must hold as a single argument is always commutative with itself.
All \acp{rv} receiving the position zero in their message are commutative and thus, \ac{cpr} groups them using a \ac{crv}, as we have seen in \cref{sec:intra_factor_symmetries}.
In case there are multiple maximal subsets $\boldsymbol S$, \ac{cpr} chooses any of them.
\Ac{cpr} iterates the message passing until convergence.
In the end, \ac{cpr} transforms all groups of \acp{rv} and factors into \acp{prv} with \acp{lv} and \acp{pf}, respectively, to obtain a valid \ac{pfg} $G'$ entailing equivalent semantics as the initial \ac{fg} $G$.
The construction of $G'$ from the obtained groupings is explained in detail below.

We give the mapping from groups to \acp{prv} and \acp{pf} for the domain-liftable fragment~\citep{VanDenBroeck2011a}, i.e., for \acp{pfg} containing only \acp{pf} in which at most two \acp{lv} appear as well as for \acp{pfg} containing only \acp{prv} having at most one \ac{lv}, because lifted inference algorithms such as \ac{lve} and \ac{ljt} are proven to be complete for this fragment~\citep{Taghipour2013b,Braun2020a}.

Each group of factors $\boldsymbol F$ is replaced by a \ac{pf} $\phi'$ and each group of \acp{rv} $\boldsymbol A$ is replaced by a \ac{prv} $R'$ such that $gr(\phi') = \boldsymbol F$ and $gr(R') = \boldsymbol A$.
The \ac{pfg} $G'$ then contains an edge between a \ac{prv} $R'$ and a \ac{pf} $\phi'$ (i.e., $R'$ appears in the argument list of $\phi'$) if there is a \ac{rv} $R \in gr(R')$ which is connected to a factor $\phi \in gr(\phi')$ in the initial \ac{fg} $G$.
For each \ac{prv} $R'$, the \acp{lv} are introduced depending on the groundings $gr(R')$.
For the introduction of \acp{lv}, we only need to consider \acp{prv} that are not parameterless, i.e., \acp{prv} that represent a group consisting of at least two \acp{rv}.
The exact conditions used by \ac{cpr} for introducing \acp{lv} are given in the following definition.
\begin{definition}[Introduction of Logvars in Randvar Groups] \label{def:lv_introduction}
	Let $\phi'(R'_1(X_{1,1}, \ldots, X_{1,k}), \ldots, R'_j(X_{j,1}, \ldots, X_{j,k}))$ be a new \ac{pf}, build from $\boldsymbol F = \{\phi_1(R_{1,1}, \ldots, R_{1,s}), \allowbreak \ldots, \allowbreak \phi_{\ell}(R_{\ell,1}, \ldots, R_{\ell,s})\}$ and let $\boldsymbol S = \{S'_1, \ldots, S'_z\}$ denote the subset of $\phi'$'s arguments with more than one grounding.
	Then, \ac{cpr} introduces the \acp{lv} of $S'_1, \ldots, S'_z$ as follows.
	\begin{enumerate}
		\item \label{item:single_lv_shared} If $\forall S'_i \in \boldsymbol S$$:$ $\abs{\boldsymbol F} = \abs{gr(S'_i)}$, all $S'_i \in \boldsymbol S$ have exactly one \ac{lv} which is identical for all $S'_i \in \boldsymbol S$.
		\item \label{item:single_lv_distinct} If $\forall S'_i \in \boldsymbol S$$:$ $\abs{\boldsymbol F} \neq \abs{gr(S'_i)}$, $S'_1, \ldots, S'_z$ have exactly one \ac{lv}.
		$S'_a \in \boldsymbol S$ and $S'_b \in \boldsymbol S$ share the same \ac{lv} if and only if $\abs{gr(S'_a)} = \abs{gr(S'_b)}$ and there exists a bijection $\tau$$:$ $gr(S'_a) \to gr(S'_b)$ such that $\tau$ maps every $S_a \in gr(S'_a)$ to $S_b \in gr(S'_b)$ with $\mathcal F(S_a) = \mathcal F(S_b)$ where $\mathcal F(S) = \{\phi(R_1, \ldots, R_s) \in \boldsymbol F \mid S \in \{R_1, \ldots, R_s\}\}$. % $S_a$ and $S_b$ appearing in the same factors in $\boldsymbol F$.
		\item \label{item:two_lvs} If $\exists S'_u \in \boldsymbol S$$:$ $\abs{\boldsymbol F} \neq \abs{gr(S'_u)} \land \exists S'_v \in \boldsymbol S$$:$ $\abs{\boldsymbol F} = \abs{gr(S'_v)}$, all $S'_i \in \boldsymbol S$ with $\abs{\boldsymbol F} = \abs{gr(S'_i)}$ have two \acp{lv}.
		The remaining $S'_i \in \boldsymbol S$ have exactly one \ac{lv} and share the same \ac{lv} under the same conditions as in \cref{item:single_lv_distinct}.
	\end{enumerate}
	For each \ac{prv} $S'_i \in \boldsymbol S$ with a single \ac{lv} $X$, choose $\mathcal D(X)$ such that $\abs{gr(S'_i)} = \abs{\mathcal D(X)}$.
	For each \ac{prv} $S'_i \in \boldsymbol S$ with two \acp{lv} $X_1$ and $X_2$, choose $\mathcal D(X_1)$ and $\mathcal D(X_2)$ such that $\abs{gr(S'_i)} = \abs{\mathcal D(X_1)} \cdot \abs{\mathcal D(X_2)}$.
\end{definition}
The intuition behind \cref{def:lv_introduction} is that after building $\phi'$ from $\boldsymbol F$, it must hold that $gr(\phi') = \boldsymbol F$ to ensure that $G'$ entails equivalent semantics as $G$.
Introducing \acp{lv} according to \cref{def:lv_introduction} ensures that grounding $G'$ results in a model equivalent to $G$ for the domain-liftable fragment.
\begin{theorem} \label{th:lv_introduction}
	\Ac{cpr} returns a valid \ac{pfg} entailing equivalent semantics as the initial \ac{fg} for the domain-liftable fragment.
\end{theorem}
We give a proof for \cref{th:lv_introduction} in \cref{appendix:logvars}.
\Cref{th:lv_introduction} gives us the theoretical guarantees that \ac{cpr} introduces \acp{lv} correctly to obtain a valid \ac{pfg}.
In the next section, we demonstrate the effectiveness \ac{cpr} in practice.

\section{Experiments} \label{sec:eval}
In addition to the theoretical results, we show that \ac{cpr} is able to drastically reduce online query times in practice.
We evaluate the impact of symmetries within factors and between factors on query times separately.
\Cref{fig:plot-results-intra-k=1,fig:plot-results-inter-p=3} display the experimental results.
In both plots, we report average query times of \ac{lve} on the resulting \ac{pfg} after running \ac{cpr}, denoted as \acs{lve} (\acs{cpr}), of \ac{lve} on the resulting \ac{pfg} after running \ac{cp}\footnote{\Ac{cp} itself does not construct a valid \ac{pfg}, so we additionally applied the steps from \cref{def:lv_introduction} on the result obtained from \ac{cp}.}, denoted as \acs{lve} (\acs{cp}), and of \acl{ve} on the initial \ac{fg} (\acs{ve}).
The average query times are given by the lines and the ribbon around the lines indicates the standard deviation.
In both plots, the y-axis uses a logarithmic scale.
We provide the data set generators along with our source code in the supplementary material.

\begin{figure}[t]
	\centering
	\begin{tikzpicture}[x=1pt,y=1pt]
		\definecolor{fillColor}{RGB}{255,255,255}
		\path[use as bounding box,fill=fillColor,fill opacity=0.00] (0,15) rectangle (238.49,115.63);
		\begin{scope}
		\path[clip] (  0.00,  0.00) rectangle (238.49,115.63);
		\definecolor{drawColor}{RGB}{255,255,255}
		\definecolor{fillColor}{RGB}{255,255,255}
		
		\path[draw=drawColor,line width= 0.6pt,line join=round,line cap=round,fill=fillColor] (  0.00,  0.00) rectangle (238.49,115.63);
		\end{scope}
		\begin{scope}
		\path[clip] ( 40.51, 30.69) rectangle (232.99,110.13);
		\definecolor{fillColor}{RGB}{255,255,255}
		
		\path[fill=fillColor] ( 40.51, 30.69) rectangle (232.99,110.13);
		\definecolor{drawColor}{RGB}{247,192,26}
		
		\path[draw=drawColor,line width= 0.6pt,line join=round] ( 49.26, 42.99) --
			( 68.70, 46.18) --
			(107.59, 45.32) --
			(146.47, 46.95) --
			(185.36, 46.98) --
			(224.24, 47.05);
		\definecolor{drawColor}{RGB}{78,155,133}
		
		\path[draw=drawColor,line width= 0.6pt,dash pattern=on 2pt off 2pt ,line join=round] ( 49.26, 40.70) --
			( 68.70, 45.37) --
			(107.59, 53.76) --
			(146.47, 61.26) --
			(185.36, 69.16) --
			(224.24, 83.97);
		\definecolor{drawColor}{RGB}{37,122,164}
		
		\path[draw=drawColor,line width= 0.6pt,dash pattern=on 4pt off 2pt ,line join=round] ( 49.26, 36.29) --
			( 68.70, 38.87) --
			(107.59, 51.35) --
			(146.47, 61.52) --
			(185.36, 80.30) --
			(224.24,105.45);
		\definecolor{fillColor}{RGB}{78,155,133}
		
		\path[fill=fillColor] (107.59, 56.81) --
			(110.23, 52.23) --
			(104.94, 52.23) --
			cycle;
		
		\path[fill=fillColor] (146.47, 64.31) --
			(149.11, 59.74) --
			(143.83, 59.74) --
			cycle;
		
		\path[fill=fillColor] (224.24, 87.02) --
			(226.88, 82.44) --
			(221.60, 82.44) --
			cycle;
		
		\path[fill=fillColor] ( 68.70, 48.42) --
			( 71.34, 43.84) --
			( 66.06, 43.84) --
			cycle;
		
		\path[fill=fillColor] (185.36, 72.21) --
			(188.00, 67.63) --
			(182.71, 67.63) --
			cycle;
		
		\path[fill=fillColor] ( 49.26, 43.75) --
			( 51.90, 39.17) --
			( 46.62, 39.17) --
			cycle;
		\definecolor{fillColor}{RGB}{247,192,26}
		
		\path[fill=fillColor] (107.59, 45.32) circle (  1.96);
		
		\path[fill=fillColor] (146.47, 46.95) circle (  1.96);
		
		\path[fill=fillColor] (224.24, 47.05) circle (  1.96);
		
		\path[fill=fillColor] ( 68.70, 46.18) circle (  1.96);
		
		\path[fill=fillColor] (185.36, 46.98) circle (  1.96);
		
		\path[fill=fillColor] ( 49.26, 42.99) circle (  1.96);
		\definecolor{fillColor}{RGB}{37,122,164}
		
		\path[fill=fillColor] (105.62, 49.39) --
			(109.55, 49.39) --
			(109.55, 53.31) --
			(105.62, 53.31) --
			cycle;
		
		\path[fill=fillColor] (144.51, 59.56) --
			(148.43, 59.56) --
			(148.43, 63.48) --
			(144.51, 63.48) --
			cycle;
		
		\path[fill=fillColor] (222.28,103.49) --
			(226.20,103.49) --
			(226.20,107.42) --
			(222.28,107.42) --
			cycle;
		
		\path[fill=fillColor] ( 66.74, 36.91) --
			( 70.66, 36.91) --
			( 70.66, 40.83) --
			( 66.74, 40.83) --
			cycle;
		
		\path[fill=fillColor] (183.39, 78.33) --
			(187.32, 78.33) --
			(187.32, 82.26) --
			(183.39, 82.26) --
			cycle;
		
		\path[fill=fillColor] ( 47.30, 34.33) --
			( 51.22, 34.33) --
			( 51.22, 38.26) --
			( 47.30, 38.26) --
			cycle;
		\definecolor{fillColor}{RGB}{247,192,26}
		
		\path[fill=fillColor,fill opacity=0.20] ( 49.26, 44.51) --
			( 68.70, 48.82) --
			(107.59, 47.63) --
			(146.47, 48.84) --
			(185.36, 49.59) --
			(224.24, 49.62) --
			(224.24, 43.67) --
			(185.36, 43.55) --
			(146.47, 44.67) --
			(107.59, 42.38) --
			( 68.70, 42.70) --
			( 49.26, 41.22) --
			cycle;
		
		\path[] ( 49.26, 44.51) --
			( 68.70, 48.82) --
			(107.59, 47.63) --
			(146.47, 48.84) --
			(185.36, 49.59) --
			(224.24, 49.62);
		
		\path[] (224.24, 43.67) --
			(185.36, 43.55) --
			(146.47, 44.67) --
			(107.59, 42.38) --
			( 68.70, 42.70) --
			( 49.26, 41.22);
		\definecolor{fillColor}{RGB}{78,155,133}
		
		\path[fill=fillColor,fill opacity=0.20] ( 49.26, 42.39) --
			( 68.70, 47.02) --
			(107.59, 56.20) --
			(146.47, 63.13) --
			(185.36, 72.34) --
			(224.24, 90.91) --
			(224.24, 59.99) --
			(185.36, 64.63) --
			(146.47, 59.01) --
			(107.59, 50.61) --
			( 68.70, 43.41) --
			( 49.26, 38.69) --
			cycle;
		
		\path[] ( 49.26, 42.39) --
			( 68.70, 47.02) --
			(107.59, 56.20) --
			(146.47, 63.13) --
			(185.36, 72.34) --
			(224.24, 90.91);
		
		\path[] (224.24, 59.99) --
			(185.36, 64.63) --
			(146.47, 59.01) --
			(107.59, 50.61) --
			( 68.70, 43.41) --
			( 49.26, 38.69);
		\definecolor{fillColor}{RGB}{37,122,164}
		
		\path[fill=fillColor,fill opacity=0.20] ( 49.26, 37.98) --
			( 68.70, 41.08) --
			(107.59, 52.74) --
			(146.47, 62.42) --
			(185.36, 80.46) --
			(224.24,106.52) --
			(224.24,104.27) --
			(185.36, 80.13) --
			(146.47, 60.53) --
			(107.59, 49.75) --
			( 68.70, 36.09) --
			( 49.26, 34.30) --
			cycle;
		
		\path[] ( 49.26, 37.98) --
			( 68.70, 41.08) --
			(107.59, 52.74) --
			(146.47, 62.42) --
			(185.36, 80.46) --
			(224.24,106.52);
		
		\path[] (224.24,104.27) --
			(185.36, 80.13) --
			(146.47, 60.53) --
			(107.59, 49.75) --
			( 68.70, 36.09) --
			( 49.26, 34.30);
		\end{scope}
		\begin{scope}
		\path[clip] (  0.00,  0.00) rectangle (238.49,115.63);
		\definecolor{drawColor}{RGB}{0,0,0}
		
		\path[draw=drawColor,line width= 0.6pt,line join=round] ( 40.51, 30.69) --
			( 40.51,110.13);
		
		\path[draw=drawColor,line width= 0.6pt,line join=round] ( 41.93,107.67) --
			( 40.51,110.13) --
			( 39.09,107.67);
		\end{scope}
		\begin{scope}
		\path[clip] (  0.00,  0.00) rectangle (238.49,115.63);
		\definecolor{drawColor}{gray}{0.30}
		
		\node[text=drawColor,anchor=base east,inner sep=0pt, outer sep=0pt, scale=  0.88] at ( 35.56, 33.32) {10};
		
		\node[text=drawColor,anchor=base east,inner sep=0pt, outer sep=0pt, scale=  0.88] at ( 35.56, 58.45) {100};
		
		\node[text=drawColor,anchor=base east,inner sep=0pt, outer sep=0pt, scale=  0.88] at ( 35.56, 83.59) {1000};
		\end{scope}
		\begin{scope}
		\path[clip] (  0.00,  0.00) rectangle (238.49,115.63);
		\definecolor{drawColor}{gray}{0.20}
		
		\path[draw=drawColor,line width= 0.6pt,line join=round] ( 37.76, 36.35) --
			( 40.51, 36.35);
		
		\path[draw=drawColor,line width= 0.6pt,line join=round] ( 37.76, 61.48) --
			( 40.51, 61.48);
		
		\path[draw=drawColor,line width= 0.6pt,line join=round] ( 37.76, 86.62) --
			( 40.51, 86.62);
		\end{scope}
		\begin{scope}
		\path[clip] (  0.00,  0.00) rectangle (238.49,115.63);
		\definecolor{drawColor}{RGB}{0,0,0}
		
		\path[draw=drawColor,line width= 0.6pt,line join=round] ( 40.51, 30.69) --
			(232.99, 30.69);
		
		\path[draw=drawColor,line width= 0.6pt,line join=round] (230.53, 29.26) --
			(232.99, 30.69) --
			(230.53, 32.11);
		\end{scope}
		\begin{scope}
		\path[clip] (  0.00,  0.00) rectangle (238.49,115.63);
		\definecolor{drawColor}{gray}{0.20}
		
		\path[draw=drawColor,line width= 0.6pt,line join=round] ( 78.42, 27.94) --
			( 78.42, 30.69);
		
		\path[draw=drawColor,line width= 0.6pt,line join=round] (127.03, 27.94) --
			(127.03, 30.69);
		
		\path[draw=drawColor,line width= 0.6pt,line join=round] (175.64, 27.94) --
			(175.64, 30.69);
		
		\path[draw=drawColor,line width= 0.6pt,line join=round] (224.24, 27.94) --
			(224.24, 30.69);
		\end{scope}
		\begin{scope}
		\path[clip] (  0.00,  0.00) rectangle (238.49,115.63);
		\definecolor{drawColor}{gray}{0.30}
		
		\node[text=drawColor,anchor=base,inner sep=0pt, outer sep=0pt, scale=  0.88] at ( 78.42, 19.68) {5};
		
		\node[text=drawColor,anchor=base,inner sep=0pt, outer sep=0pt, scale=  0.88] at (127.03, 19.68) {10};
		
		\node[text=drawColor,anchor=base,inner sep=0pt, outer sep=0pt, scale=  0.88] at (175.64, 19.68) {15};
		
		\node[text=drawColor,anchor=base,inner sep=0pt, outer sep=0pt, scale=  0.88] at (224.24, 19.68) {20};
		\end{scope}
		\begin{scope}
		\path[clip] (  0.00,  0.00) rectangle (238.49,115.63);
		\definecolor{drawColor}{RGB}{0,0,0}
		
		\node[text=drawColor,anchor=base,inner sep=0pt, outer sep=0pt, scale=  1.10] at (136.75,  7.64) {$d$};
		\end{scope}
		\begin{scope}
		\path[clip] (  0.00,  0.00) rectangle (238.49,115.63);
		\definecolor{drawColor}{RGB}{0,0,0}
		
		\node[text=drawColor,rotate= 90.00,anchor=base,inner sep=0pt, outer sep=0pt, scale=  1.10] at ( 13.08, 70.41) {time (ms)};
		\end{scope}
		\begin{scope}
		\path[clip] (  0.00,  0.00) rectangle (238.49,115.63);
		
		\path[] ( 37.04, 71.03) rectangle (109.43,125.40);
		\end{scope}
		\begin{scope}
		\path[clip] (  0.00,  0.00) rectangle (238.49,115.63);
		\definecolor{drawColor}{RGB}{247,192,26}
		
		\path[draw=drawColor,line width= 0.6pt,line join=round] ( 43.98,112.67) -- ( 55.55,112.67);
		\end{scope}
		\begin{scope}
		\path[clip] (  0.00,  0.00) rectangle (238.49,115.63);
		\definecolor{fillColor}{RGB}{247,192,26}
		
		\path[fill=fillColor] ( 49.76,112.67) circle (  1.96);
		\end{scope}
		\begin{scope}
		\path[clip] (  0.00,  0.00) rectangle (238.49,115.63);
		\definecolor{drawColor}{RGB}{78,155,133}
		
		\path[draw=drawColor,line width= 0.6pt,dash pattern=on 2pt off 2pt ,line join=round] ( 43.98, 98.22) -- ( 55.55, 98.22);
		\end{scope}
		\begin{scope}
		\path[clip] (  0.00,  0.00) rectangle (238.49,115.63);
		\definecolor{fillColor}{RGB}{78,155,133}
		
		\path[fill=fillColor] ( 49.76,101.27) --
			( 52.41, 96.69) --
			( 47.12, 96.69) --
			cycle;
		\end{scope}
		\begin{scope}
		\path[clip] (  0.00,  0.00) rectangle (238.49,115.63);
		\definecolor{drawColor}{RGB}{37,122,164}
		
		\path[draw=drawColor,line width= 0.6pt,dash pattern=on 4pt off 2pt ,line join=round] ( 43.98, 83.76) -- ( 55.55, 83.76);
		\end{scope}
		\begin{scope}
		\path[clip] (  0.00,  0.00) rectangle (238.49,115.63);
		\definecolor{fillColor}{RGB}{37,122,164}
		
		\path[fill=fillColor] ( 47.80, 81.80) --
			( 51.73, 81.80) --
			( 51.73, 85.72) --
			( 47.80, 85.72) --
			cycle;
		\end{scope}
		\begin{scope}
		\path[clip] (  0.00,  0.00) rectangle (238.49,115.63);
		\definecolor{drawColor}{RGB}{0,0,0}
		
		\node[text=drawColor,anchor=base west,inner sep=0pt, outer sep=0pt, scale=  0.80] at ( 62.49,109.91) {LVE (ACP)};
		\end{scope}
		\begin{scope}
		\path[clip] (  0.00,  0.00) rectangle (238.49,115.63);
		\definecolor{drawColor}{RGB}{0,0,0}
		
		\node[text=drawColor,anchor=base west,inner sep=0pt, outer sep=0pt, scale=  0.80] at ( 62.49, 95.46) {LVE (CP)};
		\end{scope}
		\begin{scope}
		\path[clip] (  0.00,  0.00) rectangle (238.49,115.63);
		\definecolor{drawColor}{RGB}{0,0,0}
		
		\node[text=drawColor,anchor=base west,inner sep=0pt, outer sep=0pt, scale=  0.80] at ( 62.49, 81.01) {VE};
		\end{scope}
	\end{tikzpicture}	
	\caption{Average query times and their standard deviation of \ac{lve} on the output of \ac{cp}, \ac{lve} on the output of \ac{cpr}, and \ac{ve} on the initial (propositional) \ac{fg} for input \acp{fg} containing a single commutative factor.}
	\label{fig:plot-results-intra-k=1}
\end{figure}
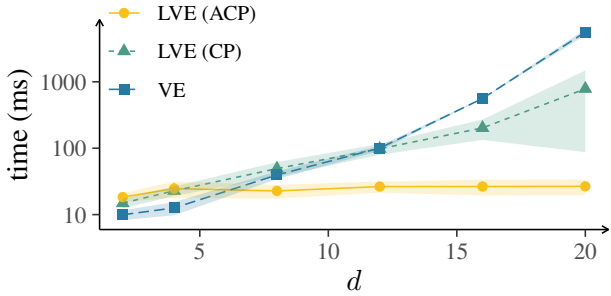

The data set used in \cref{fig:plot-results-intra-k=1} consists of \acp{fg} containing between five and 102 factors of which exactly one factor is commutative.
More specifically, for each domain size $d = 2, 4, 8, 12, 16, 20$, there are between $2d + 1$ and $d \cdot \lfloor \log_2(d) \rfloor + 2d + 2$ Boolean \acp{rv} and between $2d + 1$ and $d \cdot \lfloor \log_2(d) \rfloor + d + 2$ factors in the \acp{fg}.
All factors representing identical potentials have exactly the same tables when comparing them row by row, that is, there are no symmetries between factors which cannot be detected by \ac{cp} and hence, this is the ideal scenario for \ac{cp} (although unrealistic in practice).
The maximum number of arguments of a factor is $d+1$, i.e., there are at most $2^{d+1}$ input-output pairs for a factor.
For each choice of $d$, we evaluate multiple \acp{fg} by posing two queries per \ac{fg} and then report the average run time and standard deviation over all queries for that choice of $d$.
\Cref{fig:plot-results-intra-k=1} demonstrates that \ac{cpr} yields a significant speedup (up to factor 30) compared to \ac{cp} even though there is just a single commutative factor in the input \ac{fg} and the potentials of the factors are specified in an optimal way for \ac{cp}.
The results indicate that \ac{cp} imposes significant scalability issues for \acp{fg} containing commutative factors, even for rather small domain sizes.
Thus, \ac{cpr} is a major step to get a grip on scalability issues.
Unsurprisingly, \ac{ve} is the slowest of all algorithms.
\cref{appendix:more_eval} provides additional results for \acp{fg} containing more than one commutative factor.

\begin{figure}[t]
	\centering
	\begin{tikzpicture}[x=1pt,y=1pt]
		\definecolor{fillColor}{RGB}{255,255,255}
		\path[use as bounding box,fill=fillColor,fill opacity=0.00] (0,15) rectangle (238.49,115.63);
		\begin{scope}
		\path[clip] (  0.00,  0.00) rectangle (238.49,115.63);
		\definecolor{drawColor}{RGB}{255,255,255}
		\definecolor{fillColor}{RGB}{255,255,255}
		
		\path[draw=drawColor,line width= 0.6pt,line join=round,line cap=round,fill=fillColor] (  0.00,  0.00) rectangle (238.49,115.63);
		\end{scope}
		\begin{scope}
		\path[clip] ( 40.51, 30.69) rectangle (232.99,110.13);
		\definecolor{fillColor}{RGB}{255,255,255}
		
		\path[fill=fillColor] ( 40.51, 30.69) rectangle (232.99,110.13);
		\definecolor{drawColor}{RGB}{247,192,26}
		
		\path[draw=drawColor,line width= 0.6pt,line join=round] ( 49.26, 44.83) --
			( 49.60, 45.29) --
			( 50.29, 45.43) --
			( 50.97, 46.09) --
			( 51.66, 44.72) --
			( 52.34, 45.56) --
			( 54.40, 45.17) --
			( 59.87, 45.46) --
			( 70.83, 45.63) --
			( 92.75, 48.14) --
			(136.58, 48.71) --
			(224.24, 50.10);
		\definecolor{drawColor}{RGB}{78,155,133}
		
		\path[draw=drawColor,line width= 0.6pt,dash pattern=on 2pt off 2pt ,line join=round] ( 49.26, 43.73) --
			( 49.60, 44.83) --
			( 50.29, 45.05) --
			( 50.97, 46.85) --
			( 51.66, 47.15) --
			( 52.34, 49.26) --
			( 54.40, 51.92) --
			( 59.87, 53.05) --
			( 70.83, 58.18) --
			( 92.75, 64.89) --
			(136.58, 75.28) --
			(224.24, 85.91);
		\definecolor{drawColor}{RGB}{37,122,164}
		
		\path[draw=drawColor,line width= 0.6pt,dash pattern=on 4pt off 2pt ,line join=round] ( 49.26, 41.60) --
			( 49.60, 45.92) --
			( 50.29, 50.49) --
			( 50.97, 53.29) --
			( 51.66, 55.78) --
			( 52.34, 56.86) --
			( 54.40, 59.47) --
			( 59.87, 65.13) --
			( 70.83, 70.72) --
			( 92.75, 77.45) --
			(136.58, 87.49) --
			(224.24,100.32);
		\definecolor{fillColor}{RGB}{78,155,133}
		
		\path[fill=fillColor] ( 54.40, 54.98) --
			( 57.04, 50.40) --
			( 51.75, 50.40) --
			cycle;
		
		\path[fill=fillColor] ( 50.97, 49.90) --
			( 53.61, 45.32) --
			( 48.33, 45.32) --
			cycle;
		
		\path[fill=fillColor] ( 52.34, 52.31) --
			( 54.98, 47.73) --
			( 49.70, 47.73) --
			cycle;
		
		\path[fill=fillColor] ( 92.75, 67.94) --
			( 95.39, 63.36) --
			( 90.11, 63.36) --
			cycle;
		
		\path[fill=fillColor] ( 49.60, 47.88) --
			( 52.24, 43.30) --
			( 46.96, 43.30) --
			cycle;
		
		\path[fill=fillColor] ( 51.66, 50.21) --
			( 54.30, 45.63) --
			( 49.01, 45.63) --
			cycle;
		
		\path[fill=fillColor] (136.58, 78.33) --
			(139.22, 73.75) --
			(133.94, 73.75) --
			cycle;
		
		\path[fill=fillColor] ( 70.83, 61.23) --
			( 73.48, 56.65) --
			( 68.19, 56.65) --
			cycle;
		
		\path[fill=fillColor] ( 50.29, 48.10) --
			( 52.93, 43.52) --
			( 47.64, 43.52) --
			cycle;
		
		\path[fill=fillColor] ( 59.87, 56.10) --
			( 62.52, 51.52) --
			( 57.23, 51.52) --
			cycle;
		
		\path[fill=fillColor] (224.24, 88.96) --
			(226.88, 84.38) --
			(221.60, 84.38) --
			cycle;
		
		\path[fill=fillColor] ( 49.26, 46.79) --
			( 51.90, 42.21) --
			( 46.62, 42.21) --
			cycle;
		\definecolor{fillColor}{RGB}{247,192,26}
		
		\path[fill=fillColor] ( 54.40, 45.17) circle (  1.96);
		
		\path[fill=fillColor] ( 50.97, 46.09) circle (  1.96);
		
		\path[fill=fillColor] ( 52.34, 45.56) circle (  1.96);
		
		\path[fill=fillColor] ( 92.75, 48.14) circle (  1.96);
		
		\path[fill=fillColor] ( 49.60, 45.29) circle (  1.96);
		
		\path[fill=fillColor] ( 51.66, 44.72) circle (  1.96);
		
		\path[fill=fillColor] (136.58, 48.71) circle (  1.96);
		
		\path[fill=fillColor] ( 70.83, 45.63) circle (  1.96);
		
		\path[fill=fillColor] ( 50.29, 45.43) circle (  1.96);
		
		\path[fill=fillColor] ( 59.87, 45.46) circle (  1.96);
		
		\path[fill=fillColor] (224.24, 50.10) circle (  1.96);
		
		\path[fill=fillColor] ( 49.26, 44.83) circle (  1.96);
		\definecolor{fillColor}{RGB}{37,122,164}
		
		\path[fill=fillColor] ( 52.43, 57.51) --
			( 56.36, 57.51) --
			( 56.36, 61.43) --
			( 52.43, 61.43) --
			cycle;
		
		\path[fill=fillColor] ( 49.01, 51.33) --
			( 52.93, 51.33) --
			( 52.93, 55.25) --
			( 49.01, 55.25) --
			cycle;
		
		\path[fill=fillColor] ( 50.38, 54.90) --
			( 54.30, 54.90) --
			( 54.30, 58.82) --
			( 50.38, 58.82) --
			cycle;
		
		\path[fill=fillColor] ( 90.79, 75.49) --
			( 94.71, 75.49) --
			( 94.71, 79.41) --
			( 90.79, 79.41) --
			cycle;
		
		\path[fill=fillColor] ( 47.64, 43.96) --
			( 51.56, 43.96) --
			( 51.56, 47.88) --
			( 47.64, 47.88) --
			cycle;
		
		\path[fill=fillColor] ( 49.69, 53.82) --
			( 53.62, 53.82) --
			( 53.62, 57.74) --
			( 49.69, 57.74) --
			cycle;
		
		\path[fill=fillColor] (134.62, 85.53) --
			(138.54, 85.53) --
			(138.54, 89.45) --
			(134.62, 89.45) --
			cycle;
		
		\path[fill=fillColor] ( 68.87, 68.76) --
			( 72.79, 68.76) --
			( 72.79, 72.69) --
			( 68.87, 72.69) --
			cycle;
		
		\path[fill=fillColor] ( 48.32, 48.53) --
			( 52.25, 48.53) --
			( 52.25, 52.46) --
			( 48.32, 52.46) --
			cycle;
		
		\path[fill=fillColor] ( 57.91, 63.17) --
			( 61.84, 63.17) --
			( 61.84, 67.09) --
			( 57.91, 67.09) --
			cycle;
		
		\path[fill=fillColor] (222.28, 98.36) --
			(226.20, 98.36) --
			(226.20,102.28) --
			(222.28,102.28) --
			cycle;
		
		\path[fill=fillColor] ( 47.30, 39.64) --
			( 51.22, 39.64) --
			( 51.22, 43.57) --
			( 47.30, 43.57) --
			cycle;
		\definecolor{fillColor}{RGB}{247,192,26}
		
		\path[fill=fillColor,fill opacity=0.20] ( 49.26, 46.93) --
			( 49.60, 47.48) --
			( 50.29, 47.78) --
			( 50.97, 49.59) --
			( 51.66, 46.36) --
			( 52.34, 48.51) --
			( 54.40, 47.52) --
			( 59.87, 47.80) --
			( 70.83, 47.93) --
			( 92.75, 51.82) --
			(136.58, 51.90) --
			(224.24, 53.08) --
			(224.24, 46.03) --
			(136.58, 44.22) --
			( 92.75, 42.62) --
			( 70.83, 42.72) --
			( 59.87, 42.51) --
			( 54.40, 42.17) --
			( 52.34, 41.54) --
			( 51.66, 42.79) --
			( 50.97, 40.97) --
			( 50.29, 42.46) --
			( 49.60, 42.57) --
			( 49.26, 42.24) --
			cycle;
		
		\path[] ( 49.26, 46.93) --
			( 49.60, 47.48) --
			( 50.29, 47.78) --
			( 50.97, 49.59) --
			( 51.66, 46.36) --
			( 52.34, 48.51) --
			( 54.40, 47.52) --
			( 59.87, 47.80) --
			( 70.83, 47.93) --
			( 92.75, 51.82) --
			(136.58, 51.90) --
			(224.24, 53.08);
		
		\path[] (224.24, 46.03) --
			(136.58, 44.22) --
			( 92.75, 42.62) --
			( 70.83, 42.72) --
			( 59.87, 42.51) --
			( 54.40, 42.17) --
			( 52.34, 41.54) --
			( 51.66, 42.79) --
			( 50.97, 40.97) --
			( 50.29, 42.46) --
			( 49.60, 42.57) --
			( 49.26, 42.24);
		\definecolor{fillColor}{RGB}{78,155,133}
		
		\path[fill=fillColor,fill opacity=0.20] ( 49.26, 45.58) --
			( 49.60, 47.06) --
			( 50.29, 47.62) --
			( 50.97, 49.76) --
			( 51.66, 51.12) --
			( 52.34, 51.81) --
			( 54.40, 56.90) --
			( 59.87, 56.78) --
			( 70.83, 61.87) --
			( 92.75, 67.50) --
			(136.58, 80.30) --
			(224.24, 92.24) --
			(224.24, 69.79) --
			(136.58, 65.91) --
			( 92.75, 61.47) --
			( 70.83, 52.63) --
			( 59.87, 47.38) --
			( 54.40, 42.72) --
			( 52.34, 45.95) --
			( 51.66, 40.94) --
			( 50.97, 42.91) --
			( 50.29, 41.69) --
			( 49.60, 42.04) --
			( 49.26, 41.52) --
			cycle;
		
		\path[] ( 49.26, 45.58) --
			( 49.60, 47.06) --
			( 50.29, 47.62) --
			( 50.97, 49.76) --
			( 51.66, 51.12) --
			( 52.34, 51.81) --
			( 54.40, 56.90) --
			( 59.87, 56.78) --
			( 70.83, 61.87) --
			( 92.75, 67.50) --
			(136.58, 80.30) --
			(224.24, 92.24);
		
		\path[] (224.24, 69.79) --
			(136.58, 65.91) --
			( 92.75, 61.47) --
			( 70.83, 52.63) --
			( 59.87, 47.38) --
			( 54.40, 42.72) --
			( 52.34, 45.95) --
			( 51.66, 40.94) --
			( 50.97, 42.91) --
			( 50.29, 41.69) --
			( 49.60, 42.04) --
			( 49.26, 41.52);
		\definecolor{fillColor}{RGB}{37,122,164}
		
		\path[fill=fillColor,fill opacity=0.20] ( 49.26, 45.98) --
			( 49.60, 50.49) --
			( 50.29, 54.53) --
			( 50.97, 56.67) --
			( 51.66, 58.21) --
			( 52.34, 59.28) --
			( 54.40, 61.78) --
			( 59.87, 67.14) --
			( 70.83, 73.39) --
			( 92.75, 80.37) --
			(136.58, 91.79) --
			(224.24,106.52) --
			(224.24, 85.11) --
			(136.58, 80.41) --
			( 92.75, 73.49) --
			( 70.83, 67.21) --
			( 59.87, 62.68) --
			( 54.40, 56.55) --
			( 52.34, 53.77) --
			( 51.66, 52.67) --
			( 50.97, 48.41) --
			( 50.29, 44.11) --
			( 49.60, 38.04) --
			( 49.26, 34.30) --
			cycle;
		
		\path[] ( 49.26, 45.98) --
			( 49.60, 50.49) --
			( 50.29, 54.53) --
			( 50.97, 56.67) --
			( 51.66, 58.21) --
			( 52.34, 59.28) --
			( 54.40, 61.78) --
			( 59.87, 67.14) --
			( 70.83, 73.39) --
			( 92.75, 80.37) --
			(136.58, 91.79) --
			(224.24,106.52);
		
		\path[] (224.24, 85.11) --
			(136.58, 80.41) --
			( 92.75, 73.49) --
			( 70.83, 67.21) --
			( 59.87, 62.68) --
			( 54.40, 56.55) --
			( 52.34, 53.77) --
			( 51.66, 52.67) --
			( 50.97, 48.41) --
			( 50.29, 44.11) --
			( 49.60, 38.04) --
			( 49.26, 34.30);
		\end{scope}
		\begin{scope}
		\path[clip] (  0.00,  0.00) rectangle (238.49,115.63);
		\definecolor{drawColor}{RGB}{0,0,0}
		
		\path[draw=drawColor,line width= 0.6pt,line join=round] ( 40.51, 30.69) --
			( 40.51,110.13);
		
		\path[draw=drawColor,line width= 0.6pt,line join=round] ( 41.93,107.67) --
			( 40.51,110.13) --
			( 39.09,107.67);
		\end{scope}
		\begin{scope}
		\path[clip] (  0.00,  0.00) rectangle (238.49,115.63);
		\definecolor{drawColor}{gray}{0.30}
		
		\node[text=drawColor,anchor=base east,inner sep=0pt, outer sep=0pt, scale=  0.88] at ( 35.56, 34.74) {10};
		
		\node[text=drawColor,anchor=base east,inner sep=0pt, outer sep=0pt, scale=  0.88] at ( 35.56, 60.41) {100};
		
		\node[text=drawColor,anchor=base east,inner sep=0pt, outer sep=0pt, scale=  0.88] at ( 35.56, 86.08) {1000};
		\end{scope}
		\begin{scope}
		\path[clip] (  0.00,  0.00) rectangle (238.49,115.63);
		\definecolor{drawColor}{gray}{0.20}
		
		\path[draw=drawColor,line width= 0.6pt,line join=round] ( 37.76, 37.77) --
			( 40.51, 37.77);
		
		\path[draw=drawColor,line width= 0.6pt,line join=round] ( 37.76, 63.44) --
			( 40.51, 63.44);
		
		\path[draw=drawColor,line width= 0.6pt,line join=round] ( 37.76, 89.11) --
			( 40.51, 89.11);
		\end{scope}
		\begin{scope}
		\path[clip] (  0.00,  0.00) rectangle (238.49,115.63);
		\definecolor{drawColor}{RGB}{0,0,0}
		
		\path[draw=drawColor,line width= 0.6pt,line join=round] ( 40.51, 30.69) --
			(232.99, 30.69);
		
		\path[draw=drawColor,line width= 0.6pt,line join=round] (230.53, 29.26) --
			(232.99, 30.69) --
			(230.53, 32.11);
		\end{scope}
		\begin{scope}
		\path[clip] (  0.00,  0.00) rectangle (238.49,115.63);
		\definecolor{drawColor}{gray}{0.20}
		
		\path[draw=drawColor,line width= 0.6pt,line join=round] ( 48.92, 27.94) --
			( 48.92, 30.69);
		
		\path[draw=drawColor,line width= 0.6pt,line join=round] ( 91.72, 27.94) --
			( 91.72, 30.69);
		
		\path[draw=drawColor,line width= 0.6pt,line join=round] (134.52, 27.94) --
			(134.52, 30.69);
		
		\path[draw=drawColor,line width= 0.6pt,line join=round] (177.33, 27.94) --
			(177.33, 30.69);
		
		\path[draw=drawColor,line width= 0.6pt,line join=round] (220.13, 27.94) --
			(220.13, 30.69);
		\end{scope}
		\begin{scope}
		\path[clip] (  0.00,  0.00) rectangle (238.49,115.63);
		\definecolor{drawColor}{gray}{0.30}
		
		\node[text=drawColor,anchor=base,inner sep=0pt, outer sep=0pt, scale=  0.88] at ( 48.92, 19.68) {0};
		
		\node[text=drawColor,anchor=base,inner sep=0pt, outer sep=0pt, scale=  0.88] at ( 91.72, 19.68) {250};
		
		\node[text=drawColor,anchor=base,inner sep=0pt, outer sep=0pt, scale=  0.88] at (134.52, 19.68) {500};
		
		\node[text=drawColor,anchor=base,inner sep=0pt, outer sep=0pt, scale=  0.88] at (177.33, 19.68) {750};
		
		\node[text=drawColor,anchor=base,inner sep=0pt, outer sep=0pt, scale=  0.88] at (220.13, 19.68) {1000};
		\end{scope}
		\begin{scope}
		\path[clip] (  0.00,  0.00) rectangle (238.49,115.63);
		\definecolor{drawColor}{RGB}{0,0,0}
		
		\node[text=drawColor,anchor=base,inner sep=0pt, outer sep=0pt, scale=  1.10] at (136.75,  7.64) {$d$};
		\end{scope}
		\begin{scope}
		\path[clip] (  0.00,  0.00) rectangle (238.49,115.63);
		\definecolor{drawColor}{RGB}{0,0,0}
		
		\node[text=drawColor,rotate= 90.00,anchor=base,inner sep=0pt, outer sep=0pt, scale=  1.10] at ( 13.08, 70.41) {time (ms)};
		\end{scope}
		\begin{scope}
		\path[clip] (  0.00,  0.00) rectangle (238.49,115.63);
		
		\path[] ( 37.04, 71.03) rectangle (109.43,125.40);
		\end{scope}
		\begin{scope}
		\path[clip] (  0.00,  0.00) rectangle (238.49,115.63);
		\definecolor{drawColor}{RGB}{247,192,26}
		
		\path[draw=drawColor,line width= 0.6pt,line join=round] ( 43.98,112.67) -- ( 55.55,112.67);
		\end{scope}
		\begin{scope}
		\path[clip] (  0.00,  0.00) rectangle (238.49,115.63);
		\definecolor{fillColor}{RGB}{247,192,26}
		
		\path[fill=fillColor] ( 49.76,112.67) circle (  1.96);
		\end{scope}
		\begin{scope}
		\path[clip] (  0.00,  0.00) rectangle (238.49,115.63);
		\definecolor{drawColor}{RGB}{78,155,133}
		
		\path[draw=drawColor,line width= 0.6pt,dash pattern=on 2pt off 2pt ,line join=round] ( 43.98, 98.22) -- ( 55.55, 98.22);
		\end{scope}
		\begin{scope}
		\path[clip] (  0.00,  0.00) rectangle (238.49,115.63);
		\definecolor{fillColor}{RGB}{78,155,133}
		
		\path[fill=fillColor] ( 49.76,101.27) --
			( 52.41, 96.69) --
			( 47.12, 96.69) --
			cycle;
		\end{scope}
		\begin{scope}
		\path[clip] (  0.00,  0.00) rectangle (238.49,115.63);
		\definecolor{drawColor}{RGB}{37,122,164}
		
		\path[draw=drawColor,line width= 0.6pt,dash pattern=on 4pt off 2pt ,line join=round] ( 43.98, 83.76) -- ( 55.55, 83.76);
		\end{scope}
		\begin{scope}
		\path[clip] (  0.00,  0.00) rectangle (238.49,115.63);
		\definecolor{fillColor}{RGB}{37,122,164}
		
		\path[fill=fillColor] ( 47.80, 81.80) --
			( 51.73, 81.80) --
			( 51.73, 85.72) --
			( 47.80, 85.72) --
			cycle;
		\end{scope}
		\begin{scope}
		\path[clip] (  0.00,  0.00) rectangle (238.49,115.63);
		\definecolor{drawColor}{RGB}{0,0,0}
		
		\node[text=drawColor,anchor=base west,inner sep=0pt, outer sep=0pt, scale=  0.80] at ( 62.49,109.91) {LVE (ACP)};
		\end{scope}
		\begin{scope}
		\path[clip] (  0.00,  0.00) rectangle (238.49,115.63);
		\definecolor{drawColor}{RGB}{0,0,0}
		
		\node[text=drawColor,anchor=base west,inner sep=0pt, outer sep=0pt, scale=  0.80] at ( 62.49, 95.46) {LVE (CP)};
		\end{scope}
		\begin{scope}
		\path[clip] (  0.00,  0.00) rectangle (238.49,115.63);
		\definecolor{drawColor}{RGB}{0,0,0}
		
		\node[text=drawColor,anchor=base west,inner sep=0pt, outer sep=0pt, scale=  0.80] at ( 62.49, 81.01) {VE};
		\end{scope}
	\end{tikzpicture}
	\caption{Average query times and their standard deviation of \ac{lve} on the output of \ac{cp}, \ac{lve} on the output of \ac{cpr}, and \ac{ve} on the initial (propositional) \ac{fg} for input \acp{fg} where about three percent of the factors have permuted arguments.}
	\label{fig:plot-results-inter-p=3}
\end{figure}
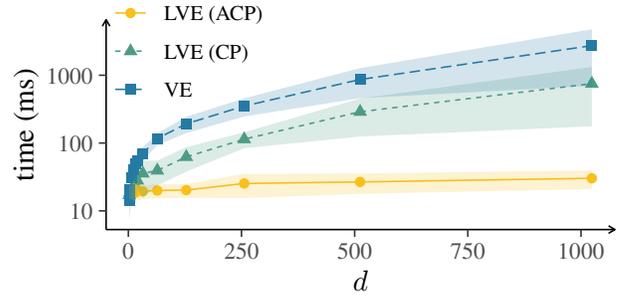

\Cref{fig:plot-results-inter-p=3} contains the results for \acp{fg} where the arguments of about three percent of the factors (randomly chosen) are permuted and there are no commutative factors.
For each domain size $d = 2, 4, 8, 12, 16, 20, 32, 64, 128, 256, 512, 1024$, there are between $5d + 1$ and $d \cdot \lfloor \log_2(d) \rfloor + 2d + 1$ Boolean \acp{rv} and between $2d$ and $d \cdot \lfloor \log_2(d) \rfloor + d + 1$ factors in the input \acp{fg}.
The maximum number of arguments of a factor is seven, i.e., the largest table contains $2^7$ rows and thus, we are able to evaluate larger values for $d$ as well (in the previous scenario, the number of rows increased exponentially with $d$).
Again, we evaluate multiple \acp{fg} for each choice of $d$ by posing two queries per \ac{fg} and then report the average run time and standard deviation over all queries for that choice of $d$.
The results depicted in \cref{fig:plot-results-inter-p=3} show that \ac{cpr} can easily handle large domains, indicating that \ac{cpr} is expected to handle \acp{fg} with tens of thousands of nodes without a hassle.
In particular, \ac{cpr} is able to achieve speedups of up to factor 25 compared to the state of the art \ac{cp}, which, again, faces serious scalability issues.
We give additional results for higher percentages of permuted factors in \cref{appendix:more_eval} and also provide an evaluation investigating when the additional offline overhead of \ac{cpr} amortises there.

\section{Conclusion} \label{sec:conclusion}
We introduce the \ac{cpr} algorithm providing a fully-fledged pipeline from propositional \ac{fg} to a valid \ac{pfg} independent of a specific inference algorithm.
\Ac{cpr} builds on the well-known \ac{cp} algorithm, which does not handle commutative factors (i.e., factors with inherent symmetries resulting in mapping input arguments to unique values regardless of the order of those arguments) and is dependent on the order of the factors' argument lists.
By using \acp{crv} and histograms, \ac{cpr} is able to efficiently encode commutative factors and handle factors representing identical potentials independent of the order of their argument lists.
\Ac{cpr} not only provides significant speedups for online query answering but also solves serious scalability issues of \ac{cp} in practice.

A fundamental problem for future research is to learn a \ac{pfg} directly from a relational database without having to construct the ground model first.
In this regard, \acp{crv} provide a crucial component to keep the size of the \ac{pfg} small.

\section*{Acknowledgements}
This work is supported by the BMBF project AnoMed 16KISA057 and 16KISA050K.
The authors also thank the anonymous reviewers for their valuable feedback.

\bibliography{./references.bib}

\clearpage

\appendix

\begin{strip}
	\centering
	\textbf{\LARGE Appendix}
	\vspace*{1cm}
\end{strip}

\acresetall

\section{Formal Description of the Colour Passing Algorithm} \label{appendix:cp_algo}
The \ac{cp} algorithm (originally named \enquote{CompressFactorGraph})~\citep{Kersting2009a,Ahmadi2013a} solves the problem of constructing a lifted representation for an input \ac{fg}.
The idea of \ac{cp} is to first find symmetries in a propositional \ac{fg} and then group together symmetric subgraphs.
\Ac{cp} looks for symmetries based on potentials of factors, on ranges and evidence of \acp{rv}, as well as on the graph structure by passing around colours.
\Cref{alg:cp} gives a formal description of the \ac{cp} algorithm, which proceeds as follows.
As an initialisation step, each variable node (\ac{rv}) is assigned a colour depending on its range and evidence, meaning that \acp{rv} with identical ranges and identical evidence are assigned the same colour, and each factor is assigned a colour depending on its potentials, i.e., factors with the same potentials get the same colour.
Afterwards, the colour passing procedure begins.
The colours are first passed from every variable node to its neighbouring factor nodes and each factor $\phi$ collects all colours of neighbouring \acp{rv} in the order of their appearance in the argument list of $\phi$.
Based on the collected colours and their own colour, factors are grouped together and reassigned a new colour (to reduce communication overhead).
Then, colours are passed from factor nodes to variable nodes and each message from a factor $\phi$ to a \ac{rv} $R$ includes the position $p(R, \phi)$ of $R$ in $\phi$ in the message.
Again, based on the collected colours and their own colour, \acp{rv} are grouped together and reassigned a new colour.
The colour passing procedure is iterated until groupings do not change anymore and in the end, all \acp{rv} and factors, respectively, are grouped together based on their colour signatures (that is, the messages they received from their neighbours plus their own colour).
\begin{algorithm}[t]
	\caption{Colour Passing}
	\label{alg:cp}
	\alginput{An \ac{fg} $G$ with \acp{rv} $\boldsymbol R = \{R_1, \dots, R_n\}$, \\\hspace*{\algorithmicindent} factors $\boldsymbol \Phi = \{\phi_1, \dots, \phi_m\}$, and evidence \\\hspace*{\algorithmicindent} $\boldsymbol E = \{R_1 = r_1, \ldots, R_k = r_k\}$.} \\
	\algoutput{Groups of \acp{rv} and factors, respectively.}
	\begin{algorithmic}[1]
		\State Assign each $R_i$ a colour according to $\mathcal R(R_i)$ and $\boldsymbol E$\;
		\State Assign each $\phi_i$ a colour according to its potentials\;
		\Repeat
			\For{each factor $\phi \in \boldsymbol \Phi$}
				\State $signature_{\phi} \gets [\,]$\;
				\For{each \ac{rv} $R \in neighbours(G, \phi)$}
					\State\Comment{In order of appearance in $\phi$}\;
					\State $append(signature_{\phi}, R.colour)$\;
				\EndFor
				\State $append(signature_{\phi}, \phi.colour)$\;
			\EndFor
			\State Group together all $\phi$s with the same signature\;
			\State Assign each such cluster a unique colour\;
			\State Set $\phi.colour$ correspondingly for all $\phi$s\;
			\For{each \ac{rv} $R \in \boldsymbol R$}
				\State $signature_{R} \gets [\,]$\;
				\For{each factor $\phi \in neighbours(G, R)$}
					\State $append(signature_{R}, (\phi.colour, p(R, \phi)))$\;
				\EndFor
				\State Sort $signature_{R}$ according to colour\;
				\State $append(signature_{R}, R.colour)$\;
			\EndFor
			\State Group together all $R$s with the same signature\;
			\State Assign each such cluster a unique colour\;
			\State Set $R.colour$ correspondingly for all $R$s\;
		\Until{grouping does not change}
	\end{algorithmic}
\end{algorithm}

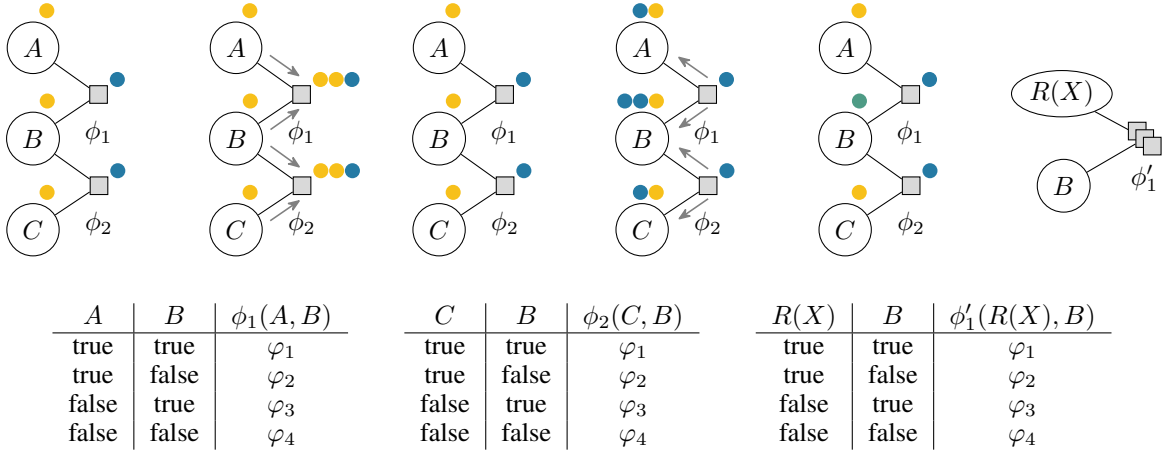
\begin{figure*}[t]
	\centering
	\input{./figures/cp_example.tex}
	\caption{A visualisation of the steps undertaken by the \ac{cp} algorithm on an input \ac{fg} with only Boolean \acp{rv} and no evidence (left). Colours are first passed from variable nodes to factor nodes, followed by a recolouring, and then passed back from factor nodes to variable nodes, again followed by a recolouring. The procedure is iterated until convergence and the \ac{pfg} corresponding to the resulting groupings of \acp{rv} and factors is depicted on the right. This example builds on Figure 2 from~\protect\citep{Ahmadi2013a}.}
	\label{fig:cp_example}
\end{figure*}

\Cref{fig:cp_example} illustrates the \ac{cp} algorithm on an example \ac{fg}~\citep{Ahmadi2013a}.
In this example, $A$, $B$, and $C$ are Boolean \acp{rv} with no evidence and thus, they all receive the same colour (e.g., $\mathrm{yellow}$).
As the potentials of $\phi_1$ and $\phi_2$ are identical, $\phi_1$ and $\phi_2$ are assigned the same colour as well (e.g., $\mathrm{blue}$).
The colour passing then starts from variable nodes to factor nodes, that is, $A$ and $B$ send their colour ($\mathrm{yellow}$) to $\phi_1$ and $B$ and $C$ send their colour ($\mathrm{yellow}$) to $\phi_2$.
$\phi_1$ and $\phi_2$ are then recoloured according to the colours they received from their neighbours to reduce the communication overhead.
Since $\phi_1$ and $\phi_2$ received identical colours (two times the colour $\mathrm{yellow}$), they are assigned the same colour during recolouring.
Afterwards, the colours are passed from factor nodes to variable nodes and this time not only the colours but also the position of the \acp{rv} in the argument list of the corresponding factor are shared.
Consequently, $\phi_1$ sends a tuple $(\mathrm{blue}, 1)$ to $A$ and a tuple $(\mathrm{blue}, 2)$ to $B$, and $\phi_2$ sends a tuple $(\mathrm{blue}, 2)$ to $B$ and a tuple $(\mathrm{blue}, 1)$ to $C$ (positions are not shown in \cref{fig:cp_example}).
Since $A$ and $C$ are both at position one in the argument list of their respective neighbouring factor, they receive identical messages and are recoloured with the same colour.
$B$ is assigned a different colour during recolouring than $A$ and $C$ because $B$ received different messages than $A$ and $C$.
The groupings do not change in further iterations and hence the algorithm terminates.
The output can be represented by the lifted representation shown on the right in \cref{fig:cp_example} where both $A$ and $C$ as well as $\phi_1$ and $\phi_2$ have been grouped together.

\section{Introducing Logical Variables in Groups of Random Variables} \label{appendix:logvars}
After running the colour passing procedure, \acp{rv} and factors belong to groups depending on their assigned colour.
In its original form, the \ac{cp} algorithm is used to run a lifted belief propagation algorithm and thus there is no set of rules specified on how to construct a \ac{pfg} given the groups after running \ac{cp}~\citep{Kersting2009a,Ahmadi2013a}.
In the following, we describe how the resulting groups after running \ac{cpr} can be used to obtain a \ac{pfg} entailing equivalent semantics as the initial \ac{fg} used as input for \ac{cpr}.

As lifted inference algorithms such as \ac{lve} and the \ac{ljt} algorithm have been shown to be complete\footnote{A lifted inference algorithm $\mathcal A$ is \emph{complete} for a model class $\mathcal M$ if $\mathcal A$ is domain lifted (that is, $\mathcal A$ runs in polynomial time with respect to domain sizes in $\mathcal M$) for each query, evidence, and \ac{pfg} in $\mathcal M$~\citep{VanDenBroeck2011a}.} for \acp{pfg} containing only \acp{pf} with at most two \acp{lv} and for \acp{pfg} containing only \acp{prv} having at most one \ac{lv}~\citep{Taghipour2013b,Braun2020a}, we concentrate on these two model classes, referred to as $\mathcal M^{2lv}$ and $\mathcal M^{prv1}$, respectively.
\begin{definition}
	The model class $\mathcal M^{2lv}$ contains every \ac{pfg} in which each \ac{pf} contains at most two \acp{lv}.
\end{definition}
\begin{definition}
	The model class $\mathcal M^{prv1}$ contains every \ac{pfg} in which each \ac{prv} has at most one \ac{lv}.
\end{definition}
We refer to $\mathcal M^{2lv} \cup \mathcal M^{prv1}$ as the \emph{domain-liftable} fragment.
$\mathcal M^{2lv}$ and $\mathcal M^{prv1}$ are able to model a variety of relations and hence provide sufficient expressiveness for most practical applications.
We now describe how a \ac{pfg} is constructed using the groups obtained by running \ac{cpr} for the domain-liftable fragment $\mathcal M^{2lv} \cup \mathcal M^{prv1}$.

Let $\phi'(R'_1(X_{1,1}, \dots, X_{1,k}), \dots, R'_j(X_{j,1}, \dots, X_{j,k}))$ be the new \ac{pf}, build from $\boldsymbol F = \{\phi_1(R_{1,1}, \dots, R_{1,s}), \allowbreak \dots, \allowbreak \phi_{\ell}(R_{\ell,1}, \dots, R_{\ell,s})\}$ (i.e., $\boldsymbol F$ is a group of $\ell$ factors having the same colour after running \ac{cpr}).
To obtain a correct mapping from $\boldsymbol F$ to $\phi'$, it must hold that $gr(\phi') = \boldsymbol F$ and hence $\abs{gr(\phi')} = \abs{\boldsymbol F} = \ell$.
Recall that $gr(\phi')$ refers to the groundings (i.e., the set of all instances) of $\phi'$.
\begin{example}
	Consider the \ac{pf} $\phi(S(X), T(Y))$ with two \acp{prv} $S(X)$ and $T(Y)$.
	Let $\mathcal D(X) = \{x_1, x_2\}$ and $\mathcal D(Y) = \{y_1, y_2\}$.
	Then, we have $gr(S) = \{S(x_1), S(x_2)\}$, $gr(T) = \{T(y_1), T(y_2)\}$, and $gr(\phi) = \{\phi(S(x_1), T(y_1)), \allowbreak \phi(S(x_1), T(y_2)), \allowbreak \phi(S(x_2), T(y_1)), \allowbreak \phi(S(x_2), T(y_2))\}$.
\end{example}
\begin{example}
	Consider the \ac{pf} $\phi(S(X), T(X))$ with two \acp{prv} $S(X)$ and $T(X)$ sharing a single \ac{lv} $X$.
	Let $\mathcal D(X) = \{x_1, x_2\}$.
	Then, it holds that $gr(S) = \{S(x_1), S(x_2)\}$, $gr(T) = \{T(x_1), T(x_2)\}$ and $gr(\phi) = \{\phi(S(x_1), T(x_1)), \allowbreak \phi(S(x_2), T(x_2))\}$.
\end{example}
Note that it is also possible for a \ac{prv} to have no \acp{lv} at all.
A \ac{prv} $R'$ is parameterless if the group it represents contains only a single \ac{rv}---in this case, it holds that $gr(R') = \{R'\}$ and no \ac{lv} needs to be introduced.

Now, we are ready to prove \cref{th:lv_introduction}, i.e., to prove that \ac{cpr} introduces the \acp{lv} correctly to obtain a valid \ac{pfg} entailing equivalent semantics as the initial \ac{fg} for the domain-liftable fragment $\mathcal M^{2lv} \cup \mathcal M^{prv1}$.
In particular, as \ac{cpr} introduces \acp{lv} as specified in \cref{def:lv_introduction}, we have to prove that applying \cref{def:lv_introduction} for the introduction of \acp{lv} yields a valid \ac{pfg} entailing equivalent semantics as the initial \ac{fg} for the domain-liftable fragment.
\begin{proof}[Proof of \Cref{th:lv_introduction}]
	To match the notation in \cref{def:lv_introduction}, let $\phi'(R'_1(X_{1,1}, \ldots, X_{1,k}), \allowbreak \ldots, \allowbreak R'_j(X_{j,1}, \ldots, X_{j,k}))$ be a new \ac{pf}, build from $\boldsymbol F = \{\phi_1(R_{1,1}, \ldots, R_{1,s}), \allowbreak \ldots, \allowbreak \phi_{\ell}(R_{\ell,1}, \ldots, R_{\ell,s})\}$ and let $\boldsymbol S = \{S'_1, \ldots, S'_z\}$ denote the subset of $\phi'$'s arguments with more than one grounding.
	We prove the correctness of \cref{item:single_lv_shared,item:single_lv_distinct,item:two_lvs} given in \cref{def:lv_introduction}:
	\begin{enumerate}
		\item We first show that it is not possible for two \acp{prv} $S'_a \in \boldsymbol S$ and $S'_b \in \boldsymbol S$, $a \neq b$, to have different \acp{lv}. For the sake of contradiction, assume that there are $S'_a(X) \in \boldsymbol S$ and $S'_b(Y) \in \boldsymbol S$ with $X \neq Y$.
		Then, it holds that $\abs{gr(\phi')} \geq \abs{gr(S'_a)} \cdot \abs{gr(S'_b)}$ because the groundings of $\phi'$ include all combinations of groundings for \acp{prv} with distinct \acp{lv} appearing in $\phi'$.
		A contradiction to our assumption that $\abs{gr(\phi')} = \abs{\boldsymbol F} = \abs{gr(S'_a)} = \abs{gr(S'_b)}$ with $\abs{gr(S'_a)} > 1$ and $\abs{gr(S'_b)} > 1$.

		Now, we show that if there were a \ac{prv} $S'_a \in \boldsymbol S$ with two \acp{lv}, all other \acp{prv} in $\phi'$ must have the same two \acp{lv} and in this case, it is possible to equivalently represent the groundings using a single shared \ac{lv} for all \ac{prv} in $\boldsymbol S$.
		If there is a \ac{prv} $S'_a(X,Y) \in \boldsymbol S$ with two \acp{lv} $X$ and $Y$, $X \neq Y$, there can be no \acp{lv} other than $X$ and $Y$ in $\phi'$ due to the restrictions of the domain-liftable fragment.
		Consequently, as $\abs{gr(S'_a)} = \abs{\mathcal D(X)} \cdot \abs{\mathcal D(Y)}$ holds and we assumed that all $S'_i \in \boldsymbol S$ have the same number of groundings, all $S'_i \in \boldsymbol S$ must have the same \acp{lv} $X$ and $Y$.
		In case all $S'_i \in \boldsymbol S$ have the same \acp{lv} $X$ and $Y$, we can equivalently represent the groundings using a single shared \ac{lv} $Z$ with $\abs{\mathcal D(Z)} = \abs{\mathcal D(X)} \cdot \abs{\mathcal D(Y)}$ for all $S'_i \in \boldsymbol S$.
		\item \label{item:single_lv_distinct_proof} We begin by proving that all \acp{prv} in $\boldsymbol S$ have a single \ac{lv}.
		For the sake of contradiction, assume that there is a \ac{prv} $S'_a(X,Y) \in \boldsymbol S$ with two \acp{lv} $X$ and $Y$ such that $X \neq Y$.
		Then, $X$ and $Y$ are the only \acp{lv} in $\phi'$ due to the restrictions of the domain-liftable fragment.
		It follows that $\abs{\boldsymbol F} = \abs{gr(\phi')} = \abs{\mathcal D(X)} \cdot \abs{\mathcal D(Y)} = \abs{gr(S'_a)}$, which is a contradiction to our assumption that $\forall S'_i \in \boldsymbol S$: $\abs{\boldsymbol F} \neq \abs{gr(S'_i)}$.
		Next, we prove the condition for two \acp{prv} to share the same \ac{lv} in two directions.

		For the first direction, we show that if two \acp{prv} $S'_a \in \boldsymbol S$ and $S'_b \in \boldsymbol S$ share the same \ac{lv} $X$, then $\abs{gr(S'_a)} = \abs{gr(S'_b)}$ holds and a bijection $\tau$ satisfying the specified condition exists.
		If $S'_a$ and $S'_b$ share the same \ac{lv} $X$, $\abs{gr(S'_a)} = \abs{\mathcal D(X)} = \abs{gr(S'_b)}$ holds.
		Further, as $gr(S'_a) = \{S'_a(x) \mid x \in \mathcal D(X)\}$ and $gr(S'_b) = \{S'_b(x) \mid x \in \mathcal D(X)\}$, choosing $\tau$ such that $S'_a(x)$ is mapped to $S'_b(x)$ for each $x \in \mathcal D(X)$ satisfies the given condition.

		For the second direction, we show that if two \acp{prv} $S'_a \in \boldsymbol S$ and $S'_b \in \boldsymbol S$ have different \acp{lv} $X$ and $Y$, then no bijection $\tau$ satisfying the specified condition exists.
		Let $S'_a(X) \in \boldsymbol S$ and $S'_b(Y) \in \boldsymbol S$ with $X \neq Y$.
		As $X \neq Y$, each combination of $S'_a(x)$ and $S'_b(y)$ with $x \in \mathcal D(X)$ and $y \in \mathcal D(Y)$ appears in exactly one factor in $\boldsymbol F$ and thus, for each pair of \acp{rv} $(S'_a(x), S'_b(y)) \in gr(S'_a) \times gr(S'_b)$, it holds that $\mathcal F(S'_a(x)) \symdiff \mathcal F(S'_b(y)) \neq \emptyset$ where $\symdiff$ denotes the symmetric difference\footnote{The symmetric difference of two sets $A$ and $B$ is defined as $A \symdiff B = (A \setminus B) \cup (B \setminus A)$.} of two sets.
		Therefore, it is not possible for a bijection $\tau$ satisfying the specified condition to exist.
		\item We first prove that there are both \acp{prv} with one \ac{lv} as well as \acp{prv} with two \acp{lv} in $\boldsymbol S$.
		By definition, it holds that $\abs{\boldsymbol F} = \prod_{X \in lv(\phi')} \abs{\mathcal D(X)}$.
		Consequently, all \acp{prv} $S'_a \in \boldsymbol S$ with $\abs{\boldsymbol F} = \abs{gr(S'_a)} = \prod_{X \in lv(S'_a)} \abs{\mathcal D(X)}$ contain all \acp{lv} occurring in $\phi'$.
		Hence, each $S'_b \in \boldsymbol S$ with $\abs{\boldsymbol F} \neq \abs{gr(S'_b)}$ cannot contain all \acp{lv} occurring in $\phi'$ and must therefore contain less \acp{lv} than the $S'_a \in \boldsymbol S$.
		Due to the restrictions of the domain-liftable fragment, we know that all $S'_a \in \boldsymbol S$ then have two \acp{lv} and all $S'_b \in \boldsymbol S$ have one \ac{lv}.
		Furthermore, there are exactly two distinct \acp{lv} in $\phi'$ and hence, all $S'_a \in \boldsymbol S$ share the same two \acp{lv} $X_1$ and $X_2$ and all $S'_b \in \boldsymbol S$ have either $X_1$ or $X_2$ as their \ac{lv}.
		\Cref{item:single_lv_distinct_proof} completes the proof for the conditions for two \acp{prv} to share the same \ac{lv}.
	\end{enumerate}
	The correctness for introducing the domain sizes follows by definition as the number of groundings of a \ac{prv} equals the size of the group of \acp{rv} it represents.
\end{proof}
A visualisation of the three cases is given in \cref{fig:example_lv_intro.tex}.
In \cref{fig:example_lv_intro_shared.tex}, it holds that $\abs{\boldsymbol F} = \abs{gr(S'_1)} = \abs{gr(S'_2)} = 2$ and thus $S'_1$ and $S'_2$ share the same \ac{lv} $X$ with $\abs{\mathcal D(X)} = 2$.
\Cref{fig:example_lv_intro_distinct.tex} shows an example of a \ac{pf} containing two \acp{prv} with distinct \acp{lv} $X$ and $Y$ with $\abs{\mathcal D(X)} = 2$ and $\abs{\mathcal D(Y)} = 2$ as $4 = \abs{\boldsymbol F} \neq \abs{gr(S'_1)} = \abs{gr(S'_2)} = 2$.
A similar situation is depicted in \cref{fig:example_lv_intro_double.tex} where a single \ac{prv} has two \acp{lv} $X$ and $Y$ with $\abs{\mathcal D(X)} = 2$ and $\abs{\mathcal D(Y)} = 2$.
Thus, it holds that $4 = \abs{\boldsymbol F} = \abs{gr(S'_2)} \neq \abs{gr(S'_1)} = 2$.

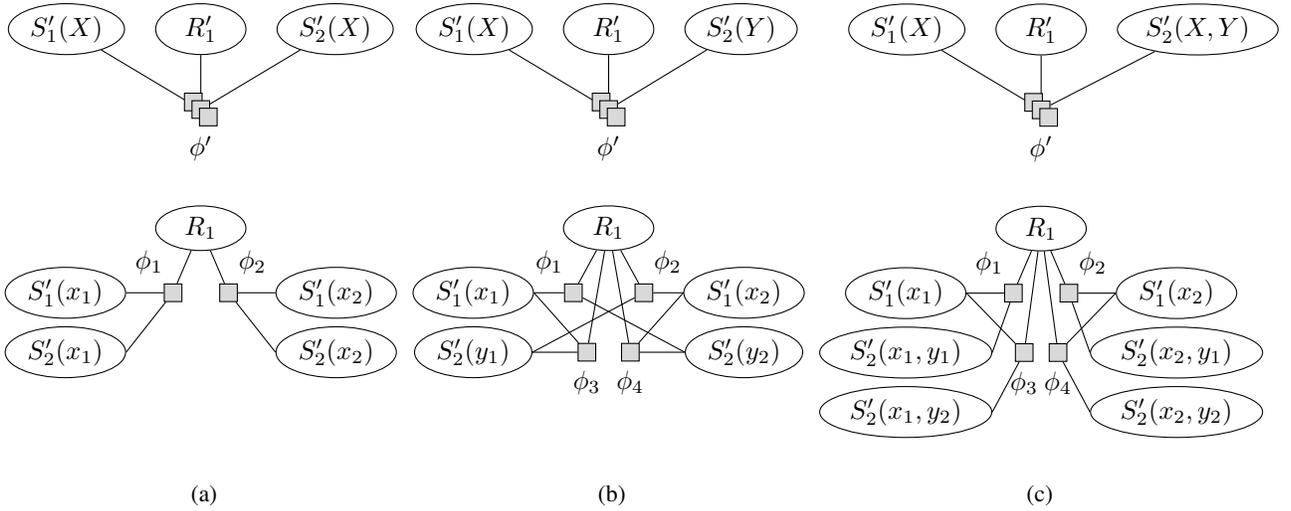
\begin{figure*}[t]
	\centering
	\begin{subfigure}{0.3\linewidth}
		\centering
		\input{./figures/example_lv_intro_shared.tex}
		\caption{}
		\label{fig:example_lv_intro_shared.tex}
	\end{subfigure}
	\begin{subfigure}{0.3\linewidth}
		\centering
		\input{./figures/example_lv_intro_distinct.tex}
		\caption{}
		\label{fig:example_lv_intro_distinct.tex}
	\end{subfigure}
	\begin{subfigure}{0.33\linewidth}
		\centering
		\input{./figures/example_lv_intro_double.tex}
		\caption{}
		\label{fig:example_lv_intro_double.tex}
	\end{subfigure}
	\caption{A visualisation of the three cases given in \cref{def:lv_introduction} based on small examples. The domains of the two \acp{lv} $X$ and $Y$ are given by $\mathcal D(X) = \{x_1, x_2\}$ and $\mathcal D(Y) = \{y_1, y_2\}$, respectively. In (a), there is a \ac{pf} $\phi'$ with a single \ac{lv} $X$ shared across two \acp{prv} $S'_1$ and $S'_2$, (b) shows a \ac{pf} $\phi'$ containing two \acp{prv} $S'_1$ and $S'_2$ with distinct \acp{lv} $X$ and $Y$, and (c) displays a \ac{pf} $\phi'$ with a \ac{prv} $S'_2$ containing two \acp{lv} $X$ and $Y$. The groundings of the \acp{pf} are illustrated below the respective \ac{pf}. In all three cases, the \ac{pf} $\phi'$ contains an additional parameterless \ac{prv} $R'_1$.}
	\label{fig:example_lv_intro.tex}
\end{figure*}

Note that \cref{def:lv_introduction} can also be applied on the resulting groups obtained from running \ac{cp} instead of \ac{cpr} to yield a valid \ac{pfg}.
So far, we have proven that \ac{cpr} introduces \acp{lv} correctly.
Next, we give the conditions for \ac{cpr} to count convert a \ac{prv} and then show that the introduction of \acp{crv} is also handled correctly by \ac{cpr}.

To check whether a \ac{crv} is required in the argument list of $\phi'(R'_1(X_{1,1}, \ldots, X_{1,k}), \allowbreak \ldots, \allowbreak R'_j(X_{j,1}, \ldots, X_{j,k}))$, it is sufficient to compare the number $j$ of arguments of $\phi'$ to the number $s$ of arguments of the $\phi_i \in \boldsymbol F = \{\phi_1(R_{1,1}, \ldots, R_{1,s}), \allowbreak \ldots, \allowbreak \phi_{\ell}(R_{\ell,1}, \ldots, R_{\ell,s})\}$.
Note that all $\phi_i$ have the same colour and thus the same number of arguments.
However, it is possible for $\phi'$ to have less arguments than each $\phi_i$ if there are \acp{rv} in the argument lists of each $\phi_i$ that are in the same group.

If the number of arguments of $\phi'$ is equal to the number of arguments of all $\phi_i$ (i.e., $j = s$), there is no need to introduce a \ac{crv} in $\phi'$ as the potentials of the $\phi_i$ can be copied into $\phi'$ (potentials are identical for all $\phi_i$ as they are in the same group).
However, if $\phi'$ has less arguments than the $\phi_i$ (i.e., $j < s$), a \ac{crv} is required to equivalently represent the potentials of the $\phi_i$ in $\phi'$.
If $j < s$, there are at least two \acp{rv} in the argument lists of each $\phi_i$ that are in the same group and hence these \acp{rv} must be in the set of commutative arguments of the $\phi_i$ (otherwise, they would have received different messages and thus would not be in the same group).
Therefore, if $j < s$, all $\phi_i$ are guaranteed to be commutative with respect to their arguments that are in the same group and thus a \ac{crv} can be used to equivalently represent the potentials of the $\phi_i$ in $\phi'$.

As a final remark regarding the introduction of \acp{crv}, we note that it is also conceivable to recursively check the remaining argument lists of the $\phi_i \in \boldsymbol F$ for commutativity after introducing a \ac{crv}.
However, it is not immediately clear how the preconditions for count conversion~\citep{Taghipour2013c} are ensured then.
As the preconditions get more restrictive for additional \acp{crv} in a single \ac{pf} and the total number of arguments of a factor is usually small in practice, we expect that there are very few scenarios in which the extra complexity of searching for additional \acp{crv} pays off.

\section{Further Experimental Results} \label{appendix:more_eval}
In addition to the experimental results provided in \cref{sec:eval}, we also report further results for input \acp{fg} containing different proportions of commutative factors and factors having permuted arguments.
Again, we evaluate the impact of commutative factors and factors with permuted arguments on query times separately by comparing the average query times of \ac{lve} on the resulting \ac{pfg} after running \ac{cpr}, denoted as \acs{lve} (\acs{cpr}), of \ac{lve} on the resulting \ac{pfg} after running \ac{cp}, denoted as \acs{lve} (\acs{cp}), and of \ac{ve} on the initial \ac{fg}.

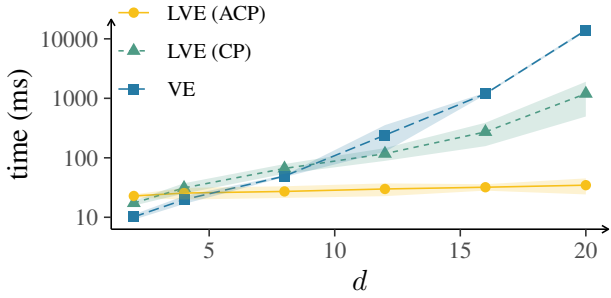
\begin{figure}[t]
	\centering
	\input{./figures/plot-results-intra-k=3.tex}
	\caption{Average query times and their standard deviation of \ac{lve} on the output of \ac{cp}, \ac{lve} on the output of \ac{cpr}, and \ac{ve} on the initial (propositional) \ac{fg} for input \acp{fg} containing three commutative factors.}
	\label{fig:plot-results-intra-k=3}
\end{figure}
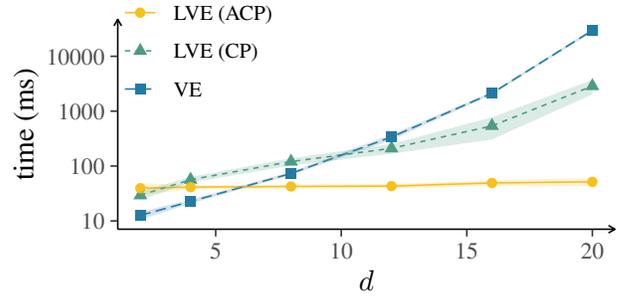
\begin{figure}[t]
	\centering
	\input{./figures/plot-results-intra-k=7.tex}
	\caption{Average query times and their standard deviation of \ac{lve} on the output of \ac{cp}, \ac{lve} on the output of \ac{cpr}, and \ac{ve} on the initial (propositional) \ac{fg} for input \acp{fg} containing seven commutative factors.}
	\label{fig:plot-results-intra-k=7}
\end{figure}

\Cref{fig:plot-results-intra-k=3,fig:plot-results-intra-k=7} depict the results for \acp{fg} containing between eight and 104 factors of which $k=3$ and $k=7$ factors are commutative, respectively.
In total, there are between $2d + (k - 2) \cdot (d + 1) + 2$ and $d \cdot \lfloor \log_2(d) \rfloor + 2d + 2 + (k - 1) \mathbin{/} 2 \cdot (d + 1)$ Boolean \acp{rv} and between $d \cdot \lfloor \log_2(d) \rfloor + d + k + 1$ and $2d + (k - 2) \cdot (d + 1) + 2$ factors for each domain size $d = 2, 4, 8, 12, 16, 20$.
There are no symmetries between factors which cannot be detected by \ac{cp} (i.e., there are no permuted arguments) and the maximum number of arguments of a factor is $d + 1$.
For each choice of $d$, we again measure the run times for two queries on multiple \acp{fg} and report the average query times and their standard deviation.
The results shown in \cref{fig:plot-results-intra-k=3,fig:plot-results-intra-k=7} are similar to the results given in \cref{fig:plot-results-intra-k=1} (y-axes are log-scaled again).
Overall, the run times are generally slightly higher than for \acp{fg} containing only one commutative factor because the \acp{fg} are a bit larger.
The results show that \ac{cpr} handles larger \acp{fg} without any problems while \ac{cp} runs again into scalability issues even for rather small domain sizes.
As expected, \Ac{ve} is again the slowest of all algorithms.

\begin{figure}[t]
	\centering
	\input{./figures/plot-results-inter-p=5.tex}
	\caption{Average query times and their standard deviation of \ac{lve} on the output of \ac{cp}, \ac{lve} on the output of \ac{cpr}, and \ac{ve} on the initial (propositional) \ac{fg} for input \acp{fg} where about five percent of the factors have permuted arguments.}
	\label{fig:plot-results-inter-p=5}
\end{figure}
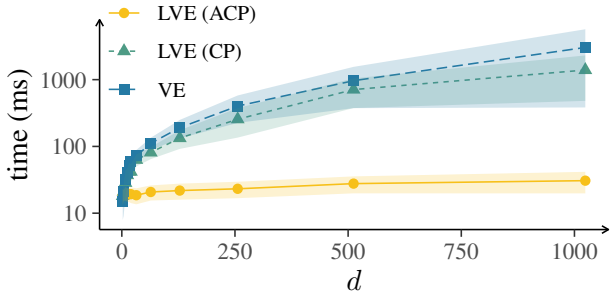
\begin{figure}[t]
	\centering
	\input{./figures/plot-results-inter-p=10.tex}
	\caption{Average query times and their standard deviation of \ac{lve} on the output of \ac{cp}, \ac{lve} on the output of \ac{cpr}, and \ac{ve} on the initial (propositional) \ac{fg} for input \acp{fg} where about ten percent of the factors have permuted arguments.}
	\label{fig:plot-results-inter-p=10}
\end{figure}
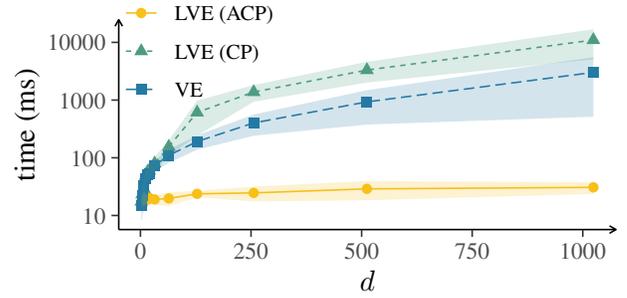
\begin{figure}[t]
	\centering
	\input{./figures/plot-results-inter-p=15.tex}
	\caption{Average query times and their standard deviation of \ac{lve} on the output of \ac{cp}, \ac{lve} on the output of \ac{cpr}, and \ac{ve} on the initial (propositional) \ac{fg} for input \acp{fg} where about 15 percent of the factors have permuted arguments.}
	\label{fig:plot-results-inter-p=15}
\end{figure}

Further experimental results for input \acp{fg} containing symmetries between factors with permuted argument lists can be found in \cref{fig:plot-results-inter-p=5,fig:plot-results-inter-p=10,fig:plot-results-inter-p=15} (y-axes are log-scaled again).
In the input \acp{fg} depicted in \cref{fig:plot-results-inter-p=5,fig:plot-results-inter-p=10,fig:plot-results-inter-p=15}, about five, ten, and 15 percent of the factors (randomly chosen) have permuted argument lists, respectively, and there are no commutative factors.
As in \cref{fig:plot-results-inter-p=3}, there are between $5d + 1$ and $d \cdot \lfloor \log_2(d) \rfloor + 2d + 1$ Boolean \acp{rv} and between $2d$ and $d \cdot \lfloor \log_2(d) \rfloor + d + 1$ factors in the input \acp{fg} for each domain size $d = 2, 4, 8, 12, 16, 20, 32, 64, 128, 256, 512, 1024$.
The maximum number of arguments of a factor is again seven and we evaluate two queries per input \ac{fg} for each domain size $d$ and then report the average run time and standard deviation over all queries for that choice of $d$.
Not surprisingly, the results shown in \cref{fig:plot-results-inter-p=5} exhibit similar patterns as the results in \cref{fig:plot-results-inter-p=3}.
The advantage of running \ac{lve} on the \ac{pfg} obtained by running \ac{cp} compared to running \ac{ve} on the initial propositional \ac{fg} is now less pronounced, as \ac{cp} is not able to detect the symmetries between factors with permuted argument lists.
Therefore, there are more groups in total after running \ac{cp} and hence the resulting \ac{pfg} is less compact.
As \ac{cpr} is able to detect symmetries between factors with permuted argument lists, the percentage of factors with permuted argument lists does not impact the compression obtained by running \ac{cpr} and hence the run times of \ac{lve} on the output of \ac{cpr} are not negatively affected by a higher percentage of permuted factors.
When further increasing the percentage of permuted factors, the run times of \ac{lve} on the output of \ac{cp} drastically increase, as shown in \cref{fig:plot-results-inter-p=10,fig:plot-results-inter-p=15}.
Interestingly, running \ac{ve} on the initial \ac{fg} becomes even faster than running \ac{lve} on the \ac{pfg} obtained by running \ac{cp} if at least ten percent of the factors have permuted argument lists.
The result that \ac{lve} (CP) becomes slower than \ac{ve} can by explained by the fact that \ac{lve} introduces some overhead compared to \ac{ve}.
Even though the overhead induced by \ac{lve} is rather small, \ac{cp} groups just a few \acp{rv} and factors, respectively, and \ac{lve} does not benefit a lot from a vast amount of small groups.
Therefore, the results suggest that the compression obtained by \ac{cp} does not compensate for the overhead of \ac{lve}.
\Cref{fig:plot-results-inter-p=10,fig:plot-results-inter-p=15} emphasise that checking for permutations is indispensible for \ac{cp} to be effective in practical applications.

\begin{figure}[t]
	\centering
	\input{./figures/plot-offline-intra.tex}
	\caption{Average number $\alpha$ of queries after which the additional offline effort amortises for input \acp{fg} containing a total of $k$ commutative factors.}
	\label{fig:plot-offline-intra}
\end{figure}
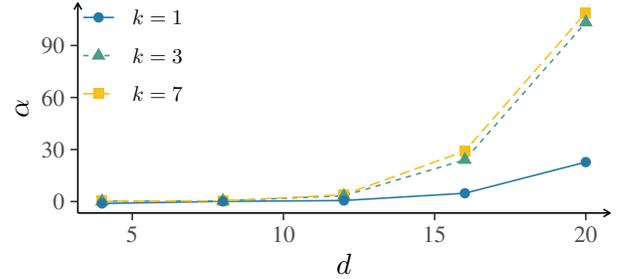
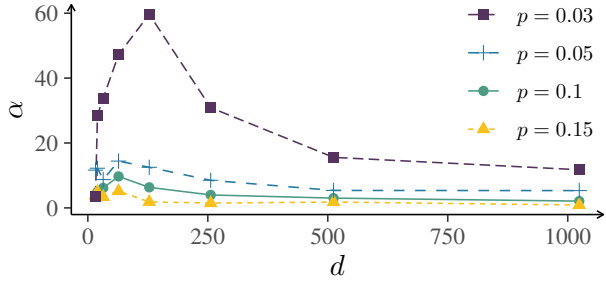
\begin{figure}[t]
	\centering
	\input{./figures/plot-offline-inter.tex}
	\caption{Average number $\alpha$ of queries after which the additional offline effort amortises for input \acp{fg} where a proportion $p$ of the factors has permuted arguments.}
	\label{fig:plot-offline-inter}
\end{figure}

Moreover, we report the average number $\alpha$ of queries after which the additional offline effort of \ac{cpr} compared to \ac{cp} amortises in \cref{fig:plot-offline-intra,fig:plot-offline-inter}.
In particular, it holds that $\alpha = \frac{\Delta_o}{\Delta_g}$ with $\Delta_o$ (\enquote{offline overhead}) denoting the difference of the offline time required by \ac{cpr} and \ac{cp} to obtain the \ac{pfg} and $\Delta_g$ (\enquote{online gain}) denoting the difference of the time required by \ac{lve} on the output of \ac{cp} and \ac{cpr} to answer a query.
Thus, after $\alpha$ queries, the additional time needed during the offline step by \ac{cpr} is saved by the faster query times of \ac{lve} on the output of \ac{cpr}.
\Cref{fig:plot-offline-intra} shows the average $\alpha$ for different domain sizes $d = 4, 8, 12, 16, 20$ on the input \acp{fg} from \cref{fig:plot-results-intra-k=1,fig:plot-results-intra-k=3,fig:plot-results-intra-k=7}.
Each line in \cref{fig:plot-offline-intra} corresponds to a different number $k$ of commutative factors present in the input \ac{fg}.
For domain sizes $d < 16$, far less than ten queries are sufficient to save the additional time needed by \ac{cpr} during the offline step for all choices of $k$.
At $d = 16$ and $d = 20$, $\alpha$ increases slightly for $k = 1$ while it increases more steeply for $k = 3$ and $k = 7$, showing that larger graphs require quite some offline effort (the graph size increases with increasing values for $k$).
However, it is important to keep in mind that the offline step can be performed in advance before a system is deployed and hence, the additional offline overhead is still worth the effort in practical applications, especially if computing resources are limited during online inference.
We also remark that the standard deviation, which is not graphically represented in \cref{fig:plot-offline-intra}, greatly varies between different domain sizes $d$.
In particular, the standard deviation is about $30$ for $d = 20$ while it is mostly far below ten for all other choices of $d$.
\Cref{fig:plot-offline-inter} shows the average $\alpha$ for different domain sizes $d = 16, 20, 32, 64, 128, 256, 512, 1024$ on the input \acp{fg} from \cref{fig:plot-results-inter-p=3,fig:plot-results-inter-p=5,fig:plot-results-inter-p=10,fig:plot-results-inter-p=15}.
Each line corresponds to a different proportion $p$ of factors with permuted arguments.
Note that the graph size does not depend on $p$ and therefore it is not surprising that the average $\alpha$ decreases the larger the proportion $p$ of factors with permuted arguments is (because the offline overhead of \ac{cpr} remains identical while at the same time the online gain increases for larger values of $p$).
Except for some outliers for $p = 0.03$, the average $\alpha$ stays below $20$ for all choices of $d$.
It is also noteworthy that $\alpha$ mostly remains below ten, indicating that the detection of permuted factors often amortises faster than the detection of commutative factors.
We report a rather high standard deviation for most of the displayed scenarios, which supports the intuition that permutations are sometimes found very fast and sometimes require quite some time.

\section{Example Run of Advanced Colour Passing}
\begin{figure*}[t]
	\centering
	\input{./figures/cpr_example.tex}
	\caption{A visualisation of the steps undertaken by \cref{alg:cp_revisited} on an input \ac{fg} with only Boolean \acp{rv} and no evidence (left). The general routine of the \ac{cp} algorithm remains the same, but additionally commutative factors are recognised and potentials are compared independently of the factors' arguments order. The resulting \ac{pfg} is depicted on the right, where two \acp{prv} over a \ac{lv} $X$ with $|\mathcal D(X)| = 2$ replace the two groups of size two ($A$, $D$ and $B$, $C$). Note that the \ac{prv} $S(X)$ appears \enquote{normal} in $\phi'_1$ and count converted in $\phi'_2$.}
	\label{fig:cpr_example}
\end{figure*}
\Cref{fig:cpr_example} illustrates \ac{cpr} on an example input \ac{fg} containing both a commutative factor $\phi_2$ and two factors $\phi_1$ and $\phi_3$ representing identical potentials but the potentials are not written in the same order into their tables.
In this example, all \acp{rv} are Boolean and there is no evidence (i.e., $\boldsymbol E = \emptyset$).
Initially, \ac{cpr} assigns all \acp{rv} the same colour (e.g., $\mathrm{yellow}$) because they have the same range (Boolean) and the same observed event (no event).
$\phi_2$ is assigned its own colour (e.g., $\mathrm{green}$) and $\phi_1$ and $\phi_3$ are assigned the same colour (e.g., $\mathrm{blue}$) because they represent identical potentials when swapping, e.g., the order of $C$ and $D$ in $\phi_3$.
Hence, $\phi_3$ includes position one into its message to $D$ and position two into its message to $C$.
The colours are then passed from variable nodes to factor nodes and as each factor has two neighbouring \acp{rv}, all factors receive the same messages.
Therefore, after recolouring the factors, the colour assignments are identical to the initial assignments.
Afterwards, the factor nodes send their colours to their neighbouring variable nodes.
Each message from a factor to a \ac{rv} contains the position of the \ac{rv} in the factor if the factor is not commutative, else the position is replaced by zero.
During this step, $A$ receives a message $(\mathrm{blue},1)$ from $\phi_1$, $B$ receives a message $(\mathrm{blue},2)$ from $\phi_1$ and a message $(\mathrm{green},0)$ from $\phi_2$, $C$ receives a message $(\mathrm{green},0)$ from $\phi_2$ and a message $(\mathrm{blue},2)$ from $\phi_3$, and $D$ receives a message $(\mathrm{blue},1)$ from $\phi_3$ (remember that the positions of $C$ and $D$ have been swapped at the beginning).
Hence, $A$ and $D$ as well as $B$ and $C$ receive identical messages (positions are not shown in \cref{fig:cpr_example}).
After the recolouring step, $A$ and $D$ share a colour and $B$ and $C$ share a different colour.
The groupings do not change in later iterations and the resulting \ac{pfg} is shown on the right.

Both groups of \acp{rv} are represented by a \ac{prv} ($R(X)$ for $A$ and $D$, and $S(X)$ for $B$ and $C$), respectively, $\phi_1$ and $\phi_3$ are replaced by a \ac{pf} $\phi'_1(R(X), S(X))$, and $\phi_2$ is replaced by a \ac{pf} $\phi'_2(\#_X[S(X)])$ in which $S(X)$ appears count converted.
The \ac{pfg} contains an edge between $\phi'_1$ and $R(X)$ because there exists a \ac{rv} in the group represented by $R(X)$ which is connected to a factor in the group represented by $\phi'_1$ in the original \ac{fg} ($A$ is connected to $\phi_1$ and $D$ is connected to $\phi_3$).
Moreover, there are edges between $\phi'_1$ and $S(X)$ and between $\phi'_2$ and $S(X)$ because there are edges between, for example, $\phi_1$ and $B$ as well as between $\phi_2$ and $B$ in the original \ac{fg}.

Since $\phi'_1$ has the same number of arguments as $\phi_1$, there are no \acp{crv} in $\phi'_1$.
In particular, there are two \acp{prv} with a shared \ac{lv} because $\abs{gr(R(X))} = \abs{gr(S(X))} = \abs{gr(\phi'_1)} = 2$ (all groups contain exactly two elements, that is, $A$ and $D$, $B$ and $C$, as well as $\phi_1$ and $\phi_3$, respectively).
As the number of arguments of $\phi'_2$ has changed compared to the number of arguments of $\phi_2$, a \ac{crv} is used to represent the potentials of $\phi_2$ in $\phi'_2$.
In this example, $\phi_2$ is commutative with respect to both $B$ and $C$ and therefore, $\phi'_2$ has only a single argument counting over all elements of the group $\{B,C\}$.
Note that $S(X)$ appears \enquote{normal} in $\phi'_1$ and count converted in $\phi'_2$.
Running \ac{cp} on the \ac{pfg} given in \cref{fig:cpr_example} ends up without grouping together anything.

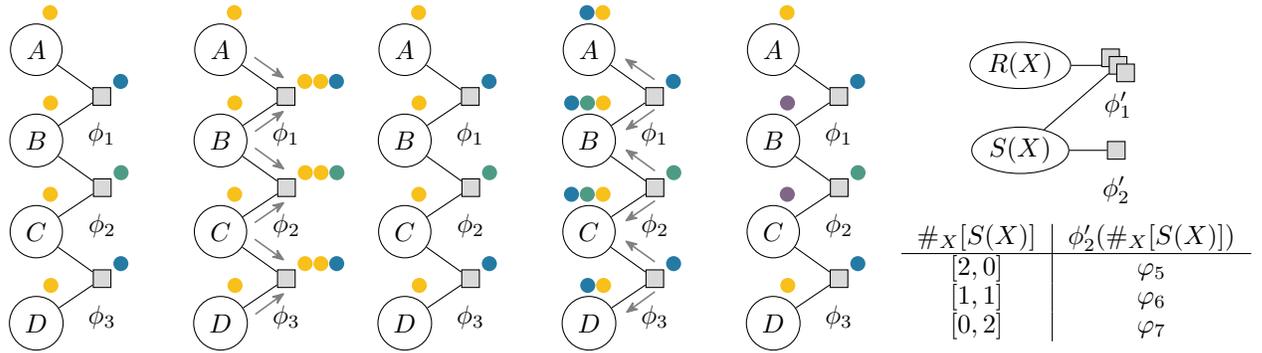
\end{document}

%% file: defs.tex
\acrodef{cp}[CP]{colour passing}
\acrodef{cpr}[ACP]{advanced colour passing}
\acrodef{crv}[CRV]{counting randvar}
\acrodef{fg}[FG]{factor graph}
\acrodef{ljt}[LJT]{lifted junction tree}
\acrodef{lv}[logvar]{logical variable}
\acrodef{lve}[LVE]{lifted variable elimination}
\acrodef{pcrv}[PCRV]{parameterised CRV}
\acrodef{pf}[parfactor]{parametric factor}
\acrodef{pfg}[PFG]{parametric factor graph}
\acrodef{prv}[PRV]{parameterised randvar}
\acrodef{rv}[randvar]{random variable}
\acrodef{ve}[VE]{variable elimination}
\acrodef{wl}[WL]{Weisfeiler-Leman}

\newcommand{\alginput}[1]{\hspace*{\algorithmicindent} \textbf{Input:} #1}
\newcommand{\algoutput}[1]{\hspace*{\algorithmicindent} \textbf{Output:} #1}
\newcommand{\abs}[1]{\lvert #1 \rvert}
\newcommand{\symdiff}{\ensuremath{\mathbin{\triangle}}}

\definecolor{myyellow}{RGB}{247,192,26}
\definecolor{myblue}{RGB}{37,122,164}
\definecolor{mygreen}{RGB}{78,155,133}
\definecolor{mypurple}{RGB}{86,51,94}

\pgfdeclarelayer{bg}
\pgfsetlayers{bg,main}

\tikzset{
	rv/.style={draw, ellipse},
	pf/.style={draw, rectangle, fill = gray!30},
	arc/.style = {->, semithick, >={[round,sep]Stealth}},
}

\newcommand\factor[6]{
	\node[pf, #1=#3 of #2, label=#4:{#5}](#6) {};
}

\newcommand\nodecolorshift[5]{
	\node[circle, fill=#1, above right=0.1cm of #2, inner sep=0pt, minimum size=2mm, xshift=#4, yshift=#5](#3) {};
}

\newcommand\factorcolor[3]{
	\node[circle, fill=#1, above right=0cm and 0.05cm of #2, inner sep=0pt, minimum size=2mm](#3) {};
}
\newcommand\factorcolorshift[4]{
	\node[circle, fill=#1, above right=0cm and 0.05cm of #2, inner sep=0pt, minimum size=2mm, xshift=#4](#3) {};
}

\newcommand\pfs[8]{
	\node[pf, #1=#3 of #2, xshift=-1mm, yshift=1mm](#6) {};
	\node[pf, #1=#3 of #2, label={[label distance=1mm]#4:{#5}}](#7) {};
	\node[pf, #1=#3 of #2, xshift=1mm, yshift=-1mm](#8) {};
}

%% file: figures/cp_example.tex
\begin{tikzpicture}[label distance=1mm]
	%%% Input model %%%
	\node[circle, draw] (A) {$A$};
	\node[circle, draw] (B) [below = 0.5cm of A] {$B$};
	\node[circle, draw] (C) [below = 0.5cm of B] {$C$};
	\factor{below right}{A}{0.25cm and 0.5cm}{270}{$\phi_1$}{f1}
	\factor{below right}{B}{0.25cm and 0.5cm}{270}{$\phi_2$}{f2}

	\nodecolorshift{myyellow}{A}{Acol}{-2.1mm}{1mm}
	\nodecolorshift{myyellow}{B}{Bcol}{-2.1mm}{1mm}
	\nodecolorshift{myyellow}{C}{Ccol}{-2.1mm}{1mm}

	\factorcolor{myblue}{f1}{f1col}
	\factorcolor{myblue}{f2}{f2col}

	\draw (A) -- (f1);
	\draw (B) -- (f1);
	\draw (B) -- (f2);
	\draw (C) -- (f2);

	%%% Color Passing: Node to factor %%%
	\node[circle, draw, right = 2cm of A] (A1) {$A$};
	\node[circle, draw, below = 0.5cm of A1] (B1) {$B$};
	\node[circle, draw, below = 0.5cm of B1] (C1) {$C$};
	\factor{below right}{A1}{0.25cm and 0.5cm}{270}{$\phi_1$}{f1_1}
	\factor{below right}{B1}{0.25cm and 0.5cm}{270}{$\phi_2$}{f2_1}

	\nodecolorshift{myyellow}{A1}{A1col}{-2.1mm}{1mm}
	\nodecolorshift{myyellow}{B1}{B1col}{-2.1mm}{1mm}
	\nodecolorshift{myyellow}{C1}{C1col}{-2.1mm}{1mm}

	\factorcolor{myyellow}{f1_1}{f1_1col1}
	\factorcolorshift{myyellow}{f1_1}{f1_1col2}{2.1mm}{0mm}
	\factorcolorshift{myblue}{f1_1}{f1_1col3}{4.2mm}{0mm}
	\factorcolor{myyellow}{f2_1}{f2_1col1}
	\factorcolorshift{myyellow}{f2_1}{f2_1col2}{2.1mm}{0mm}
	\factorcolorshift{myblue}{f2_1}{f2_1col3}{4.2mm}{0mm}

	\coordinate[right=0.1cm of A1, yshift=-0.1cm] (CA1);
	\coordinate[above=0.2cm of f1_1, yshift=-0.1cm] (Cf1_1);
	\coordinate[right=0.1cm of B1, yshift=0.12cm] (CB1);
	\coordinate[right=0.1cm of B1, yshift=-0.1cm] (CB1_1);
	\coordinate[below=0.2cm of f1_1, yshift=0.15cm] (Cf1_1b);
	\coordinate[above=0.2cm of f2_1, yshift=-0.1cm] (Cf2_1);
	\coordinate[right=0.1cm of C1, yshift=0.12cm] (CC1);
	\coordinate[below=0.2cm of f2_1, yshift=0.15cm] (Cf2_1b);

	\begin{pgfonlayer}{bg}
		\draw (A1) -- (f1_1);
		\draw [arc, gray] (CA1) -- (Cf1_1);
		\draw (B1) -- (f1_1);
		\draw [arc, gray] (CB1) -- (Cf1_1b);
		\draw (B1) -- (f2_1);
		\draw [arc, gray] (CB1_1) -- (Cf2_1);
		\draw (C1) -- (f2_1);
		\draw [arc, gray] (CC1) -- (Cf2_1b);
	\end{pgfonlayer}

	%%% Color Passing: Recoloring %%%
	\node[circle, draw, right = 2cm of A1] (A2) {$A$};
	\node[circle, draw, below = 0.5cm of A2] (B2) {$B$};
	\node[circle, draw, below = 0.5cm of B2] (C2) {$C$};
	\factor{below right}{A2}{0.25cm and 0.5cm}{270}{$\phi_1$}{f1_2}
	\factor{below right}{B2}{0.25cm and 0.5cm}{270}{$\phi_2$}{f2_2}

	\nodecolorshift{myyellow}{A2}{A2col}{-2.1mm}{1mm}
	\nodecolorshift{myyellow}{B2}{B2col}{-2.1mm}{1mm}
	\nodecolorshift{myyellow}{C2}{C2col}{-2.1mm}{1mm}

	\factorcolor{myblue}{f1_2}{f1_2col1}
	\factorcolor{myblue}{f2_2}{f2_2col1}

	\draw (A2) -- (f1_2);
	\draw (B2) -- (f1_2);
	\draw (B2) -- (f2_2);
	\draw (C2) -- (f2_2);

	%%% Color Passing: Factor to node %%%
	\node[circle, draw, right = 2cm of A2] (A3) {$A$};
	\node[circle, draw, below = 0.5cm of A3] (B3) {$B$};
	\node[circle, draw, below = 0.5cm of B3] (C3) {$C$};
	\factor{below right}{A3}{0.25cm and 0.5cm}{270}{$\phi_1$}{f1_3}
	\factor{below right}{B3}{0.25cm and 0.5cm}{270}{$\phi_2$}{f2_3}

	\nodecolorshift{myblue}{A3}{A3col1}{-4.2mm}{1mm}
	\nodecolorshift{myyellow}{A3}{A3col2}{-2.1mm}{1mm}
	\nodecolorshift{myblue}{B3}{B3col1}{-6.3mm}{1mm}
	\nodecolorshift{myblue}{B3}{B3col2}{-4.2mm}{1mm}
	\nodecolorshift{myyellow}{B3}{B3col3}{-2.1mm}{1mm}
	\nodecolorshift{myblue}{C3}{C3col1}{-4.2mm}{1mm}
	\nodecolorshift{myyellow}{C3}{C3col2}{-2.1mm}{1mm}

	\factorcolor{myblue}{f1_3}{f1_3col1}
	\factorcolor{myblue}{f2_3}{f2_3col1}

	\coordinate[right=0.1cm of A3, yshift=-0.1cm] (CA3);
	\coordinate[above=0.2cm of f1_3, yshift=-0.1cm] (Cf1_3);
	\coordinate[right=0.1cm of B3, yshift=0.12cm] (CB3);
	\coordinate[right=0.1cm of B3, yshift=-0.1cm] (CB1_3);
	\coordinate[below=0.2cm of f1_3, yshift=0.15cm] (Cf1_3b);
	\coordinate[above=0.2cm of f2_3, yshift=-0.1cm] (Cf2_3);
	\coordinate[right=0.1cm of C3, yshift=0.12cm] (CC3);
	\coordinate[below=0.2cm of f2_3, yshift=0.15cm] (Cf2_3b);

	\begin{pgfonlayer}{bg}
		\draw (A3) -- (f1_3);
		\draw [arc, gray] (Cf1_3) -- (CA3);
		\draw (B3) -- (f1_3);
		\draw [arc, gray] (Cf1_3b) -- (CB3);
		\draw (B3) -- (f2_3);
		\draw [arc, gray] (Cf2_3) -- (CB1_3);
		\draw (C3) -- (f2_3);
		\draw [arc, gray] (Cf2_3b) -- (CC3);
	\end{pgfonlayer}

	%%% Color Passing: Recoloring %%%
	\node[circle, draw, right = 2cm of A3] (A4) {$A$};
	\node[circle, draw, below = 0.5cm of A4] (B4) {$B$};
	\node[circle, draw, below = 0.5cm of B4] (C4) {$C$};
	\factor{below right}{A4}{0.25cm and 0.5cm}{270}{$\phi_1$}{f1_4}
	\factor{below right}{B4}{0.25cm and 0.5cm}{270}{$\phi_2$}{f2_4}

	\nodecolorshift{myyellow}{A4}{A4col}{-2.1mm}{1mm}
	\nodecolorshift{mygreen}{B4}{B4col}{-2.1mm}{1mm}
	\nodecolorshift{myyellow}{C4}{C4col}{-2.1mm}{1mm}

	\factorcolor{myblue}{f1_4}{f1_4col1}
	\factorcolor{myblue}{f2_4}{f2_4col1}

	\draw (A4) -- (f1_4);
	\draw (B4) -- (f1_4);
	\draw (B4) -- (f2_4);
	\draw (C4) -- (f2_4);

	%%% Resulting compressed model %%%
	\pfs{right}{B4}{3.5cm}{270}{$\phi'_1$}{f12a}{f12}{f12b}

	\node[ellipse, inner sep = 1.2pt, draw, above left = 0.25cm and 0.5cm of f12] (AC) {$R(X)$};
	\node[circle, draw] (B) [below left = 0.25cm and 0.7cm of f12] {$B$};

	\begin{pgfonlayer}{bg}
		\draw (AC) -- (f12);
		\draw (B) -- (f12);
	\end{pgfonlayer}

	%%% Tables %%%
	\node[below = 0.5cm of C2, xshift=1.5cm] (tab_f2) {
		\begin{tabular}{c|c|c}
			$C$   & $B$   & $\phi_2(C,B)$ \\ \hline
			true  & true  & $\varphi_1$ \\
			true  & false & $\varphi_2$ \\
			false & true  & $\varphi_3$ \\
			false & false & $\varphi_4$ \\
		\end{tabular}
	};

	\node[left = 0.5cm of tab_f2] (tab_f1) {
		\begin{tabular}{c|c|c}
			$A$   & $B$   & $\phi_1(A,B)$ \\ \hline
			true  & true  & $\varphi_1$ \\
			true  & false & $\varphi_2$ \\
			false & true  & $\varphi_3$ \\
			false & false & $\varphi_4$ \\
		\end{tabular}
	};

	\node[right = 0.5cm of tab_f2] (tab_f12) {
		\begin{tabular}{c|c|c}
			$R(X)$   & $B$   & $\phi'_1(R(X),B)$ \\ \hline
			true  & true  & $\varphi_1$ \\
			true  & false & $\varphi_2$ \\
			false & true  & $\varphi_3$ \\
			false & false & $\varphi_4$ \\
		\end{tabular}
	};
\end{tikzpicture}

%% file: figures/example_lv_intro_shared.tex
\begin{tikzpicture}[rv/.append style = {inner sep = 1.8pt}]
	\node[rv, minimum width = 1.2cm, draw] (r1) {$R'_1$};
	\node[rv, draw, left = 0.4cm of r1] (s1) {$S'_1(X)$};
	\node[rv, draw, right = 0.4cm of r1] (s2) {$S'_2(X)$};
	\pfs{below}{r1}{0.6cm}{270}{$\phi'$}{fa}{f}{fb}

	\node[rv, minimum width = 1.2cm, draw, below = 2cm of r1] (r1gr) {$R_1$};
	\node[rv, draw, below left = 0.4cm and 0.8cm of r1gr] (s1x1) {$S'_1(x_1)$};
	\node[rv, draw, below = 0.1cm of s1x1] (s2x1) {$S'_2(x_1)$};
	\node[rv, draw, below right = 0.4cm and 0.8cm of r1gr] (s1x2) {$S'_1(x_2)$};
	\node[rv, draw, below = 0.1cm of s1x2] (s2x2) {$S'_2(x_2)$};

	\node[ellipse, below = 0.1cm of s2x1, minimum height = 1.075cm] (c1) {};
	\node[ellipse, below = 0.1cm of s2x2, minimum height = 1.075cm] (c2) {};

	\factor{right}{s1x1}{0.5cm}{95}{$\phi_1$}{f1gr}
	\factor{left}{s1x2}{0.5cm}{85}{$\phi_2$}{f2gr}

	\begin{pgfonlayer}{bg}
		\draw (s1) -- (f);
		\draw (s2) -- (f);
		\draw (r1) -- (f);
		\draw (r1gr) -- (f1gr);
		\draw (r1gr) -- (f2gr);
		\draw (s1x1.east) -- (f1gr);
		\draw (s2x1.east) -- (f1gr);
		\draw (s1x2.west) -- (f2gr);
		\draw (s2x2.west) -- (f2gr);
	\end{pgfonlayer}
\end{tikzpicture}

%% file: figures/example_lv_intro_distinct.tex
\begin{tikzpicture}[rv/.append style = {inner sep = 1.8pt}]
	\node[rv, minimum width = 1.2cm, draw] (r1) {$R'_1$};
	\node[rv, draw, left = 0.4cm of r1] (s1) {$S'_1(X)$};
	\node[rv, draw, right = 0.4cm of r1] (s2) {$S'_2(Y)$};
	\pfs{below}{r1}{0.6cm}{270}{$\phi'$}{fa}{f}{fb}

	\node[rv, minimum width = 1.2cm, draw, below = 2cm of r1] (r1gr) {$R_1$};
	\node[rv, draw, below left = 0.4cm and 0.8cm of r1gr] (s1x1) {$S'_1(x_1)$};
	\node[rv, draw, below = 0.1cm of s1x1] (s2y1) {$S'_2(y_1)$};
	\node[rv, draw, below right = 0.4cm and 0.8cm of r1gr] (s1x2) {$S'_1(x_2)$};
	\node[rv, draw, below = 0.1cm of s1x2] (s2y2) {$S'_2(y_2)$};

	\node[ellipse, below = 0.1cm of s2y1, minimum height = 1.075cm] (c1) {};
	\node[ellipse, below = 0.1cm of s2y2, minimum height = 1.075cm] (c2) {};

	\factor{right}{s1x1}{0.4cm}{95}{$\phi_1$}{f1gr}
	\factor{left}{s1x2}{0.4cm}{85}{$\phi_2$}{f2gr}
	\factor{right}{s2y1}{0.6cm}{270}{$\phi_3$}{f3gr}
	\factor{left}{s2y2}{0.6cm}{270}{$\phi_4$}{f4gr}

	\begin{pgfonlayer}{bg}
		\draw (s1) -- (f);
		\draw (s2) -- (f);
		\draw (r1) -- (f);
		\draw (r1gr) -- (f1gr);
		\draw (r1gr) -- (f2gr);
		\draw (r1gr) -- (f3gr);
		\draw (r1gr) -- (f4gr);
		\draw (s1x1.east) -- (f1gr);
		\draw (s2y2.west) -- (f1gr);
		\draw (s1x2.west) -- (f2gr);
		\draw (s2y1.east) -- (f2gr);
		\draw (s1x1.east) -- (f3gr);
		\draw (s2y1.east) -- (f3gr);
		\draw (s1x2.west) -- (f4gr);
		\draw (s2y2.west) -- (f4gr);
	\end{pgfonlayer}
\end{tikzpicture}

%% file: figures/example_lv_intro_double.tex
\begin{tikzpicture}[rv/.append style = {inner sep = 1.8pt}]
	\node[rv, minimum width = 1.2cm, draw] (r1) {$R'_1$};
	\node[rv, draw, left = 0.4cm of r1] (s1) {$S'_1(X)$};
	\node[rv, draw, right = 0.4cm of r1] (s2) {$S'_2(X,Y)$};
	\pfs{below}{r1}{0.6cm}{270}{$\phi'$}{fa}{f}{fb}

	\node[rv, minimum width = 1.2cm, draw, below = 2cm of r1] (r1gr) {$R_1$};
	
	\node[rv, draw, below left = 0.4cm and 0.8cm of r1gr] (s1x1) {$S'_1(x_1)$};
	\node[rv, draw, below right = 0.4cm and 0.8cm of r1gr] (s1x2) {$S'_1(x_2)$};
	\node[rv, draw, below = 0.1cm of s1x1] (s2x1y1) {$S'_2(x_1,y_1)$};
	\node[rv, draw, below = 0.1cm of s2x1y1] (s2x1y2) {$S'_2(x_1,y_2)$};
	\node[rv, draw, below = 0.1cm of s1x2] (s2x2y1) {$S'_2(x_2,y_1)$};
	\node[rv, draw, below = 0.1cm of s2x2y1] (s2x2y2) {$S'_2(x_2,y_2)$};

	\factor{right}{s1x1}{0.5cm}{95}{$\phi_1$}{f1gr}
	\factor{left}{s1x2}{0.5cm}{85}{$\phi_2$}{f2gr}
	\factor{right}{s2x1y1}{0.3cm}{270}{$\phi_3$}{f3gr}
	\factor{left}{s2x2y1}{0.3cm}{270}{$\phi_4$}{f4gr}

	\begin{pgfonlayer}{bg}
		\draw (s1) -- (f);
		\draw (s2) -- (f);
		\draw (r1) -- (f);
		\draw (r1gr) -- (f1gr);
		\draw (r1gr) -- (f2gr);
		\draw (r1gr) -- (f3gr);
		\draw (r1gr) -- (f4gr);
		\draw (s1x1.east) -- (f1gr);
		\draw (s1x1.east) -- (f3gr);
		\draw (s1x2.west) -- (f2gr);
		\draw (s1x2.west) -- (f4gr);
		\draw (s2x1y1.east) -- (f1gr);
		\draw (s2x1y2.east) -- (f3gr);
		\draw (s2x2y1.west) -- (f2gr);
		\draw (s2x2y2.west) -- (f4gr);
	\end{pgfonlayer}
\end{tikzpicture}

%% file: figures/plot-results-intra-k=3.tex
% Created by tikzDevice version 0.12.3.1 on 2023-08-04 17:02:04
% !TEX encoding = UTF-8 Unicode
\begin{tikzpicture}[x=1pt,y=1pt]
\definecolor{fillColor}{RGB}{255,255,255}
\path[use as bounding box,fill=fillColor,fill opacity=0.00] (0,15) rectangle (238.49,115.63);
\begin{scope}
\path[clip] (  0.00,  0.00) rectangle (238.49,115.63);
\definecolor{drawColor}{RGB}{255,255,255}
\definecolor{fillColor}{RGB}{255,255,255}

\path[draw=drawColor,line width= 0.6pt,line join=round,line cap=round,fill=fillColor] (  0.00,  0.00) rectangle (238.49,115.63);
\end{scope}
\begin{scope}
\path[clip] ( 44.91, 30.69) rectangle (232.99,110.13);
\definecolor{fillColor}{RGB}{255,255,255}

\path[fill=fillColor] ( 44.91, 30.69) rectangle (232.99,110.13);
\definecolor{drawColor}{RGB}{247,192,26}

\path[draw=drawColor,line width= 0.6pt,line join=round] ( 53.46, 43.27) --
	( 72.46, 44.37) --
	(110.45, 44.99) --
	(148.45, 45.92) --
	(186.45, 46.58) --
	(224.44, 47.36);
\definecolor{drawColor}{RGB}{78,155,133}

\path[draw=drawColor,line width= 0.6pt,dash pattern=on 2pt off 2pt ,line join=round] ( 53.46, 40.60) --
	( 72.46, 46.33) --
	(110.45, 53.69) --
	(148.45, 59.27) --
	(186.45, 67.52) --
	(224.44, 81.99);
\definecolor{drawColor}{RGB}{37,122,164}

\path[draw=drawColor,line width= 0.6pt,dash pattern=on 4pt off 2pt ,line join=round] ( 53.46, 35.35) --
	( 72.46, 41.88) --
	(110.45, 50.79) --
	(148.45, 66.30) --
	(186.45, 81.97) --
	(224.44,106.01);
\definecolor{fillColor}{RGB}{78,155,133}

\path[fill=fillColor] (110.45, 56.74) --
	(113.10, 52.16) --
	(107.81, 52.16) --
	cycle;

\path[fill=fillColor] (148.45, 62.32) --
	(151.09, 57.74) --
	(145.81, 57.74) --
	cycle;

\path[fill=fillColor] (224.44, 85.04) --
	(227.08, 80.46) --
	(221.80, 80.46) --
	cycle;

\path[fill=fillColor] ( 72.46, 49.38) --
	( 75.10, 44.80) --
	( 69.81, 44.80) --
	cycle;

\path[fill=fillColor] (186.45, 70.57) --
	(189.09, 66.00) --
	(183.80, 66.00) --
	cycle;

\path[fill=fillColor] ( 53.46, 43.66) --
	( 56.10, 39.08) --
	( 50.82, 39.08) --
	cycle;
\definecolor{fillColor}{RGB}{247,192,26}

\path[fill=fillColor] (110.45, 44.99) circle (  1.96);

\path[fill=fillColor] (148.45, 45.92) circle (  1.96);

\path[fill=fillColor] (224.44, 47.36) circle (  1.96);

\path[fill=fillColor] ( 72.46, 44.37) circle (  1.96);

\path[fill=fillColor] (186.45, 46.58) circle (  1.96);

\path[fill=fillColor] ( 53.46, 43.27) circle (  1.96);
\definecolor{fillColor}{RGB}{37,122,164}

\path[fill=fillColor] (108.49, 48.83) --
	(112.41, 48.83) --
	(112.41, 52.76) --
	(108.49, 52.76) --
	cycle;

\path[fill=fillColor] (146.49, 64.33) --
	(150.41, 64.33) --
	(150.41, 68.26) --
	(146.49, 68.26) --
	cycle;

\path[fill=fillColor] (222.48,104.05) --
	(226.40,104.05) --
	(226.40,107.97) --
	(222.48,107.97) --
	cycle;

\path[fill=fillColor] ( 70.49, 39.92) --
	( 74.42, 39.92) --
	( 74.42, 43.84) --
	( 70.49, 43.84) --
	cycle;

\path[fill=fillColor] (184.48, 80.01) --
	(188.41, 80.01) --
	(188.41, 83.93) --
	(184.48, 83.93) --
	cycle;

\path[fill=fillColor] ( 51.50, 33.39) --
	( 55.42, 33.39) --
	( 55.42, 37.32) --
	( 51.50, 37.32) --
	cycle;
\definecolor{fillColor}{RGB}{247,192,26}

\path[fill=fillColor,fill opacity=0.20] ( 53.46, 43.86) --
	( 72.46, 46.33) --
	(110.45, 46.95) --
	(148.45, 48.04) --
	(186.45, 47.73) --
	(224.44, 49.93) --
	(224.44, 43.86) --
	(186.45, 45.29) --
	(148.45, 43.20) --
	(110.45, 42.54) --
	( 72.46, 41.91) --
	( 53.46, 42.64) --
	cycle;

\path[] ( 53.46, 43.86) --
	( 72.46, 46.33) --
	(110.45, 46.95) --
	(148.45, 48.04) --
	(186.45, 47.73) --
	(224.44, 49.93);

\path[] (224.44, 43.86) --
	(186.45, 45.29) --
	(148.45, 43.20) --
	(110.45, 42.54) --
	( 72.46, 41.91) --
	( 53.46, 42.64);
\definecolor{fillColor}{RGB}{78,155,133}

\path[fill=fillColor,fill opacity=0.20] ( 53.46, 41.49) --
	( 72.46, 48.09) --
	(110.45, 55.36) --
	(148.45, 61.39) --
	(186.45, 70.98) --
	(224.44, 86.50) --
	(224.44, 73.35) --
	(186.45, 62.13) --
	(148.45, 56.56) --
	(110.45, 51.67) --
	( 72.46, 44.18) --
	( 53.46, 39.63) --
	cycle;

\path[] ( 53.46, 41.49) --
	( 72.46, 48.09) --
	(110.45, 55.36) --
	(148.45, 61.39) --
	(186.45, 70.98) --
	(224.44, 86.50);

\path[] (224.44, 73.35) --
	(186.45, 62.13) --
	(148.45, 56.56) --
	(110.45, 51.67) --
	( 72.46, 44.18) --
	( 53.46, 39.63);
\definecolor{fillColor}{RGB}{37,122,164}

\path[fill=fillColor,fill opacity=0.20] ( 53.46, 36.31) --
	( 72.46, 43.29) --
	(110.45, 51.15) --
	(148.45, 70.10) --
	(186.45, 82.35) --
	(224.44,106.52) --
	(224.44,105.47) --
	(186.45, 81.58) --
	(148.45, 59.98) --
	(110.45, 50.42) --
	( 72.46, 40.23) --
	( 53.46, 34.30) --
	cycle;

\path[] ( 53.46, 36.31) --
	( 72.46, 43.29) --
	(110.45, 51.15) --
	(148.45, 70.10) --
	(186.45, 82.35) --
	(224.44,106.52);

\path[] (224.44,105.47) --
	(186.45, 81.58) --
	(148.45, 59.98) --
	(110.45, 50.42) --
	( 72.46, 40.23) --
	( 53.46, 34.30);
\end{scope}
\begin{scope}
\path[clip] (  0.00,  0.00) rectangle (238.49,115.63);
\definecolor{drawColor}{RGB}{0,0,0}

\path[draw=drawColor,line width= 0.6pt,line join=round] ( 44.91, 30.69) --
	( 44.91,110.13);

\path[draw=drawColor,line width= 0.6pt,line join=round] ( 46.33,107.67) --
	( 44.91,110.13) --
	( 43.49,107.67);
\end{scope}
\begin{scope}
\path[clip] (  0.00,  0.00) rectangle (238.49,115.63);
\definecolor{drawColor}{gray}{0.30}

\node[text=drawColor,anchor=base east,inner sep=0pt, outer sep=0pt, scale=  0.88] at ( 39.96, 32.17) {10};

\node[text=drawColor,anchor=base east,inner sep=0pt, outer sep=0pt, scale=  0.88] at ( 39.96, 54.69) {100};

\node[text=drawColor,anchor=base east,inner sep=0pt, outer sep=0pt, scale=  0.88] at ( 39.96, 77.21) {1000};

\node[text=drawColor,anchor=base east,inner sep=0pt, outer sep=0pt, scale=  0.88] at ( 39.96, 99.73) {10000};
\end{scope}
\begin{scope}
\path[clip] (  0.00,  0.00) rectangle (238.49,115.63);
\definecolor{drawColor}{gray}{0.20}

\path[draw=drawColor,line width= 0.6pt,line join=round] ( 42.16, 35.21) --
	( 44.91, 35.21);

\path[draw=drawColor,line width= 0.6pt,line join=round] ( 42.16, 57.72) --
	( 44.91, 57.72);

\path[draw=drawColor,line width= 0.6pt,line join=round] ( 42.16, 80.24) --
	( 44.91, 80.24);

\path[draw=drawColor,line width= 0.6pt,line join=round] ( 42.16,102.76) --
	( 44.91,102.76);
\end{scope}
\begin{scope}
\path[clip] (  0.00,  0.00) rectangle (238.49,115.63);
\definecolor{drawColor}{RGB}{0,0,0}

\path[draw=drawColor,line width= 0.6pt,line join=round] ( 44.91, 30.69) --
	(232.99, 30.69);

\path[draw=drawColor,line width= 0.6pt,line join=round] (230.53, 29.26) --
	(232.99, 30.69) --
	(230.53, 32.11);
\end{scope}
\begin{scope}
\path[clip] (  0.00,  0.00) rectangle (238.49,115.63);
\definecolor{drawColor}{gray}{0.20}

\path[draw=drawColor,line width= 0.6pt,line join=round] ( 81.96, 27.94) --
	( 81.96, 30.69);

\path[draw=drawColor,line width= 0.6pt,line join=round] (129.45, 27.94) --
	(129.45, 30.69);

\path[draw=drawColor,line width= 0.6pt,line join=round] (176.95, 27.94) --
	(176.95, 30.69);

\path[draw=drawColor,line width= 0.6pt,line join=round] (224.44, 27.94) --
	(224.44, 30.69);
\end{scope}
\begin{scope}
\path[clip] (  0.00,  0.00) rectangle (238.49,115.63);
\definecolor{drawColor}{gray}{0.30}

\node[text=drawColor,anchor=base,inner sep=0pt, outer sep=0pt, scale=  0.88] at ( 81.96, 19.68) {5};

\node[text=drawColor,anchor=base,inner sep=0pt, outer sep=0pt, scale=  0.88] at (129.45, 19.68) {10};

\node[text=drawColor,anchor=base,inner sep=0pt, outer sep=0pt, scale=  0.88] at (176.95, 19.68) {15};

\node[text=drawColor,anchor=base,inner sep=0pt, outer sep=0pt, scale=  0.88] at (224.44, 19.68) {20};
\end{scope}
\begin{scope}
\path[clip] (  0.00,  0.00) rectangle (238.49,115.63);
\definecolor{drawColor}{RGB}{0,0,0}

\node[text=drawColor,anchor=base,inner sep=0pt, outer sep=0pt, scale=  1.10] at (138.95,  7.64) {$d$};
\end{scope}
\begin{scope}
\path[clip] (  0.00,  0.00) rectangle (238.49,115.63);
\definecolor{drawColor}{RGB}{0,0,0}

\node[text=drawColor,rotate= 90.00,anchor=base,inner sep=0pt, outer sep=0pt, scale=  1.10] at ( 13.08, 70.41) {time (ms)};
\end{scope}
\begin{scope}
\path[clip] (  0.00,  0.00) rectangle (238.49,115.63);

\path[] ( 40.69, 71.03) rectangle (113.08,125.40);
\end{scope}
\begin{scope}
\path[clip] (  0.00,  0.00) rectangle (238.49,115.63);
\definecolor{drawColor}{RGB}{247,192,26}

\path[draw=drawColor,line width= 0.6pt,line join=round] ( 47.63,112.67) -- ( 59.20,112.67);
\end{scope}
\begin{scope}
\path[clip] (  0.00,  0.00) rectangle (238.49,115.63);
\definecolor{fillColor}{RGB}{247,192,26}

\path[fill=fillColor] ( 53.42,112.67) circle (  1.96);
\end{scope}
\begin{scope}
\path[clip] (  0.00,  0.00) rectangle (238.49,115.63);
\definecolor{drawColor}{RGB}{78,155,133}

\path[draw=drawColor,line width= 0.6pt,dash pattern=on 2pt off 2pt ,line join=round] ( 47.63, 98.22) -- ( 59.20, 98.22);
\end{scope}
\begin{scope}
\path[clip] (  0.00,  0.00) rectangle (238.49,115.63);
\definecolor{fillColor}{RGB}{78,155,133}

\path[fill=fillColor] ( 53.42,101.27) --
	( 56.06, 96.69) --
	( 50.77, 96.69) --
	cycle;
\end{scope}
\begin{scope}
\path[clip] (  0.00,  0.00) rectangle (238.49,115.63);
\definecolor{drawColor}{RGB}{37,122,164}

\path[draw=drawColor,line width= 0.6pt,dash pattern=on 4pt off 2pt ,line join=round] ( 47.63, 83.76) -- ( 59.20, 83.76);
\end{scope}
\begin{scope}
\path[clip] (  0.00,  0.00) rectangle (238.49,115.63);
\definecolor{fillColor}{RGB}{37,122,164}

\path[fill=fillColor] ( 51.45, 81.80) --
	( 55.38, 81.80) --
	( 55.38, 85.72) --
	( 51.45, 85.72) --
	cycle;
\end{scope}
\begin{scope}
\path[clip] (  0.00,  0.00) rectangle (238.49,115.63);
\definecolor{drawColor}{RGB}{0,0,0}

\node[text=drawColor,anchor=base west,inner sep=0pt, outer sep=0pt, scale=  0.80] at ( 66.14,109.91) {LVE (ACP)};
\end{scope}
\begin{scope}
\path[clip] (  0.00,  0.00) rectangle (238.49,115.63);
\definecolor{drawColor}{RGB}{0,0,0}

\node[text=drawColor,anchor=base west,inner sep=0pt, outer sep=0pt, scale=  0.80] at ( 66.14, 95.46) {LVE (CP)};
\end{scope}
\begin{scope}
\path[clip] (  0.00,  0.00) rectangle (238.49,115.63);
\definecolor{drawColor}{RGB}{0,0,0}

\node[text=drawColor,anchor=base west,inner sep=0pt, outer sep=0pt, scale=  0.80] at ( 66.14, 81.01) {VE};
\end{scope}
\end{tikzpicture}

%% file: figures/plot-results-intra-k=7.tex
% Created by tikzDevice version 0.12.3.1 on 2023-08-04 17:02:04
% !TEX encoding = UTF-8 Unicode
\begin{tikzpicture}[x=1pt,y=1pt]
\definecolor{fillColor}{RGB}{255,255,255}
\path[use as bounding box,fill=fillColor,fill opacity=0.00] (0,15) rectangle (238.49,115.63);
\begin{scope}
\path[clip] (  0.00,  0.00) rectangle (238.49,115.63);
\definecolor{drawColor}{RGB}{255,255,255}
\definecolor{fillColor}{RGB}{255,255,255}

\path[draw=drawColor,line width= 0.6pt,line join=round,line cap=round,fill=fillColor] (  0.00,  0.00) rectangle (238.49,115.63);
\end{scope}
\begin{scope}
\path[clip] ( 44.91, 30.69) rectangle (232.99,110.13);
\definecolor{fillColor}{RGB}{255,255,255}

\path[fill=fillColor] ( 44.91, 30.69) rectangle (232.99,110.13);
\definecolor{drawColor}{RGB}{247,192,26}

\path[draw=drawColor,line width= 0.6pt,line join=round] ( 53.46, 46.19) --
	( 72.46, 46.62) --
	(110.45, 46.88) --
	(148.45, 47.01) --
	(186.45, 48.21) --
	(224.44, 48.62);
\definecolor{drawColor}{RGB}{78,155,133}

\path[draw=drawColor,line width= 0.6pt,dash pattern=on 2pt off 2pt ,line join=round] ( 53.46, 43.53) --
	( 72.46, 49.44) --
	(110.45, 56.29) --
	(148.45, 61.30) --
	(186.45, 69.80) --
	(224.44, 84.87);
\definecolor{drawColor}{RGB}{37,122,164}

\path[draw=drawColor,line width= 0.6pt,dash pattern=on 4pt off 2pt ,line join=round] ( 53.46, 35.90) --
	( 72.46, 41.19) --
	(110.45, 51.79) --
	(148.45, 65.59) --
	(186.45, 82.16) --
	(224.44,105.94);
\definecolor{fillColor}{RGB}{78,155,133}

\path[fill=fillColor] (110.45, 59.34) --
	(113.10, 54.76) --
	(107.81, 54.76) --
	cycle;

\path[fill=fillColor] (148.45, 64.35) --
	(151.09, 59.77) --
	(145.81, 59.77) --
	cycle;

\path[fill=fillColor] (224.44, 87.92) --
	(227.08, 83.35) --
	(221.80, 83.35) --
	cycle;

\path[fill=fillColor] ( 72.46, 52.49) --
	( 75.10, 47.91) --
	( 69.81, 47.91) --
	cycle;

\path[fill=fillColor] (186.45, 72.85) --
	(189.09, 68.28) --
	(183.80, 68.28) --
	cycle;

\path[fill=fillColor] ( 53.46, 46.58) --
	( 56.10, 42.00) --
	( 50.82, 42.00) --
	cycle;
\definecolor{fillColor}{RGB}{247,192,26}

\path[fill=fillColor] (110.45, 46.88) circle (  1.96);

\path[fill=fillColor] (148.45, 47.01) circle (  1.96);

\path[fill=fillColor] (224.44, 48.62) circle (  1.96);

\path[fill=fillColor] ( 72.46, 46.62) circle (  1.96);

\path[fill=fillColor] (186.45, 48.21) circle (  1.96);

\path[fill=fillColor] ( 53.46, 46.19) circle (  1.96);
\definecolor{fillColor}{RGB}{37,122,164}

\path[fill=fillColor] (108.49, 49.83) --
	(112.41, 49.83) --
	(112.41, 53.76) --
	(108.49, 53.76) --
	cycle;

\path[fill=fillColor] (146.49, 63.63) --
	(150.41, 63.63) --
	(150.41, 67.55) --
	(146.49, 67.55) --
	cycle;

\path[fill=fillColor] (222.48,103.98) --
	(226.40,103.98) --
	(226.40,107.90) --
	(222.48,107.90) --
	cycle;

\path[fill=fillColor] ( 70.49, 39.23) --
	( 74.42, 39.23) --
	( 74.42, 43.15) --
	( 70.49, 43.15) --
	cycle;

\path[fill=fillColor] (184.48, 80.20) --
	(188.41, 80.20) --
	(188.41, 84.12) --
	(184.48, 84.12) --
	cycle;

\path[fill=fillColor] ( 51.50, 33.94) --
	( 55.42, 33.94) --
	( 55.42, 37.86) --
	( 51.50, 37.86) --
	cycle;
\definecolor{fillColor}{RGB}{247,192,26}

\path[fill=fillColor,fill opacity=0.20] ( 53.46, 48.12) --
	( 72.46, 47.36) --
	(110.45, 47.61) --
	(148.45, 47.82) --
	(186.45, 49.27) --
	(224.44, 49.99) --
	(224.44, 47.02) --
	(186.45, 47.00) --
	(148.45, 46.13) --
	(110.45, 46.09) --
	( 72.46, 45.80) --
	( 53.46, 43.72) --
	cycle;

\path[] ( 53.46, 48.12) --
	( 72.46, 47.36) --
	(110.45, 47.61) --
	(148.45, 47.82) --
	(186.45, 49.27) --
	(224.44, 49.99);

\path[] (224.44, 47.02) --
	(186.45, 47.00) --
	(148.45, 46.13) --
	(110.45, 46.09) --
	( 72.46, 45.80) --
	( 53.46, 43.72);
\definecolor{fillColor}{RGB}{78,155,133}

\path[fill=fillColor,fill opacity=0.20] ( 53.46, 44.99) --
	( 72.46, 50.45) --
	(110.45, 57.55) --
	(148.45, 62.95) --
	(186.45, 73.06) --
	(224.44, 87.19) --
	(224.44, 81.75) --
	(186.45, 64.64) --
	(148.45, 59.28) --
	(110.45, 54.81) --
	( 72.46, 48.31) --
	( 53.46, 41.79) --
	cycle;

\path[] ( 53.46, 44.99) --
	( 72.46, 50.45) --
	(110.45, 57.55) --
	(148.45, 62.95) --
	(186.45, 73.06) --
	(224.44, 87.19);

\path[] (224.44, 81.75) --
	(186.45, 64.64) --
	(148.45, 59.28) --
	(110.45, 54.81) --
	( 72.46, 48.31) --
	( 53.46, 41.79);
\definecolor{fillColor}{RGB}{37,122,164}

\path[fill=fillColor,fill opacity=0.20] ( 53.46, 37.26) --
	( 72.46, 41.90) --
	(110.45, 52.34) --
	(148.45, 66.59) --
	(186.45, 82.67) --
	(224.44,106.52) --
	(224.44,105.33) --
	(186.45, 81.62) --
	(148.45, 64.46) --
	(110.45, 51.22) --
	( 72.46, 40.42) --
	( 53.46, 34.30) --
	cycle;

\path[] ( 53.46, 37.26) --
	( 72.46, 41.90) --
	(110.45, 52.34) --
	(148.45, 66.59) --
	(186.45, 82.67) --
	(224.44,106.52);

\path[] (224.44,105.33) --
	(186.45, 81.62) --
	(148.45, 64.46) --
	(110.45, 51.22) --
	( 72.46, 40.42) --
	( 53.46, 34.30);
\end{scope}
\begin{scope}
\path[clip] (  0.00,  0.00) rectangle (238.49,115.63);
\definecolor{drawColor}{RGB}{0,0,0}

\path[draw=drawColor,line width= 0.6pt,line join=round] ( 44.91, 30.69) --
	( 44.91,110.13);

\path[draw=drawColor,line width= 0.6pt,line join=round] ( 46.33,107.67) --
	( 44.91,110.13) --
	( 43.49,107.67);
\end{scope}
\begin{scope}
\path[clip] (  0.00,  0.00) rectangle (238.49,115.63);
\definecolor{drawColor}{gray}{0.30}

\node[text=drawColor,anchor=base east,inner sep=0pt, outer sep=0pt, scale=  0.88] at ( 39.96, 30.80) {10};

\node[text=drawColor,anchor=base east,inner sep=0pt, outer sep=0pt, scale=  0.88] at ( 39.96, 51.57) {100};

\node[text=drawColor,anchor=base east,inner sep=0pt, outer sep=0pt, scale=  0.88] at ( 39.96, 72.35) {1000};

\node[text=drawColor,anchor=base east,inner sep=0pt, outer sep=0pt, scale=  0.88] at ( 39.96, 93.12) {10000};
\end{scope}
\begin{scope}
\path[clip] (  0.00,  0.00) rectangle (238.49,115.63);
\definecolor{drawColor}{gray}{0.20}

\path[draw=drawColor,line width= 0.6pt,line join=round] ( 42.16, 33.83) --
	( 44.91, 33.83);

\path[draw=drawColor,line width= 0.6pt,line join=round] ( 42.16, 54.60) --
	( 44.91, 54.60);

\path[draw=drawColor,line width= 0.6pt,line join=round] ( 42.16, 75.38) --
	( 44.91, 75.38);

\path[draw=drawColor,line width= 0.6pt,line join=round] ( 42.16, 96.15) --
	( 44.91, 96.15);
\end{scope}
\begin{scope}
\path[clip] (  0.00,  0.00) rectangle (238.49,115.63);
\definecolor{drawColor}{RGB}{0,0,0}

\path[draw=drawColor,line width= 0.6pt,line join=round] ( 44.91, 30.69) --
	(232.99, 30.69);

\path[draw=drawColor,line width= 0.6pt,line join=round] (230.53, 29.26) --
	(232.99, 30.69) --
	(230.53, 32.11);
\end{scope}
\begin{scope}
\path[clip] (  0.00,  0.00) rectangle (238.49,115.63);
\definecolor{drawColor}{gray}{0.20}

\path[draw=drawColor,line width= 0.6pt,line join=round] ( 81.96, 27.94) --
	( 81.96, 30.69);

\path[draw=drawColor,line width= 0.6pt,line join=round] (129.45, 27.94) --
	(129.45, 30.69);

\path[draw=drawColor,line width= 0.6pt,line join=round] (176.95, 27.94) --
	(176.95, 30.69);

\path[draw=drawColor,line width= 0.6pt,line join=round] (224.44, 27.94) --
	(224.44, 30.69);
\end{scope}
\begin{scope}
\path[clip] (  0.00,  0.00) rectangle (238.49,115.63);
\definecolor{drawColor}{gray}{0.30}

\node[text=drawColor,anchor=base,inner sep=0pt, outer sep=0pt, scale=  0.88] at ( 81.96, 19.68) {5};

\node[text=drawColor,anchor=base,inner sep=0pt, outer sep=0pt, scale=  0.88] at (129.45, 19.68) {10};

\node[text=drawColor,anchor=base,inner sep=0pt, outer sep=0pt, scale=  0.88] at (176.95, 19.68) {15};

\node[text=drawColor,anchor=base,inner sep=0pt, outer sep=0pt, scale=  0.88] at (224.44, 19.68) {20};
\end{scope}
\begin{scope}
\path[clip] (  0.00,  0.00) rectangle (238.49,115.63);
\definecolor{drawColor}{RGB}{0,0,0}

\node[text=drawColor,anchor=base,inner sep=0pt, outer sep=0pt, scale=  1.10] at (138.95,  7.64) {$d$};
\end{scope}
\begin{scope}
\path[clip] (  0.00,  0.00) rectangle (238.49,115.63);
\definecolor{drawColor}{RGB}{0,0,0}

\node[text=drawColor,rotate= 90.00,anchor=base,inner sep=0pt, outer sep=0pt, scale=  1.10] at ( 13.08, 70.41) {time (ms)};
\end{scope}
\begin{scope}
\path[clip] (  0.00,  0.00) rectangle (238.49,115.63);

\path[] ( 40.69, 71.03) rectangle (113.08,125.40);
\end{scope}
\begin{scope}
\path[clip] (  0.00,  0.00) rectangle (238.49,115.63);
\definecolor{drawColor}{RGB}{247,192,26}

\path[draw=drawColor,line width= 0.6pt,line join=round] ( 47.63,112.67) -- ( 59.20,112.67);
\end{scope}
\begin{scope}
\path[clip] (  0.00,  0.00) rectangle (238.49,115.63);
\definecolor{fillColor}{RGB}{247,192,26}

\path[fill=fillColor] ( 53.42,112.67) circle (  1.96);
\end{scope}
\begin{scope}
\path[clip] (  0.00,  0.00) rectangle (238.49,115.63);
\definecolor{drawColor}{RGB}{78,155,133}

\path[draw=drawColor,line width= 0.6pt,dash pattern=on 2pt off 2pt ,line join=round] ( 47.63, 98.22) -- ( 59.20, 98.22);
\end{scope}
\begin{scope}
\path[clip] (  0.00,  0.00) rectangle (238.49,115.63);
\definecolor{fillColor}{RGB}{78,155,133}

\path[fill=fillColor] ( 53.42,101.27) --
	( 56.06, 96.69) --
	( 50.77, 96.69) --
	cycle;
\end{scope}
\begin{scope}
\path[clip] (  0.00,  0.00) rectangle (238.49,115.63);
\definecolor{drawColor}{RGB}{37,122,164}

\path[draw=drawColor,line width= 0.6pt,dash pattern=on 4pt off 2pt ,line join=round] ( 47.63, 83.76) -- ( 59.20, 83.76);
\end{scope}
\begin{scope}
\path[clip] (  0.00,  0.00) rectangle (238.49,115.63);
\definecolor{fillColor}{RGB}{37,122,164}

\path[fill=fillColor] ( 51.45, 81.80) --
	( 55.38, 81.80) --
	( 55.38, 85.72) --
	( 51.45, 85.72) --
	cycle;
\end{scope}
\begin{scope}
\path[clip] (  0.00,  0.00) rectangle (238.49,115.63);
\definecolor{drawColor}{RGB}{0,0,0}

\node[text=drawColor,anchor=base west,inner sep=0pt, outer sep=0pt, scale=  0.80] at ( 66.14,109.91) {LVE (ACP)};
\end{scope}
\begin{scope}
\path[clip] (  0.00,  0.00) rectangle (238.49,115.63);
\definecolor{drawColor}{RGB}{0,0,0}

\node[text=drawColor,anchor=base west,inner sep=0pt, outer sep=0pt, scale=  0.80] at ( 66.14, 95.46) {LVE (CP)};
\end{scope}
\begin{scope}
\path[clip] (  0.00,  0.00) rectangle (238.49,115.63);
\definecolor{drawColor}{RGB}{0,0,0}

\node[text=drawColor,anchor=base west,inner sep=0pt, outer sep=0pt, scale=  0.80] at ( 66.14, 81.01) {VE};
\end{scope}
\end{tikzpicture}

%% file: figures/plot-results-inter-p=5.tex
% Created by tikzDevice version 0.12.3.1 on 2023-08-04 17:02:03
% !TEX encoding = UTF-8 Unicode
\begin{tikzpicture}[x=1pt,y=1pt]
\definecolor{fillColor}{RGB}{255,255,255}
\path[use as bounding box,fill=fillColor,fill opacity=0.00] (0,15) rectangle (238.49,115.63);
\begin{scope}
\path[clip] (  0.00,  0.00) rectangle (238.49,115.63);
\definecolor{drawColor}{RGB}{255,255,255}
\definecolor{fillColor}{RGB}{255,255,255}

\path[draw=drawColor,line width= 0.6pt,line join=round,line cap=round,fill=fillColor] (  0.00,  0.00) rectangle (238.49,115.63);
\end{scope}
\begin{scope}
\path[clip] ( 40.51, 30.69) rectangle (232.99,110.13);
\definecolor{fillColor}{RGB}{255,255,255}

\path[fill=fillColor] ( 40.51, 30.69) rectangle (232.99,110.13);
\definecolor{drawColor}{RGB}{247,192,26}

\path[draw=drawColor,line width= 0.6pt,line join=round] ( 49.26, 43.51) --
	( 49.60, 43.76) --
	( 50.29, 44.52) --
	( 50.97, 44.06) --
	( 51.66, 44.11) --
	( 52.34, 44.08) --
	( 54.40, 43.69) --
	( 59.87, 44.91) --
	( 70.83, 45.45) --
	( 92.75, 46.12) --
	(136.58, 48.08) --
	(224.24, 49.20);
\definecolor{drawColor}{RGB}{78,155,133}

\path[draw=drawColor,line width= 0.6pt,dash pattern=on 2pt off 2pt ,line join=round] ( 49.26, 43.40) --
	( 49.60, 45.01) --
	( 50.29, 48.38) --
	( 50.97, 51.29) --
	( 51.66, 52.64) --
	( 52.34, 52.61) --
	( 54.40, 57.10) --
	( 59.87, 59.87) --
	( 70.83, 65.31) --
	( 92.75, 72.45) --
	(136.58, 83.68) --
	(224.24, 91.08);
\definecolor{drawColor}{RGB}{37,122,164}

\path[draw=drawColor,line width= 0.6pt,dash pattern=on 4pt off 2pt ,line join=round] ( 49.26, 41.48) --
	( 49.60, 44.91) --
	( 50.29, 49.59) --
	( 50.97, 52.14) --
	( 51.66, 54.84) --
	( 52.34, 56.27) --
	( 54.40, 58.54) --
	( 59.87, 63.33) --
	( 70.83, 69.25) --
	( 92.75, 77.44) --
	(136.58, 87.04) --
	(224.24, 99.63);
\definecolor{fillColor}{RGB}{78,155,133}

\path[fill=fillColor] ( 54.40, 60.15) --
	( 57.04, 55.57) --
	( 51.75, 55.57) --
	cycle;

\path[fill=fillColor] ( 50.97, 54.35) --
	( 53.61, 49.77) --
	( 48.33, 49.77) --
	cycle;

\path[fill=fillColor] ( 52.34, 55.66) --
	( 54.98, 51.08) --
	( 49.70, 51.08) --
	cycle;

\path[fill=fillColor] ( 92.75, 75.50) --
	( 95.39, 70.92) --
	( 90.11, 70.92) --
	cycle;

\path[fill=fillColor] ( 49.60, 48.06) --
	( 52.24, 43.48) --
	( 46.96, 43.48) --
	cycle;

\path[fill=fillColor] ( 51.66, 55.69) --
	( 54.30, 51.11) --
	( 49.01, 51.11) --
	cycle;

\path[fill=fillColor] (136.58, 86.74) --
	(139.22, 82.16) --
	(133.94, 82.16) --
	cycle;

\path[fill=fillColor] ( 70.83, 68.36) --
	( 73.48, 63.79) --
	( 68.19, 63.79) --
	cycle;

\path[fill=fillColor] ( 50.29, 51.43) --
	( 52.93, 46.85) --
	( 47.64, 46.85) --
	cycle;

\path[fill=fillColor] ( 59.87, 62.92) --
	( 62.52, 58.35) --
	( 57.23, 58.35) --
	cycle;

\path[fill=fillColor] (224.24, 94.13) --
	(226.88, 89.55) --
	(221.60, 89.55) --
	cycle;

\path[fill=fillColor] ( 49.26, 46.46) --
	( 51.90, 41.88) --
	( 46.62, 41.88) --
	cycle;
\definecolor{fillColor}{RGB}{247,192,26}

\path[fill=fillColor] ( 54.40, 43.69) circle (  1.96);

\path[fill=fillColor] ( 50.97, 44.06) circle (  1.96);

\path[fill=fillColor] ( 52.34, 44.08) circle (  1.96);

\path[fill=fillColor] ( 92.75, 46.12) circle (  1.96);

\path[fill=fillColor] ( 49.60, 43.76) circle (  1.96);

\path[fill=fillColor] ( 51.66, 44.11) circle (  1.96);

\path[fill=fillColor] (136.58, 48.08) circle (  1.96);

\path[fill=fillColor] ( 70.83, 45.45) circle (  1.96);

\path[fill=fillColor] ( 50.29, 44.52) circle (  1.96);

\path[fill=fillColor] ( 59.87, 44.91) circle (  1.96);

\path[fill=fillColor] (224.24, 49.20) circle (  1.96);

\path[fill=fillColor] ( 49.26, 43.51) circle (  1.96);
\definecolor{fillColor}{RGB}{37,122,164}

\path[fill=fillColor] ( 52.43, 56.58) --
	( 56.36, 56.58) --
	( 56.36, 60.50) --
	( 52.43, 60.50) --
	cycle;

\path[fill=fillColor] ( 49.01, 50.17) --
	( 52.93, 50.17) --
	( 52.93, 54.10) --
	( 49.01, 54.10) --
	cycle;

\path[fill=fillColor] ( 50.38, 54.30) --
	( 54.30, 54.30) --
	( 54.30, 58.23) --
	( 50.38, 58.23) --
	cycle;

\path[fill=fillColor] ( 90.79, 75.48) --
	( 94.71, 75.48) --
	( 94.71, 79.41) --
	( 90.79, 79.41) --
	cycle;

\path[fill=fillColor] ( 47.64, 42.95) --
	( 51.56, 42.95) --
	( 51.56, 46.87) --
	( 47.64, 46.87) --
	cycle;

\path[fill=fillColor] ( 49.69, 52.88) --
	( 53.62, 52.88) --
	( 53.62, 56.81) --
	( 49.69, 56.81) --
	cycle;

\path[fill=fillColor] (134.62, 85.08) --
	(138.54, 85.08) --
	(138.54, 89.00) --
	(134.62, 89.00) --
	cycle;

\path[fill=fillColor] ( 68.87, 67.28) --
	( 72.79, 67.28) --
	( 72.79, 71.21) --
	( 68.87, 71.21) --
	cycle;

\path[fill=fillColor] ( 48.32, 47.63) --
	( 52.25, 47.63) --
	( 52.25, 51.55) --
	( 48.32, 51.55) --
	cycle;

\path[fill=fillColor] ( 57.91, 61.37) --
	( 61.84, 61.37) --
	( 61.84, 65.30) --
	( 57.91, 65.30) --
	cycle;

\path[fill=fillColor] (222.28, 97.66) --
	(226.20, 97.66) --
	(226.20,101.59) --
	(222.28,101.59) --
	cycle;

\path[fill=fillColor] ( 47.30, 39.52) --
	( 51.22, 39.52) --
	( 51.22, 43.44) --
	( 47.30, 43.44) --
	cycle;
\definecolor{fillColor}{RGB}{247,192,26}

\path[fill=fillColor,fill opacity=0.20] ( 49.26, 46.17) --
	( 49.60, 45.67) --
	( 50.29, 46.79) --
	( 50.97, 46.32) --
	( 51.66, 46.46) --
	( 52.34, 46.65) --
	( 54.40, 46.30) --
	( 59.87, 47.40) --
	( 70.83, 48.08) --
	( 92.75, 48.79) --
	(136.58, 50.82) --
	(224.24, 52.54) --
	(224.24, 44.39) --
	(136.58, 44.40) --
	( 92.75, 42.57) --
	( 70.83, 42.00) --
	( 59.87, 41.68) --
	( 54.40, 40.25) --
	( 52.34, 40.72) --
	( 51.66, 41.10) --
	( 50.97, 41.21) --
	( 50.29, 41.65) --
	( 49.60, 41.45) --
	( 49.26, 39.98) --
	cycle;

\path[] ( 49.26, 46.17) --
	( 49.60, 45.67) --
	( 50.29, 46.79) --
	( 50.97, 46.32) --
	( 51.66, 46.46) --
	( 52.34, 46.65) --
	( 54.40, 46.30) --
	( 59.87, 47.40) --
	( 70.83, 48.08) --
	( 92.75, 48.79) --
	(136.58, 50.82) --
	(224.24, 52.54);

\path[] (224.24, 44.39) --
	(136.58, 44.40) --
	( 92.75, 42.57) --
	( 70.83, 42.00) --
	( 59.87, 41.68) --
	( 54.40, 40.25) --
	( 52.34, 40.72) --
	( 51.66, 41.10) --
	( 50.97, 41.21) --
	( 50.29, 41.65) --
	( 49.60, 41.45) --
	( 49.26, 39.98);
\definecolor{fillColor}{RGB}{78,155,133}

\path[fill=fillColor,fill opacity=0.20] ( 49.26, 46.29) --
	( 49.60, 47.88) --
	( 50.29, 53.07) --
	( 50.97, 53.84) --
	( 51.66, 55.95) --
	( 52.34, 55.51) --
	( 54.40, 59.13) --
	( 59.87, 61.91) --
	( 70.83, 68.29) --
	( 92.75, 76.67) --
	(136.58, 87.93) --
	(224.24, 96.61) --
	(224.24, 79.41) --
	(136.58, 76.67) --
	( 92.75, 65.50) --
	( 70.83, 61.22) --
	( 59.87, 57.36) --
	( 54.40, 54.59) --
	( 52.34, 48.65) --
	( 51.66, 47.86) --
	( 50.97, 47.98) --
	( 50.29, 40.03) --
	( 49.60, 41.10) --
	( 49.26, 39.48) --
	cycle;

\path[] ( 49.26, 46.29) --
	( 49.60, 47.88) --
	( 50.29, 53.07) --
	( 50.97, 53.84) --
	( 51.66, 55.95) --
	( 52.34, 55.51) --
	( 54.40, 59.13) --
	( 59.87, 61.91) --
	( 70.83, 68.29) --
	( 92.75, 76.67) --
	(136.58, 87.93) --
	(224.24, 96.61);

\path[] (224.24, 79.41) --
	(136.58, 76.67) --
	( 92.75, 65.50) --
	( 70.83, 61.22) --
	( 59.87, 57.36) --
	( 54.40, 54.59) --
	( 52.34, 48.65) --
	( 51.66, 47.86) --
	( 50.97, 47.98) --
	( 50.29, 40.03) --
	( 49.60, 41.10) --
	( 49.26, 39.48);
\definecolor{fillColor}{RGB}{37,122,164}

\path[fill=fillColor,fill opacity=0.20] ( 49.26, 45.79) --
	( 49.60, 49.80) --
	( 50.29, 53.56) --
	( 50.97, 55.38) --
	( 51.66, 57.02) --
	( 52.34, 58.23) --
	( 54.40, 60.79) --
	( 59.87, 65.46) --
	( 70.83, 72.19) --
	( 92.75, 81.45) --
	(136.58, 92.28) --
	(224.24,106.52) --
	(224.24, 76.91) --
	(136.58, 76.65) --
	( 92.75, 71.07) --
	( 70.83, 65.22) --
	( 59.87, 60.70) --
	( 54.40, 55.71) --
	( 52.34, 53.87) --
	( 51.66, 52.12) --
	( 50.97, 47.51) --
	( 50.29, 43.33) --
	( 49.60, 35.85) --
	( 49.26, 34.30) --
	cycle;

\path[] ( 49.26, 45.79) --
	( 49.60, 49.80) --
	( 50.29, 53.56) --
	( 50.97, 55.38) --
	( 51.66, 57.02) --
	( 52.34, 58.23) --
	( 54.40, 60.79) --
	( 59.87, 65.46) --
	( 70.83, 72.19) --
	( 92.75, 81.45) --
	(136.58, 92.28) --
	(224.24,106.52);

\path[] (224.24, 76.91) --
	(136.58, 76.65) --
	( 92.75, 71.07) --
	( 70.83, 65.22) --
	( 59.87, 60.70) --
	( 54.40, 55.71) --
	( 52.34, 53.87) --
	( 51.66, 52.12) --
	( 50.97, 47.51) --
	( 50.29, 43.33) --
	( 49.60, 35.85) --
	( 49.26, 34.30);
\end{scope}
\begin{scope}
\path[clip] (  0.00,  0.00) rectangle (238.49,115.63);
\definecolor{drawColor}{RGB}{0,0,0}

\path[draw=drawColor,line width= 0.6pt,line join=round] ( 40.51, 30.69) --
	( 40.51,110.13);

\path[draw=drawColor,line width= 0.6pt,line join=round] ( 41.93,107.67) --
	( 40.51,110.13) --
	( 39.09,107.67);
\end{scope}
\begin{scope}
\path[clip] (  0.00,  0.00) rectangle (238.49,115.63);
\definecolor{drawColor}{gray}{0.30}

\node[text=drawColor,anchor=base east,inner sep=0pt, outer sep=0pt, scale=  0.88] at ( 35.56, 33.85) {10};

\node[text=drawColor,anchor=base east,inner sep=0pt, outer sep=0pt, scale=  0.88] at ( 35.56, 59.13) {100};

\node[text=drawColor,anchor=base east,inner sep=0pt, outer sep=0pt, scale=  0.88] at ( 35.56, 84.41) {1000};
\end{scope}
\begin{scope}
\path[clip] (  0.00,  0.00) rectangle (238.49,115.63);
\definecolor{drawColor}{gray}{0.20}

\path[draw=drawColor,line width= 0.6pt,line join=round] ( 37.76, 36.89) --
	( 40.51, 36.89);

\path[draw=drawColor,line width= 0.6pt,line join=round] ( 37.76, 62.16) --
	( 40.51, 62.16);

\path[draw=drawColor,line width= 0.6pt,line join=round] ( 37.76, 87.44) --
	( 40.51, 87.44);
\end{scope}
\begin{scope}
\path[clip] (  0.00,  0.00) rectangle (238.49,115.63);
\definecolor{drawColor}{RGB}{0,0,0}

\path[draw=drawColor,line width= 0.6pt,line join=round] ( 40.51, 30.69) --
	(232.99, 30.69);

\path[draw=drawColor,line width= 0.6pt,line join=round] (230.53, 29.26) --
	(232.99, 30.69) --
	(230.53, 32.11);
\end{scope}
\begin{scope}
\path[clip] (  0.00,  0.00) rectangle (238.49,115.63);
\definecolor{drawColor}{gray}{0.20}

\path[draw=drawColor,line width= 0.6pt,line join=round] ( 48.92, 27.94) --
	( 48.92, 30.69);

\path[draw=drawColor,line width= 0.6pt,line join=round] ( 91.72, 27.94) --
	( 91.72, 30.69);

\path[draw=drawColor,line width= 0.6pt,line join=round] (134.52, 27.94) --
	(134.52, 30.69);

\path[draw=drawColor,line width= 0.6pt,line join=round] (177.33, 27.94) --
	(177.33, 30.69);

\path[draw=drawColor,line width= 0.6pt,line join=round] (220.13, 27.94) --
	(220.13, 30.69);
\end{scope}
\begin{scope}
\path[clip] (  0.00,  0.00) rectangle (238.49,115.63);
\definecolor{drawColor}{gray}{0.30}

\node[text=drawColor,anchor=base,inner sep=0pt, outer sep=0pt, scale=  0.88] at ( 48.92, 19.68) {0};

\node[text=drawColor,anchor=base,inner sep=0pt, outer sep=0pt, scale=  0.88] at ( 91.72, 19.68) {250};

\node[text=drawColor,anchor=base,inner sep=0pt, outer sep=0pt, scale=  0.88] at (134.52, 19.68) {500};

\node[text=drawColor,anchor=base,inner sep=0pt, outer sep=0pt, scale=  0.88] at (177.33, 19.68) {750};

\node[text=drawColor,anchor=base,inner sep=0pt, outer sep=0pt, scale=  0.88] at (220.13, 19.68) {1000};
\end{scope}
\begin{scope}
\path[clip] (  0.00,  0.00) rectangle (238.49,115.63);
\definecolor{drawColor}{RGB}{0,0,0}

\node[text=drawColor,anchor=base,inner sep=0pt, outer sep=0pt, scale=  1.10] at (136.75,  7.64) {$d$};
\end{scope}
\begin{scope}
\path[clip] (  0.00,  0.00) rectangle (238.49,115.63);
\definecolor{drawColor}{RGB}{0,0,0}

\node[text=drawColor,rotate= 90.00,anchor=base,inner sep=0pt, outer sep=0pt, scale=  1.10] at ( 13.08, 70.41) {time (ms)};
\end{scope}
\begin{scope}
\path[clip] (  0.00,  0.00) rectangle (238.49,115.63);

\path[] ( 37.04, 71.03) rectangle (109.43,125.40);
\end{scope}
\begin{scope}
\path[clip] (  0.00,  0.00) rectangle (238.49,115.63);
\definecolor{drawColor}{RGB}{247,192,26}

\path[draw=drawColor,line width= 0.6pt,line join=round] ( 43.98,112.67) -- ( 55.55,112.67);
\end{scope}
\begin{scope}
\path[clip] (  0.00,  0.00) rectangle (238.49,115.63);
\definecolor{fillColor}{RGB}{247,192,26}

\path[fill=fillColor] ( 49.76,112.67) circle (  1.96);
\end{scope}
\begin{scope}
\path[clip] (  0.00,  0.00) rectangle (238.49,115.63);
\definecolor{drawColor}{RGB}{78,155,133}

\path[draw=drawColor,line width= 0.6pt,dash pattern=on 2pt off 2pt ,line join=round] ( 43.98, 98.22) -- ( 55.55, 98.22);
\end{scope}
\begin{scope}
\path[clip] (  0.00,  0.00) rectangle (238.49,115.63);
\definecolor{fillColor}{RGB}{78,155,133}

\path[fill=fillColor] ( 49.76,101.27) --
	( 52.41, 96.69) --
	( 47.12, 96.69) --
	cycle;
\end{scope}
\begin{scope}
\path[clip] (  0.00,  0.00) rectangle (238.49,115.63);
\definecolor{drawColor}{RGB}{37,122,164}

\path[draw=drawColor,line width= 0.6pt,dash pattern=on 4pt off 2pt ,line join=round] ( 43.98, 83.76) -- ( 55.55, 83.76);
\end{scope}
\begin{scope}
\path[clip] (  0.00,  0.00) rectangle (238.49,115.63);
\definecolor{fillColor}{RGB}{37,122,164}

\path[fill=fillColor] ( 47.80, 81.80) --
	( 51.73, 81.80) --
	( 51.73, 85.72) --
	( 47.80, 85.72) --
	cycle;
\end{scope}
\begin{scope}
\path[clip] (  0.00,  0.00) rectangle (238.49,115.63);
\definecolor{drawColor}{RGB}{0,0,0}

\node[text=drawColor,anchor=base west,inner sep=0pt, outer sep=0pt, scale=  0.80] at ( 62.49,109.91) {LVE (ACP)};
\end{scope}
\begin{scope}
\path[clip] (  0.00,  0.00) rectangle (238.49,115.63);
\definecolor{drawColor}{RGB}{0,0,0}

\node[text=drawColor,anchor=base west,inner sep=0pt, outer sep=0pt, scale=  0.80] at ( 62.49, 95.46) {LVE (CP)};
\end{scope}
\begin{scope}
\path[clip] (  0.00,  0.00) rectangle (238.49,115.63);
\definecolor{drawColor}{RGB}{0,0,0}

\node[text=drawColor,anchor=base west,inner sep=0pt, outer sep=0pt, scale=  0.80] at ( 62.49, 81.01) {VE};
\end{scope}
\end{tikzpicture}

%% file: figures/plot-results-inter-p=10.tex
% Created by tikzDevice version 0.12.3.1 on 2023-08-04 17:02:04
% !TEX encoding = UTF-8 Unicode
\begin{tikzpicture}[x=1pt,y=1pt]
\definecolor{fillColor}{RGB}{255,255,255}
\path[use as bounding box,fill=fillColor,fill opacity=0.00] (0,15) rectangle (238.49,115.63);
\begin{scope}
\path[clip] (  0.00,  0.00) rectangle (238.49,115.63);
\definecolor{drawColor}{RGB}{255,255,255}
\definecolor{fillColor}{RGB}{255,255,255}

\path[draw=drawColor,line width= 0.6pt,line join=round,line cap=round,fill=fillColor] (  0.00,  0.00) rectangle (238.49,115.63);
\end{scope}
\begin{scope}
\path[clip] ( 40.51, 30.69) rectangle (232.99,110.13);
\definecolor{fillColor}{RGB}{255,255,255}

\path[fill=fillColor] ( 40.51, 30.69) rectangle (232.99,110.13);
\definecolor{drawColor}{RGB}{247,192,26}

\path[draw=drawColor,line width= 0.6pt,line join=round] ( 49.26, 44.78) --
	( 49.60, 44.74) --
	( 50.29, 44.51) --
	( 50.97, 44.39) --
	( 51.66, 44.30) --
	( 52.34, 44.84) --
	( 54.40, 45.17) --
	( 59.87, 45.47) --
	( 70.83, 46.15) --
	( 92.75, 46.55) --
	(136.58, 47.79) --
	(224.24, 49.25);
\definecolor{drawColor}{RGB}{78,155,133}

\path[draw=drawColor,line width= 0.6pt,dash pattern=on 2pt off 2pt ,line join=round] ( 49.26, 44.07) --
	( 49.60, 45.76) --
	( 50.29, 49.37) --
	( 50.97, 54.60) --
	( 51.66, 56.30) --
	( 52.34, 56.30) --
	( 54.40, 58.23) --
	( 59.87, 61.56) --
	( 70.83, 70.38) --
	( 92.75, 83.82) --
	(136.58, 93.69) --
	(224.24,101.89);
\definecolor{drawColor}{RGB}{37,122,164}

\path[draw=drawColor,line width= 0.6pt,dash pattern=on 4pt off 2pt ,line join=round] ( 49.26, 41.87) --
	( 49.60, 45.28) --
	( 50.29, 49.89) --
	( 50.97, 51.87) --
	( 51.66, 54.30) --
	( 52.34, 56.46) --
	( 54.40, 58.58) --
	( 59.87, 62.72) --
	( 70.83, 70.25) --
	( 92.75, 77.78) --
	(136.58, 85.75) --
	(224.24, 98.83);
\definecolor{fillColor}{RGB}{78,155,133}

\path[fill=fillColor] ( 54.40, 61.28) --
	( 57.04, 56.70) --
	( 51.75, 56.70) --
	cycle;

\path[fill=fillColor] ( 50.97, 57.65) --
	( 53.61, 53.07) --
	( 48.33, 53.07) --
	cycle;

\path[fill=fillColor] ( 52.34, 59.35) --
	( 54.98, 54.77) --
	( 49.70, 54.77) --
	cycle;

\path[fill=fillColor] ( 92.75, 86.87) --
	( 95.39, 82.30) --
	( 90.11, 82.30) --
	cycle;

\path[fill=fillColor] ( 49.60, 48.81) --
	( 52.24, 44.24) --
	( 46.96, 44.24) --
	cycle;

\path[fill=fillColor] ( 51.66, 59.35) --
	( 54.30, 54.78) --
	( 49.01, 54.78) --
	cycle;

\path[fill=fillColor] (136.58, 96.74) --
	(139.22, 92.17) --
	(133.94, 92.17) --
	cycle;

\path[fill=fillColor] ( 70.83, 73.43) --
	( 73.48, 68.85) --
	( 68.19, 68.85) --
	cycle;

\path[fill=fillColor] ( 50.29, 52.42) --
	( 52.93, 47.84) --
	( 47.64, 47.84) --
	cycle;

\path[fill=fillColor] ( 59.87, 64.61) --
	( 62.52, 60.04) --
	( 57.23, 60.04) --
	cycle;

\path[fill=fillColor] (224.24,104.94) --
	(226.88,100.36) --
	(221.60,100.36) --
	cycle;

\path[fill=fillColor] ( 49.26, 47.12) --
	( 51.90, 42.55) --
	( 46.62, 42.55) --
	cycle;
\definecolor{fillColor}{RGB}{247,192,26}

\path[fill=fillColor] ( 54.40, 45.17) circle (  1.96);

\path[fill=fillColor] ( 50.97, 44.39) circle (  1.96);

\path[fill=fillColor] ( 52.34, 44.84) circle (  1.96);

\path[fill=fillColor] ( 92.75, 46.55) circle (  1.96);

\path[fill=fillColor] ( 49.60, 44.74) circle (  1.96);

\path[fill=fillColor] ( 51.66, 44.30) circle (  1.96);

\path[fill=fillColor] (136.58, 47.79) circle (  1.96);

\path[fill=fillColor] ( 70.83, 46.15) circle (  1.96);

\path[fill=fillColor] ( 50.29, 44.51) circle (  1.96);

\path[fill=fillColor] ( 59.87, 45.47) circle (  1.96);

\path[fill=fillColor] (224.24, 49.25) circle (  1.96);

\path[fill=fillColor] ( 49.26, 44.78) circle (  1.96);
\definecolor{fillColor}{RGB}{37,122,164}

\path[fill=fillColor] ( 52.43, 56.62) --
	( 56.36, 56.62) --
	( 56.36, 60.54) --
	( 52.43, 60.54) --
	cycle;

\path[fill=fillColor] ( 49.01, 49.90) --
	( 52.93, 49.90) --
	( 52.93, 53.83) --
	( 49.01, 53.83) --
	cycle;

\path[fill=fillColor] ( 50.38, 54.50) --
	( 54.30, 54.50) --
	( 54.30, 58.42) --
	( 50.38, 58.42) --
	cycle;

\path[fill=fillColor] ( 90.79, 75.82) --
	( 94.71, 75.82) --
	( 94.71, 79.75) --
	( 90.79, 79.75) --
	cycle;

\path[fill=fillColor] ( 47.64, 43.31) --
	( 51.56, 43.31) --
	( 51.56, 47.24) --
	( 47.64, 47.24) --
	cycle;

\path[fill=fillColor] ( 49.69, 52.34) --
	( 53.62, 52.34) --
	( 53.62, 56.26) --
	( 49.69, 56.26) --
	cycle;

\path[fill=fillColor] (134.62, 83.79) --
	(138.54, 83.79) --
	(138.54, 87.72) --
	(134.62, 87.72) --
	cycle;

\path[fill=fillColor] ( 68.87, 68.28) --
	( 72.79, 68.28) --
	( 72.79, 72.21) --
	( 68.87, 72.21) --
	cycle;

\path[fill=fillColor] ( 48.32, 47.92) --
	( 52.25, 47.92) --
	( 52.25, 51.85) --
	( 48.32, 51.85) --
	cycle;

\path[fill=fillColor] ( 57.91, 60.76) --
	( 61.84, 60.76) --
	( 61.84, 64.68) --
	( 57.91, 64.68) --
	cycle;

\path[fill=fillColor] (222.28, 96.87) --
	(226.20, 96.87) --
	(226.20,100.79) --
	(222.28,100.79) --
	cycle;

\path[fill=fillColor] ( 47.30, 39.90) --
	( 51.22, 39.90) --
	( 51.22, 43.83) --
	( 47.30, 43.83) --
	cycle;
\definecolor{fillColor}{RGB}{247,192,26}

\path[fill=fillColor,fill opacity=0.20] ( 49.26, 47.66) --
	( 49.60, 47.13) --
	( 50.29, 46.60) --
	( 50.97, 46.59) --
	( 51.66, 46.64) --
	( 52.34, 46.62) --
	( 54.40, 47.66) --
	( 59.87, 48.39) --
	( 70.83, 48.82) --
	( 92.75, 49.68) --
	(136.58, 51.05) --
	(224.24, 52.40) --
	(224.24, 44.79) --
	(136.58, 43.08) --
	( 92.75, 42.12) --
	( 70.83, 42.60) --
	( 59.87, 41.47) --
	( 54.40, 41.92) --
	( 52.34, 42.71) --
	( 51.66, 41.30) --
	( 50.97, 41.63) --
	( 50.29, 41.90) --
	( 49.60, 41.68) --
	( 49.26, 40.85) --
	cycle;

\path[] ( 49.26, 47.66) --
	( 49.60, 47.13) --
	( 50.29, 46.60) --
	( 50.97, 46.59) --
	( 51.66, 46.64) --
	( 52.34, 46.62) --
	( 54.40, 47.66) --
	( 59.87, 48.39) --
	( 70.83, 48.82) --
	( 92.75, 49.68) --
	(136.58, 51.05) --
	(224.24, 52.40);

\path[] (224.24, 44.79) --
	(136.58, 43.08) --
	( 92.75, 42.12) --
	( 70.83, 42.60) --
	( 59.87, 41.47) --
	( 54.40, 41.92) --
	( 52.34, 42.71) --
	( 51.66, 41.30) --
	( 50.97, 41.63) --
	( 50.29, 41.90) --
	( 49.60, 41.68) --
	( 49.26, 40.85);
\definecolor{fillColor}{RGB}{78,155,133}

\path[fill=fillColor,fill opacity=0.20] ( 49.26, 46.44) --
	( 49.60, 48.85) --
	( 50.29, 53.82) --
	( 50.97, 55.72) --
	( 51.66, 57.82) --
	( 52.34, 58.80) --
	( 54.40, 60.39) --
	( 59.87, 64.85) --
	( 70.83, 74.31) --
	( 92.75, 88.11) --
	(136.58, 97.38) --
	(224.24,106.52) --
	(224.24, 93.58) --
	(136.58, 88.05) --
	( 92.75, 76.59) --
	( 70.83, 64.11) --
	( 59.87, 56.80) --
	( 54.40, 55.52) --
	( 52.34, 53.03) --
	( 51.66, 54.53) --
	( 50.97, 53.35) --
	( 50.29, 41.65) --
	( 49.60, 41.41) --
	( 49.26, 41.03) --
	cycle;

\path[] ( 49.26, 46.44) --
	( 49.60, 48.85) --
	( 50.29, 53.82) --
	( 50.97, 55.72) --
	( 51.66, 57.82) --
	( 52.34, 58.80) --
	( 54.40, 60.39) --
	( 59.87, 64.85) --
	( 70.83, 74.31) --
	( 92.75, 88.11) --
	(136.58, 97.38) --
	(224.24,106.52);

\path[] (224.24, 93.58) --
	(136.58, 88.05) --
	( 92.75, 76.59) --
	( 70.83, 64.11) --
	( 59.87, 56.80) --
	( 54.40, 55.52) --
	( 52.34, 53.03) --
	( 51.66, 54.53) --
	( 50.97, 53.35) --
	( 50.29, 41.65) --
	( 49.60, 41.41) --
	( 49.26, 41.03);
\definecolor{fillColor}{RGB}{37,122,164}

\path[fill=fillColor,fill opacity=0.20] ( 49.26, 46.27) --
	( 49.60, 49.80) --
	( 50.29, 53.42) --
	( 50.97, 54.26) --
	( 51.66, 56.51) --
	( 52.34, 58.06) --
	( 54.40, 59.54) --
	( 59.87, 65.01) --
	( 70.83, 73.41) --
	( 92.75, 81.70) --
	(136.58, 90.82) --
	(224.24,105.59) --
	(224.24, 76.60) --
	(136.58, 75.89) --
	( 92.75, 71.57) --
	( 70.83, 65.76) --
	( 59.87, 59.82) --
	( 54.40, 57.53) --
	( 52.34, 54.58) --
	( 51.66, 51.51) --
	( 50.97, 48.78) --
	( 50.29, 44.59) --
	( 49.60, 37.32) --
	( 49.26, 34.30) --
	cycle;

\path[] ( 49.26, 46.27) --
	( 49.60, 49.80) --
	( 50.29, 53.42) --
	( 50.97, 54.26) --
	( 51.66, 56.51) --
	( 52.34, 58.06) --
	( 54.40, 59.54) --
	( 59.87, 65.01) --
	( 70.83, 73.41) --
	( 92.75, 81.70) --
	(136.58, 90.82) --
	(224.24,105.59);

\path[] (224.24, 76.60) --
	(136.58, 75.89) --
	( 92.75, 71.57) --
	( 70.83, 65.76) --
	( 59.87, 59.82) --
	( 54.40, 57.53) --
	( 52.34, 54.58) --
	( 51.66, 51.51) --
	( 50.97, 48.78) --
	( 50.29, 44.59) --
	( 49.60, 37.32) --
	( 49.26, 34.30);
\end{scope}
\begin{scope}
\path[clip] (  0.00,  0.00) rectangle (238.49,115.63);
\definecolor{drawColor}{RGB}{0,0,0}

\path[draw=drawColor,line width= 0.6pt,line join=round] ( 40.51, 30.69) --
	( 40.51,110.13);

\path[draw=drawColor,line width= 0.6pt,line join=round] ( 41.93,107.67) --
	( 40.51,110.13) --
	( 39.09,107.67);
\end{scope}
\begin{scope}
\path[clip] (  0.00,  0.00) rectangle (238.49,115.63);
\definecolor{drawColor}{gray}{0.30}

\node[text=drawColor,anchor=base east,inner sep=0pt, outer sep=0pt, scale=  0.88] at ( 35.56, 34.44) {10};

\node[text=drawColor,anchor=base east,inner sep=0pt, outer sep=0pt, scale=  0.88] at ( 35.56, 59.25) {100};

\node[text=drawColor,anchor=base east,inner sep=0pt, outer sep=0pt, scale=  0.88] at ( 35.56, 84.07) {1000};
\end{scope}
\begin{scope}
\path[clip] (  0.00,  0.00) rectangle (238.49,115.63);
\definecolor{drawColor}{gray}{0.20}

\path[draw=drawColor,line width= 0.6pt,line join=round] ( 37.76, 37.47) --
	( 40.51, 37.47);

\path[draw=drawColor,line width= 0.6pt,line join=round] ( 37.76, 62.28) --
	( 40.51, 62.28);

\path[draw=drawColor,line width= 0.6pt,line join=round] ( 37.76, 87.10) --
	( 40.51, 87.10);
\end{scope}
\begin{scope}
\path[clip] (  0.00,  0.00) rectangle (238.49,115.63);
\definecolor{drawColor}{RGB}{0,0,0}

\path[draw=drawColor,line width= 0.6pt,line join=round] ( 40.51, 30.69) --
	(232.99, 30.69);

\path[draw=drawColor,line width= 0.6pt,line join=round] (230.53, 29.26) --
	(232.99, 30.69) --
	(230.53, 32.11);
\end{scope}
\begin{scope}
\path[clip] (  0.00,  0.00) rectangle (238.49,115.63);
\definecolor{drawColor}{gray}{0.20}

\path[draw=drawColor,line width= 0.6pt,line join=round] ( 48.92, 27.94) --
	( 48.92, 30.69);

\path[draw=drawColor,line width= 0.6pt,line join=round] ( 91.72, 27.94) --
	( 91.72, 30.69);

\path[draw=drawColor,line width= 0.6pt,line join=round] (134.52, 27.94) --
	(134.52, 30.69);

\path[draw=drawColor,line width= 0.6pt,line join=round] (177.33, 27.94) --
	(177.33, 30.69);

\path[draw=drawColor,line width= 0.6pt,line join=round] (220.13, 27.94) --
	(220.13, 30.69);
\end{scope}
\begin{scope}
\path[clip] (  0.00,  0.00) rectangle (238.49,115.63);
\definecolor{drawColor}{gray}{0.30}

\node[text=drawColor,anchor=base,inner sep=0pt, outer sep=0pt, scale=  0.88] at ( 48.92, 19.68) {0};

\node[text=drawColor,anchor=base,inner sep=0pt, outer sep=0pt, scale=  0.88] at ( 91.72, 19.68) {250};

\node[text=drawColor,anchor=base,inner sep=0pt, outer sep=0pt, scale=  0.88] at (134.52, 19.68) {500};

\node[text=drawColor,anchor=base,inner sep=0pt, outer sep=0pt, scale=  0.88] at (177.33, 19.68) {750};

\node[text=drawColor,anchor=base,inner sep=0pt, outer sep=0pt, scale=  0.88] at (220.13, 19.68) {1000};
\end{scope}
\begin{scope}
\path[clip] (  0.00,  0.00) rectangle (238.49,115.63);
\definecolor{drawColor}{RGB}{0,0,0}

\node[text=drawColor,anchor=base,inner sep=0pt, outer sep=0pt, scale=  1.10] at (136.75,  7.64) {$d$};
\end{scope}
\begin{scope}
\path[clip] (  0.00,  0.00) rectangle (238.49,115.63);
\definecolor{drawColor}{RGB}{0,0,0}

\node[text=drawColor,rotate= 90.00,anchor=base,inner sep=0pt, outer sep=0pt, scale=  1.10] at ( 13.08, 70.41) {time (ms)};
\end{scope}
\begin{scope}
\path[clip] (  0.00,  0.00) rectangle (238.49,115.63);

\path[] ( 37.04, 71.03) rectangle (109.43,125.40);
\end{scope}
\begin{scope}
\path[clip] (  0.00,  0.00) rectangle (238.49,115.63);
\definecolor{drawColor}{RGB}{247,192,26}

\path[draw=drawColor,line width= 0.6pt,line join=round] ( 43.98,112.67) -- ( 55.55,112.67);
\end{scope}
\begin{scope}
\path[clip] (  0.00,  0.00) rectangle (238.49,115.63);
\definecolor{fillColor}{RGB}{247,192,26}

\path[fill=fillColor] ( 49.76,112.67) circle (  1.96);
\end{scope}
\begin{scope}
\path[clip] (  0.00,  0.00) rectangle (238.49,115.63);
\definecolor{drawColor}{RGB}{78,155,133}

\path[draw=drawColor,line width= 0.6pt,dash pattern=on 2pt off 2pt ,line join=round] ( 43.98, 98.22) -- ( 55.55, 98.22);
\end{scope}
\begin{scope}
\path[clip] (  0.00,  0.00) rectangle (238.49,115.63);
\definecolor{fillColor}{RGB}{78,155,133}

\path[fill=fillColor] ( 49.76,101.27) --
	( 52.41, 96.69) --
	( 47.12, 96.69) --
	cycle;
\end{scope}
\begin{scope}
\path[clip] (  0.00,  0.00) rectangle (238.49,115.63);
\definecolor{drawColor}{RGB}{37,122,164}

\path[draw=drawColor,line width= 0.6pt,dash pattern=on 4pt off 2pt ,line join=round] ( 43.98, 83.76) -- ( 55.55, 83.76);
\end{scope}
\begin{scope}
\path[clip] (  0.00,  0.00) rectangle (238.49,115.63);
\definecolor{fillColor}{RGB}{37,122,164}

\path[fill=fillColor] ( 47.80, 81.80) --
	( 51.73, 81.80) --
	( 51.73, 85.72) --
	( 47.80, 85.72) --
	cycle;
\end{scope}
\begin{scope}
\path[clip] (  0.00,  0.00) rectangle (238.49,115.63);
\definecolor{drawColor}{RGB}{0,0,0}

\node[text=drawColor,anchor=base west,inner sep=0pt, outer sep=0pt, scale=  0.80] at ( 62.49,109.91) {LVE (ACP)};
\end{scope}
\begin{scope}
\path[clip] (  0.00,  0.00) rectangle (238.49,115.63);
\definecolor{drawColor}{RGB}{0,0,0}

\node[text=drawColor,anchor=base west,inner sep=0pt, outer sep=0pt, scale=  0.80] at ( 62.49, 95.46) {LVE (CP)};
\end{scope}
\begin{scope}
\path[clip] (  0.00,  0.00) rectangle (238.49,115.63);
\definecolor{drawColor}{RGB}{0,0,0}

\node[text=drawColor,anchor=base west,inner sep=0pt, outer sep=0pt, scale=  0.80] at ( 62.49, 81.01) {VE};
\end{scope}
\end{tikzpicture}

%% file: figures/plot-results-inter-p=15.tex
% Created by tikzDevice version 0.12.3.1 on 2023-08-04 17:02:04
% !TEX encoding = UTF-8 Unicode
\begin{tikzpicture}[x=1pt,y=1pt]
\definecolor{fillColor}{RGB}{255,255,255}
\path[use as bounding box,fill=fillColor,fill opacity=0.00] (0,15) rectangle (238.49,115.63);
\begin{scope}
\path[clip] (  0.00,  0.00) rectangle (238.49,115.63);
\definecolor{drawColor}{RGB}{255,255,255}
\definecolor{fillColor}{RGB}{255,255,255}

\path[draw=drawColor,line width= 0.6pt,line join=round,line cap=round,fill=fillColor] (  0.00,  0.00) rectangle (238.49,115.63);
\end{scope}
\begin{scope}
\path[clip] ( 44.91, 30.69) rectangle (232.99,110.13);
\definecolor{fillColor}{RGB}{255,255,255}

\path[fill=fillColor] ( 44.91, 30.69) rectangle (232.99,110.13);
\definecolor{drawColor}{RGB}{247,192,26}

\path[draw=drawColor,line width= 0.6pt,line join=round] ( 53.46, 41.30) --
	( 53.79, 42.00) --
	( 54.46, 42.78) --
	( 55.13, 42.50) --
	( 55.80, 42.01) --
	( 56.47, 42.68) --
	( 58.48, 42.08) --
	( 63.83, 42.49) --
	( 74.54, 44.21) --
	( 95.95, 44.57) --
	(138.78, 46.09) --
	(224.44, 46.66);
\definecolor{drawColor}{RGB}{78,155,133}

\path[draw=drawColor,line width= 0.6pt,dash pattern=on 2pt off 2pt ,line join=round] ( 53.46, 41.26) --
	( 53.79, 44.04) --
	( 54.46, 46.03) --
	( 55.13, 51.07) --
	( 55.80, 52.55) --
	( 56.47, 53.01) --
	( 58.48, 55.69) --
	( 63.83, 62.03) --
	( 74.54, 75.02) --
	( 95.95, 82.74) --
	(138.78, 91.06) --
	(224.44,102.44);
\definecolor{drawColor}{RGB}{37,122,164}

\path[draw=drawColor,line width= 0.6pt,dash pattern=on 4pt off 2pt ,line join=round] ( 53.46, 39.89) --
	( 53.79, 43.49) --
	( 54.46, 47.38) --
	( 55.13, 49.82) --
	( 55.80, 51.45) --
	( 56.47, 51.64) --
	( 58.48, 54.96) --
	( 63.83, 58.77) --
	( 74.54, 64.01) --
	( 95.95, 71.05) --
	(138.78, 79.02) --
	(224.44, 90.18);
\definecolor{fillColor}{RGB}{78,155,133}

\path[fill=fillColor] ( 58.48, 58.75) --
	( 61.12, 54.17) --
	( 55.83, 54.17) --
	cycle;

\path[fill=fillColor] ( 55.13, 54.12) --
	( 57.77, 49.55) --
	( 52.49, 49.55) --
	cycle;

\path[fill=fillColor] ( 56.47, 56.06) --
	( 59.11, 51.48) --
	( 53.83, 51.48) --
	cycle;

\path[fill=fillColor] ( 95.95, 85.79) --
	( 98.60, 81.22) --
	( 93.31, 81.22) --
	cycle;

\path[fill=fillColor] ( 53.79, 47.09) --
	( 56.44, 42.52) --
	( 51.15, 42.52) --
	cycle;

\path[fill=fillColor] ( 55.80, 55.60) --
	( 58.44, 51.02) --
	( 53.16, 51.02) --
	cycle;

\path[fill=fillColor] (138.78, 94.11) --
	(141.43, 89.54) --
	(136.14, 89.54) --
	cycle;

\path[fill=fillColor] ( 74.54, 78.07) --
	( 77.18, 73.49) --
	( 71.90, 73.49) --
	cycle;

\path[fill=fillColor] ( 54.46, 49.08) --
	( 57.10, 44.50) --
	( 51.82, 44.50) --
	cycle;

\path[fill=fillColor] ( 63.83, 65.08) --
	( 66.47, 60.50) --
	( 61.19, 60.50) --
	cycle;

\path[fill=fillColor] (224.44,105.49) --
	(227.08,100.92) --
	(221.80,100.92) --
	cycle;

\path[fill=fillColor] ( 53.46, 44.31) --
	( 56.10, 39.74) --
	( 50.82, 39.74) --
	cycle;
\definecolor{fillColor}{RGB}{247,192,26}

\path[fill=fillColor] ( 58.48, 42.08) circle (  1.96);

\path[fill=fillColor] ( 55.13, 42.50) circle (  1.96);

\path[fill=fillColor] ( 56.47, 42.68) circle (  1.96);

\path[fill=fillColor] ( 95.95, 44.57) circle (  1.96);

\path[fill=fillColor] ( 53.79, 42.00) circle (  1.96);

\path[fill=fillColor] ( 55.80, 42.01) circle (  1.96);

\path[fill=fillColor] (138.78, 46.09) circle (  1.96);

\path[fill=fillColor] ( 74.54, 44.21) circle (  1.96);

\path[fill=fillColor] ( 54.46, 42.78) circle (  1.96);

\path[fill=fillColor] ( 63.83, 42.49) circle (  1.96);

\path[fill=fillColor] (224.44, 46.66) circle (  1.96);

\path[fill=fillColor] ( 53.46, 41.30) circle (  1.96);
\definecolor{fillColor}{RGB}{37,122,164}

\path[fill=fillColor] ( 56.52, 53.00) --
	( 60.44, 53.00) --
	( 60.44, 56.92) --
	( 56.52, 56.92) --
	cycle;

\path[fill=fillColor] ( 53.17, 47.86) --
	( 57.09, 47.86) --
	( 57.09, 51.78) --
	( 53.17, 51.78) --
	cycle;

\path[fill=fillColor] ( 54.51, 49.68) --
	( 58.43, 49.68) --
	( 58.43, 53.61) --
	( 54.51, 53.61) --
	cycle;

\path[fill=fillColor] ( 93.99, 69.08) --
	( 97.92, 69.08) --
	( 97.92, 73.01) --
	( 93.99, 73.01) --
	cycle;

\path[fill=fillColor] ( 51.83, 41.53) --
	( 55.75, 41.53) --
	( 55.75, 45.45) --
	( 51.83, 45.45) --
	cycle;

\path[fill=fillColor] ( 53.84, 49.49) --
	( 57.76, 49.49) --
	( 57.76, 53.42) --
	( 53.84, 53.42) --
	cycle;

\path[fill=fillColor] (136.82, 77.05) --
	(140.74, 77.05) --
	(140.74, 80.98) --
	(136.82, 80.98) --
	cycle;

\path[fill=fillColor] ( 72.58, 62.05) --
	( 76.50, 62.05) --
	( 76.50, 65.98) --
	( 72.58, 65.98) --
	cycle;

\path[fill=fillColor] ( 52.50, 45.41) --
	( 56.42, 45.41) --
	( 56.42, 49.34) --
	( 52.50, 49.34) --
	cycle;

\path[fill=fillColor] ( 61.87, 56.81) --
	( 65.79, 56.81) --
	( 65.79, 60.74) --
	( 61.87, 60.74) --
	cycle;

\path[fill=fillColor] (222.48, 88.22) --
	(226.40, 88.22) --
	(226.40, 92.14) --
	(222.48, 92.14) --
	cycle;

\path[fill=fillColor] ( 51.50, 37.93) --
	( 55.42, 37.93) --
	( 55.42, 41.85) --
	( 51.50, 41.85) --
	cycle;
\definecolor{fillColor}{RGB}{247,192,26}

\path[fill=fillColor,fill opacity=0.20] ( 53.46, 42.91) --
	( 53.79, 43.90) --
	( 54.46, 45.06) --
	( 55.13, 44.36) --
	( 55.80, 43.66) --
	( 56.47, 44.51) --
	( 58.48, 43.74) --
	( 63.83, 44.72) --
	( 74.54, 45.36) --
	( 95.95, 46.98) --
	(138.78, 49.04) --
	(224.44, 48.63) --
	(224.44, 44.16) --
	(138.78, 41.79) --
	( 95.95, 41.34) --
	( 74.54, 42.91) --
	( 63.83, 39.56) --
	( 58.48, 40.06) --
	( 56.47, 40.41) --
	( 55.80, 40.02) --
	( 55.13, 40.17) --
	( 54.46, 39.78) --
	( 53.79, 39.63) --
	( 53.46, 39.36) --
	cycle;

\path[] ( 53.46, 42.91) --
	( 53.79, 43.90) --
	( 54.46, 45.06) --
	( 55.13, 44.36) --
	( 55.80, 43.66) --
	( 56.47, 44.51) --
	( 58.48, 43.74) --
	( 63.83, 44.72) --
	( 74.54, 45.36) --
	( 95.95, 46.98) --
	(138.78, 49.04) --
	(224.44, 48.63);

\path[] (224.44, 44.16) --
	(138.78, 41.79) --
	( 95.95, 41.34) --
	( 74.54, 42.91) --
	( 63.83, 39.56) --
	( 58.48, 40.06) --
	( 56.47, 40.41) --
	( 55.80, 40.02) --
	( 55.13, 40.17) --
	( 54.46, 39.78) --
	( 53.79, 39.63) --
	( 53.46, 39.36);
\definecolor{fillColor}{RGB}{78,155,133}

\path[fill=fillColor,fill opacity=0.20] ( 53.46, 42.99) --
	( 53.79, 46.55) --
	( 54.46, 50.72) --
	( 55.13, 53.69) --
	( 55.80, 54.81) --
	( 56.47, 55.34) --
	( 58.48, 58.11) --
	( 63.83, 64.47) --
	( 74.54, 79.50) --
	( 95.95, 85.38) --
	(138.78, 94.19) --
	(224.44,106.52) --
	(224.44, 95.13) --
	(138.78, 86.35) --
	( 95.95, 79.08) --
	( 74.54, 66.26) --
	( 63.83, 58.73) --
	( 58.48, 52.45) --
	( 56.47, 49.91) --
	( 55.80, 49.56) --
	( 55.13, 47.44) --
	( 54.46, 36.36) --
	( 53.79, 40.62) --
	( 53.46, 39.15) --
	cycle;

\path[] ( 53.46, 42.99) --
	( 53.79, 46.55) --
	( 54.46, 50.72) --
	( 55.13, 53.69) --
	( 55.80, 54.81) --
	( 56.47, 55.34) --
	( 58.48, 58.11) --
	( 63.83, 64.47) --
	( 74.54, 79.50) --
	( 95.95, 85.38) --
	(138.78, 94.19) --
	(224.44,106.52);

\path[] (224.44, 95.13) --
	(138.78, 86.35) --
	( 95.95, 79.08) --
	( 74.54, 66.26) --
	( 63.83, 58.73) --
	( 58.48, 52.45) --
	( 56.47, 49.91) --
	( 55.80, 49.56) --
	( 55.13, 47.44) --
	( 54.46, 36.36) --
	( 53.79, 40.62) --
	( 53.46, 39.15);
\definecolor{fillColor}{RGB}{37,122,164}

\path[fill=fillColor,fill opacity=0.20] ( 53.46, 43.39) --
	( 53.79, 48.04) --
	( 54.46, 50.44) --
	( 55.13, 52.61) --
	( 55.80, 53.67) --
	( 56.47, 53.54) --
	( 58.48, 56.81) --
	( 63.83, 60.91) --
	( 74.54, 66.35) --
	( 95.95, 74.23) --
	(138.78, 83.45) --
	(224.44, 95.91) --
	(224.44, 73.41) --
	(138.78, 70.44) --
	( 95.95, 66.21) --
	( 74.54, 60.90) --
	( 63.83, 56.01) --
	( 58.48, 52.66) --
	( 56.47, 49.28) --
	( 55.80, 48.56) --
	( 55.13, 45.85) --
	( 54.46, 42.83) --
	( 53.79, 34.41) --
	( 53.46, 34.30) --
	cycle;

\path[] ( 53.46, 43.39) --
	( 53.79, 48.04) --
	( 54.46, 50.44) --
	( 55.13, 52.61) --
	( 55.80, 53.67) --
	( 56.47, 53.54) --
	( 58.48, 56.81) --
	( 63.83, 60.91) --
	( 74.54, 66.35) --
	( 95.95, 74.23) --
	(138.78, 83.45) --
	(224.44, 95.91);

\path[] (224.44, 73.41) --
	(138.78, 70.44) --
	( 95.95, 66.21) --
	( 74.54, 60.90) --
	( 63.83, 56.01) --
	( 58.48, 52.66) --
	( 56.47, 49.28) --
	( 55.80, 48.56) --
	( 55.13, 45.85) --
	( 54.46, 42.83) --
	( 53.79, 34.41) --
	( 53.46, 34.30);
\end{scope}
\begin{scope}
\path[clip] (  0.00,  0.00) rectangle (238.49,115.63);
\definecolor{drawColor}{RGB}{0,0,0}

\path[draw=drawColor,line width= 0.6pt,line join=round] ( 44.91, 30.69) --
	( 44.91,110.13);

\path[draw=drawColor,line width= 0.6pt,line join=round] ( 46.33,107.67) --
	( 44.91,110.13) --
	( 43.49,107.67);
\end{scope}
\begin{scope}
\path[clip] (  0.00,  0.00) rectangle (238.49,115.63);
\definecolor{drawColor}{gray}{0.30}

\node[text=drawColor,anchor=base east,inner sep=0pt, outer sep=0pt, scale=  0.88] at ( 39.96, 32.99) {10};

\node[text=drawColor,anchor=base east,inner sep=0pt, outer sep=0pt, scale=  0.88] at ( 39.96, 54.83) {100};

\node[text=drawColor,anchor=base east,inner sep=0pt, outer sep=0pt, scale=  0.88] at ( 39.96, 76.67) {1000};

\node[text=drawColor,anchor=base east,inner sep=0pt, outer sep=0pt, scale=  0.88] at ( 39.96, 98.52) {10000};
\end{scope}
\begin{scope}
\path[clip] (  0.00,  0.00) rectangle (238.49,115.63);
\definecolor{drawColor}{gray}{0.20}

\path[draw=drawColor,line width= 0.6pt,line join=round] ( 42.16, 36.02) --
	( 44.91, 36.02);

\path[draw=drawColor,line width= 0.6pt,line join=round] ( 42.16, 57.86) --
	( 44.91, 57.86);

\path[draw=drawColor,line width= 0.6pt,line join=round] ( 42.16, 79.70) --
	( 44.91, 79.70);

\path[draw=drawColor,line width= 0.6pt,line join=round] ( 42.16,101.55) --
	( 44.91,101.55);
\end{scope}
\begin{scope}
\path[clip] (  0.00,  0.00) rectangle (238.49,115.63);
\definecolor{drawColor}{RGB}{0,0,0}

\path[draw=drawColor,line width= 0.6pt,line join=round] ( 44.91, 30.69) --
	(232.99, 30.69);

\path[draw=drawColor,line width= 0.6pt,line join=round] (230.53, 29.26) --
	(232.99, 30.69) --
	(230.53, 32.11);
\end{scope}
\begin{scope}
\path[clip] (  0.00,  0.00) rectangle (238.49,115.63);
\definecolor{drawColor}{gray}{0.20}

\path[draw=drawColor,line width= 0.6pt,line join=round] ( 53.12, 27.94) --
	( 53.12, 30.69);

\path[draw=drawColor,line width= 0.6pt,line join=round] ( 94.95, 27.94) --
	( 94.95, 30.69);

\path[draw=drawColor,line width= 0.6pt,line join=round] (136.78, 27.94) --
	(136.78, 30.69);

\path[draw=drawColor,line width= 0.6pt,line join=round] (178.60, 27.94) --
	(178.60, 30.69);

\path[draw=drawColor,line width= 0.6pt,line join=round] (220.43, 27.94) --
	(220.43, 30.69);
\end{scope}
\begin{scope}
\path[clip] (  0.00,  0.00) rectangle (238.49,115.63);
\definecolor{drawColor}{gray}{0.30}

\node[text=drawColor,anchor=base,inner sep=0pt, outer sep=0pt, scale=  0.88] at ( 53.12, 19.68) {0};

\node[text=drawColor,anchor=base,inner sep=0pt, outer sep=0pt, scale=  0.88] at ( 94.95, 19.68) {250};

\node[text=drawColor,anchor=base,inner sep=0pt, outer sep=0pt, scale=  0.88] at (136.78, 19.68) {500};

\node[text=drawColor,anchor=base,inner sep=0pt, outer sep=0pt, scale=  0.88] at (178.60, 19.68) {750};

\node[text=drawColor,anchor=base,inner sep=0pt, outer sep=0pt, scale=  0.88] at (220.43, 19.68) {1000};
\end{scope}
\begin{scope}
\path[clip] (  0.00,  0.00) rectangle (238.49,115.63);
\definecolor{drawColor}{RGB}{0,0,0}

\node[text=drawColor,anchor=base,inner sep=0pt, outer sep=0pt, scale=  1.10] at (138.95,  7.64) {$d$};
\end{scope}
\begin{scope}
\path[clip] (  0.00,  0.00) rectangle (238.49,115.63);
\definecolor{drawColor}{RGB}{0,0,0}

\node[text=drawColor,rotate= 90.00,anchor=base,inner sep=0pt, outer sep=0pt, scale=  1.10] at ( 13.08, 70.41) {time (ms)};
\end{scope}
\begin{scope}
\path[clip] (  0.00,  0.00) rectangle (238.49,115.63);

\path[] ( 40.69, 71.03) rectangle (113.08,125.40);
\end{scope}
\begin{scope}
\path[clip] (  0.00,  0.00) rectangle (238.49,115.63);
\definecolor{drawColor}{RGB}{247,192,26}

\path[draw=drawColor,line width= 0.6pt,line join=round] ( 47.63,112.67) -- ( 59.20,112.67);
\end{scope}
\begin{scope}
\path[clip] (  0.00,  0.00) rectangle (238.49,115.63);
\definecolor{fillColor}{RGB}{247,192,26}

\path[fill=fillColor] ( 53.42,112.67) circle (  1.96);
\end{scope}
\begin{scope}
\path[clip] (  0.00,  0.00) rectangle (238.49,115.63);
\definecolor{drawColor}{RGB}{78,155,133}

\path[draw=drawColor,line width= 0.6pt,dash pattern=on 2pt off 2pt ,line join=round] ( 47.63, 98.22) -- ( 59.20, 98.22);
\end{scope}
\begin{scope}
\path[clip] (  0.00,  0.00) rectangle (238.49,115.63);
\definecolor{fillColor}{RGB}{78,155,133}

\path[fill=fillColor] ( 53.42,101.27) --
	( 56.06, 96.69) --
	( 50.77, 96.69) --
	cycle;
\end{scope}
\begin{scope}
\path[clip] (  0.00,  0.00) rectangle (238.49,115.63);
\definecolor{drawColor}{RGB}{37,122,164}

\path[draw=drawColor,line width= 0.6pt,dash pattern=on 4pt off 2pt ,line join=round] ( 47.63, 83.76) -- ( 59.20, 83.76);
\end{scope}
\begin{scope}
\path[clip] (  0.00,  0.00) rectangle (238.49,115.63);
\definecolor{fillColor}{RGB}{37,122,164}

\path[fill=fillColor] ( 51.45, 81.80) --
	( 55.38, 81.80) --
	( 55.38, 85.72) --
	( 51.45, 85.72) --
	cycle;
\end{scope}
\begin{scope}
\path[clip] (  0.00,  0.00) rectangle (238.49,115.63);
\definecolor{drawColor}{RGB}{0,0,0}

\node[text=drawColor,anchor=base west,inner sep=0pt, outer sep=0pt, scale=  0.80] at ( 66.14,109.91) {LVE (ACP)};
\end{scope}
\begin{scope}
\path[clip] (  0.00,  0.00) rectangle (238.49,115.63);
\definecolor{drawColor}{RGB}{0,0,0}

\node[text=drawColor,anchor=base west,inner sep=0pt, outer sep=0pt, scale=  0.80] at ( 66.14, 95.46) {LVE (CP)};
\end{scope}
\begin{scope}
\path[clip] (  0.00,  0.00) rectangle (238.49,115.63);
\definecolor{drawColor}{RGB}{0,0,0}

\node[text=drawColor,anchor=base west,inner sep=0pt, outer sep=0pt, scale=  0.80] at ( 66.14, 81.01) {VE};
\end{scope}
\end{tikzpicture}

%% file: figures/plot-offline-intra.tex
% Created by tikzDevice version 0.12.5 on 2023-12-14 13:08:43
% !TEX encoding = UTF-8 Unicode
\begin{tikzpicture}[x=1pt,y=1pt]
\definecolor{fillColor}{RGB}{255,255,255}
\path[use as bounding box,fill=fillColor,fill opacity=0.00] (0,0) rectangle (238.49,115.63);
\begin{scope}
\path[clip] (  0.00,  0.00) rectangle (238.49,115.63);
\definecolor{drawColor}{RGB}{255,255,255}
\definecolor{fillColor}{RGB}{255,255,255}

\path[draw=drawColor,line width= 0.6pt,line join=round,line cap=round,fill=fillColor] (  0.00,  0.00) rectangle (238.49,115.63);
\end{scope}
\begin{scope}
\path[clip] ( 31.71, 30.69) rectangle (232.99,110.13);
\definecolor{fillColor}{RGB}{255,255,255}

\path[fill=fillColor] ( 31.71, 30.69) rectangle (232.99,110.13);
\definecolor{drawColor}{RGB}{37,122,164}

\path[draw=drawColor,line width= 0.6pt,line join=round] ( 40.86, 34.30) --
	( 86.61, 35.08) --
	(132.35, 35.41) --
	(178.10, 38.16) --
	(223.84, 49.95);
\definecolor{drawColor}{RGB}{78,155,133}

\path[draw=drawColor,line width= 0.6pt,dash pattern=on 2pt off 2pt ,line join=round] ( 40.86, 35.15) --
	( 86.61, 35.28) --
	(132.35, 37.20) --
	(178.10, 50.81) --
	(223.84,102.85);
\definecolor{drawColor}{RGB}{247,192,26}

\path[draw=drawColor,line width= 0.6pt,dash pattern=on 4pt off 2pt ,line join=round] ( 40.86, 35.15) --
	( 86.61, 35.28) --
	(132.35, 37.61) --
	(178.10, 54.13) --
	(223.84,106.52);
\definecolor{fillColor}{RGB}{247,192,26}

\path[fill=fillColor] ( 84.64, 33.31) --
	( 88.57, 33.31) --
	( 88.57, 37.24) --
	( 84.64, 37.24) --
	cycle;

\path[fill=fillColor] (130.39, 35.65) --
	(134.31, 35.65) --
	(134.31, 39.57) --
	(130.39, 39.57) --
	cycle;

\path[fill=fillColor] (221.88,104.56) --
	(225.80,104.56) --
	(225.80,108.48) --
	(221.88,108.48) --
	cycle;

\path[fill=fillColor] ( 38.90, 33.18) --
	( 42.82, 33.18) --
	( 42.82, 37.11) --
	( 38.90, 37.11) --
	cycle;

\path[fill=fillColor] (176.13, 52.17) --
	(180.06, 52.17) --
	(180.06, 56.09) --
	(176.13, 56.09) --
	cycle;
\definecolor{fillColor}{RGB}{78,155,133}

\path[fill=fillColor] ( 86.61, 38.33) --
	( 89.25, 33.76) --
	( 83.96, 33.76) --
	cycle;

\path[fill=fillColor] (132.35, 40.25) --
	(134.99, 35.67) --
	(129.71, 35.67) --
	cycle;

\path[fill=fillColor] (223.84,105.90) --
	(226.48,101.33) --
	(221.20,101.33) --
	cycle;

\path[fill=fillColor] ( 40.86, 38.20) --
	( 43.50, 33.62) --
	( 38.22, 33.62) --
	cycle;

\path[fill=fillColor] (178.10, 53.86) --
	(180.74, 49.28) --
	(175.45, 49.28) --
	cycle;
\definecolor{fillColor}{RGB}{37,122,164}

\path[fill=fillColor] ( 86.61, 35.08) circle (  1.96);

\path[fill=fillColor] (132.35, 35.41) circle (  1.96);

\path[fill=fillColor] (223.84, 49.95) circle (  1.96);

\path[fill=fillColor] ( 40.86, 34.30) circle (  1.96);

\path[fill=fillColor] (178.10, 38.16) circle (  1.96);
\end{scope}
\begin{scope}
\path[clip] (  0.00,  0.00) rectangle (238.49,115.63);
\definecolor{drawColor}{RGB}{0,0,0}

\path[draw=drawColor,line width= 0.6pt,line join=round] ( 31.71, 30.69) --
	( 31.71,110.13);

\path[draw=drawColor,line width= 0.6pt,line join=round] ( 33.13,107.67) --
	( 31.71,110.13) --
	( 30.29,107.67);
\end{scope}
\begin{scope}
\path[clip] (  0.00,  0.00) rectangle (238.49,115.63);
\definecolor{drawColor}{gray}{0.30}

\node[text=drawColor,anchor=base east,inner sep=0pt, outer sep=0pt, scale=  0.88] at ( 26.76, 31.99) {0};

\node[text=drawColor,anchor=base east,inner sep=0pt, outer sep=0pt, scale=  0.88] at ( 26.76, 51.68) {30};

\node[text=drawColor,anchor=base east,inner sep=0pt, outer sep=0pt, scale=  0.88] at ( 26.76, 71.38) {60};

\node[text=drawColor,anchor=base east,inner sep=0pt, outer sep=0pt, scale=  0.88] at ( 26.76, 91.07) {90};
\end{scope}
\begin{scope}
\path[clip] (  0.00,  0.00) rectangle (238.49,115.63);
\definecolor{drawColor}{gray}{0.20}

\path[draw=drawColor,line width= 0.6pt,line join=round] ( 28.96, 35.02) --
	( 31.71, 35.02);

\path[draw=drawColor,line width= 0.6pt,line join=round] ( 28.96, 54.71) --
	( 31.71, 54.71);

\path[draw=drawColor,line width= 0.6pt,line join=round] ( 28.96, 74.41) --
	( 31.71, 74.41);

\path[draw=drawColor,line width= 0.6pt,line join=round] ( 28.96, 94.10) --
	( 31.71, 94.10);
\end{scope}
\begin{scope}
\path[clip] (  0.00,  0.00) rectangle (238.49,115.63);
\definecolor{drawColor}{RGB}{0,0,0}

\path[draw=drawColor,line width= 0.6pt,line join=round] ( 31.71, 30.69) --
	(232.99, 30.69);

\path[draw=drawColor,line width= 0.6pt,line join=round] (230.53, 29.26) --
	(232.99, 30.69) --
	(230.53, 32.11);
\end{scope}
\begin{scope}
\path[clip] (  0.00,  0.00) rectangle (238.49,115.63);
\definecolor{drawColor}{gray}{0.20}

\path[draw=drawColor,line width= 0.6pt,line join=round] ( 52.30, 27.94) --
	( 52.30, 30.69);

\path[draw=drawColor,line width= 0.6pt,line join=round] (109.48, 27.94) --
	(109.48, 30.69);

\path[draw=drawColor,line width= 0.6pt,line join=round] (166.66, 27.94) --
	(166.66, 30.69);

\path[draw=drawColor,line width= 0.6pt,line join=round] (223.84, 27.94) --
	(223.84, 30.69);
\end{scope}
\begin{scope}
\path[clip] (  0.00,  0.00) rectangle (238.49,115.63);
\definecolor{drawColor}{gray}{0.30}

\node[text=drawColor,anchor=base,inner sep=0pt, outer sep=0pt, scale=  0.88] at ( 52.30, 19.68) {5};

\node[text=drawColor,anchor=base,inner sep=0pt, outer sep=0pt, scale=  0.88] at (109.48, 19.68) {10};

\node[text=drawColor,anchor=base,inner sep=0pt, outer sep=0pt, scale=  0.88] at (166.66, 19.68) {15};

\node[text=drawColor,anchor=base,inner sep=0pt, outer sep=0pt, scale=  0.88] at (223.84, 19.68) {20};
\end{scope}
\begin{scope}
\path[clip] (  0.00,  0.00) rectangle (238.49,115.63);
\definecolor{drawColor}{RGB}{0,0,0}

\node[text=drawColor,anchor=base,inner sep=0pt, outer sep=0pt, scale=  1.10] at (132.35,  7.64) {$d$};
\end{scope}
\begin{scope}
\path[clip] (  0.00,  0.00) rectangle (238.49,115.63);
\definecolor{drawColor}{RGB}{0,0,0}

\node[text=drawColor,rotate= 90.00,anchor=base,inner sep=0pt, outer sep=0pt, scale=  1.10] at ( 13.08, 70.41) {$\alpha$};
\end{scope}
\begin{scope}
\path[clip] (  0.00,  0.00) rectangle (238.49,115.63);

\path[] ( 26.82, 63.09) rectangle ( 76.86,117.45);
\end{scope}
\begin{scope}
\path[clip] (  0.00,  0.00) rectangle (238.49,115.63);
\definecolor{drawColor}{RGB}{37,122,164}

\path[draw=drawColor,line width= 0.6pt,line join=round] ( 33.77,104.72) -- ( 45.33,104.72);
\end{scope}
\begin{scope}
\path[clip] (  0.00,  0.00) rectangle (238.49,115.63);
\definecolor{fillColor}{RGB}{37,122,164}

\path[fill=fillColor] ( 39.55,104.72) circle (  1.96);
\end{scope}
\begin{scope}
\path[clip] (  0.00,  0.00) rectangle (238.49,115.63);
\definecolor{drawColor}{RGB}{78,155,133}

\path[draw=drawColor,line width= 0.6pt,dash pattern=on 2pt off 2pt ,line join=round] ( 33.77, 90.27) -- ( 45.33, 90.27);
\end{scope}
\begin{scope}
\path[clip] (  0.00,  0.00) rectangle (238.49,115.63);
\definecolor{fillColor}{RGB}{78,155,133}

\path[fill=fillColor] ( 39.55, 93.32) --
	( 42.19, 88.74) --
	( 36.91, 88.74) --
	cycle;
\end{scope}
\begin{scope}
\path[clip] (  0.00,  0.00) rectangle (238.49,115.63);
\definecolor{drawColor}{RGB}{247,192,26}

\path[draw=drawColor,line width= 0.6pt,dash pattern=on 4pt off 2pt ,line join=round] ( 33.77, 75.82) -- ( 45.33, 75.82);
\end{scope}
\begin{scope}
\path[clip] (  0.00,  0.00) rectangle (238.49,115.63);
\definecolor{fillColor}{RGB}{247,192,26}

\path[fill=fillColor] ( 37.59, 73.85) --
	( 41.51, 73.85) --
	( 41.51, 77.78) --
	( 37.59, 77.78) --
	cycle;
\end{scope}
\begin{scope}
\path[clip] (  0.00,  0.00) rectangle (238.49,115.63);
\definecolor{drawColor}{RGB}{0,0,0}

\node[text=drawColor,anchor=base west,inner sep=0pt, outer sep=0pt, scale=  0.80] at ( 52.28,101.97) {$k=1$};
\end{scope}
\begin{scope}
\path[clip] (  0.00,  0.00) rectangle (238.49,115.63);
\definecolor{drawColor}{RGB}{0,0,0}

\node[text=drawColor,anchor=base west,inner sep=0pt, outer sep=0pt, scale=  0.80] at ( 52.28, 87.52) {$k=3$};
\end{scope}
\begin{scope}
\path[clip] (  0.00,  0.00) rectangle (238.49,115.63);
\definecolor{drawColor}{RGB}{0,0,0}

\node[text=drawColor,anchor=base west,inner sep=0pt, outer sep=0pt, scale=  0.80] at ( 52.28, 73.06) {$k=7$};
\end{scope}
\end{tikzpicture}

%% file: figures/plot-offline-inter.tex
% Created by tikzDevice version 0.12.5 on 2023-12-14 13:08:43
% !TEX encoding = UTF-8 Unicode
\begin{tikzpicture}[x=1pt,y=1pt]
\definecolor{fillColor}{RGB}{255,255,255}
\path[use as bounding box,fill=fillColor,fill opacity=0.00] (0,0) rectangle (238.49,115.63);
\begin{scope}
\path[clip] (  0.00,  0.00) rectangle (238.49,115.63);
\definecolor{drawColor}{RGB}{255,255,255}
\definecolor{fillColor}{RGB}{255,255,255}

\path[draw=drawColor,line width= 0.6pt,line join=round,line cap=round,fill=fillColor] (  0.00,  0.00) rectangle (238.49,115.63);
\end{scope}
\begin{scope}
\path[clip] ( 31.71, 30.69) rectangle (232.99,110.13);
\definecolor{fillColor}{RGB}{255,255,255}

\path[fill=fillColor] ( 31.71, 30.69) rectangle (232.99,110.13);
\definecolor{drawColor}{RGB}{78,155,133}

\path[draw=drawColor,line width= 0.6pt,line join=round] ( 40.86, 37.36) --
	( 41.59, 39.32) --
	( 43.77, 40.78) --
	( 49.57, 45.12) --
	( 61.19, 40.99) --
	( 84.43, 38.11) --
	(130.90, 36.91) --
	(223.84, 35.79);
\definecolor{drawColor}{RGB}{247,192,26}

\path[draw=drawColor,line width= 0.6pt,dash pattern=on 2pt off 2pt ,line join=round] ( 40.86, 37.90) --
	( 41.59, 39.21) --
	( 43.77, 37.36) --
	( 49.57, 39.58) --
	( 61.19, 35.49) --
	( 84.43, 35.03) --
	(130.90, 35.43) --
	(223.84, 34.30);
\definecolor{drawColor}{RGB}{86,51,94}

\path[draw=drawColor,line width= 0.6pt,dash pattern=on 4pt off 2pt ,line join=round] ( 40.86, 37.53) --
	( 41.59, 68.19) --
	( 43.77, 74.78) --
	( 49.57, 91.18) --
	( 61.19,106.52) --
	( 84.43, 71.16) --
	(130.90, 52.35) --
	(223.84, 47.65);
\definecolor{drawColor}{RGB}{37,122,164}

\path[draw=drawColor,line width= 0.6pt,dash pattern=on 4pt off 4pt ,line join=round] ( 40.86, 47.51) --
	( 41.59, 48.23) --
	( 43.77, 43.94) --
	( 49.57, 50.93) --
	( 61.19, 48.57) --
	( 84.43, 43.65) --
	(130.90, 39.82) --
	(223.84, 39.76);
\definecolor{fillColor}{RGB}{78,155,133}

\path[fill=fillColor] ( 43.77, 40.78) circle (  1.96);

\path[fill=fillColor] ( 41.59, 39.32) circle (  1.96);

\path[fill=fillColor] ( 84.43, 38.11) circle (  1.96);

\path[fill=fillColor] ( 40.86, 37.36) circle (  1.96);

\path[fill=fillColor] (130.90, 36.91) circle (  1.96);

\path[fill=fillColor] ( 61.19, 40.99) circle (  1.96);

\path[fill=fillColor] ( 49.57, 45.12) circle (  1.96);

\path[fill=fillColor] (223.84, 35.79) circle (  1.96);
\definecolor{fillColor}{RGB}{247,192,26}

\path[fill=fillColor] ( 43.77, 40.41) --
	( 46.41, 35.83) --
	( 41.12, 35.83) --
	cycle;

\path[fill=fillColor] ( 41.59, 42.26) --
	( 44.23, 37.68) --
	( 38.94, 37.68) --
	cycle;

\path[fill=fillColor] ( 84.43, 38.08) --
	( 87.07, 33.50) --
	( 81.79, 33.50) --
	cycle;

\path[fill=fillColor] ( 40.86, 40.96) --
	( 43.50, 36.38) --
	( 38.22, 36.38) --
	cycle;

\path[fill=fillColor] (130.90, 38.48) --
	(133.54, 33.91) --
	(128.26, 33.91) --
	cycle;

\path[fill=fillColor] ( 61.19, 38.54) --
	( 63.83, 33.96) --
	( 58.55, 33.96) --
	cycle;

\path[fill=fillColor] ( 49.57, 42.63) --
	( 52.22, 38.05) --
	( 46.93, 38.05) --
	cycle;

\path[fill=fillColor] (223.84, 37.35) --
	(226.48, 32.77) --
	(221.20, 32.77) --
	cycle;
\definecolor{fillColor}{RGB}{86,51,94}

\path[fill=fillColor] ( 41.80, 72.82) --
	( 45.73, 72.82) --
	( 45.73, 76.75) --
	( 41.80, 76.75) --
	cycle;

\path[fill=fillColor] ( 39.63, 66.23) --
	( 43.55, 66.23) --
	( 43.55, 70.15) --
	( 39.63, 70.15) --
	cycle;

\path[fill=fillColor] ( 82.47, 69.20) --
	( 86.39, 69.20) --
	( 86.39, 73.13) --
	( 82.47, 73.13) --
	cycle;

\path[fill=fillColor] ( 38.90, 35.57) --
	( 42.82, 35.57) --
	( 42.82, 39.49) --
	( 38.90, 39.49) --
	cycle;

\path[fill=fillColor] (128.94, 50.38) --
	(132.86, 50.38) --
	(132.86, 54.31) --
	(128.94, 54.31) --
	cycle;

\path[fill=fillColor] ( 59.23,104.56) --
	( 63.15,104.56) --
	( 63.15,108.48) --
	( 59.23,108.48) --
	cycle;

\path[fill=fillColor] ( 47.61, 89.22) --
	( 51.54, 89.22) --
	( 51.54, 93.14) --
	( 47.61, 93.14) --
	cycle;

\path[fill=fillColor] (221.88, 45.69) --
	(225.80, 45.69) --
	(225.80, 49.61) --
	(221.88, 49.61) --
	cycle;

\path[draw=drawColor,line width= 0.4pt,line join=round,line cap=round] ( 40.99, 43.94) -- ( 46.54, 43.94);

\path[draw=drawColor,line width= 0.4pt,line join=round,line cap=round] ( 43.77, 41.17) -- ( 43.77, 46.72);

\path[draw=drawColor,line width= 0.4pt,line join=round,line cap=round] ( 38.81, 48.23) -- ( 44.36, 48.23);

\path[draw=drawColor,line width= 0.4pt,line join=round,line cap=round] ( 41.59, 45.45) -- ( 41.59, 51.00);

\path[draw=drawColor,line width= 0.4pt,line join=round,line cap=round] ( 81.65, 43.65) -- ( 87.20, 43.65);

\path[draw=drawColor,line width= 0.4pt,line join=round,line cap=round] ( 84.43, 40.88) -- ( 84.43, 46.43);

\path[draw=drawColor,line width= 0.4pt,line join=round,line cap=round] ( 38.09, 47.51) -- ( 43.64, 47.51);

\path[draw=drawColor,line width= 0.4pt,line join=round,line cap=round] ( 40.86, 44.74) -- ( 40.86, 50.29);

\path[draw=drawColor,line width= 0.4pt,line join=round,line cap=round] (128.12, 39.82) -- (133.67, 39.82);

\path[draw=drawColor,line width= 0.4pt,line join=round,line cap=round] (130.90, 37.04) -- (130.90, 42.59);

\path[draw=drawColor,line width= 0.4pt,line join=round,line cap=round] ( 58.42, 48.57) -- ( 63.97, 48.57);

\path[draw=drawColor,line width= 0.4pt,line join=round,line cap=round] ( 61.19, 45.79) -- ( 61.19, 51.34);

\path[draw=drawColor,line width= 0.4pt,line join=round,line cap=round] ( 46.80, 50.93) -- ( 52.35, 50.93);

\path[draw=drawColor,line width= 0.4pt,line join=round,line cap=round] ( 49.57, 48.15) -- ( 49.57, 53.70);

\path[draw=drawColor,line width= 0.4pt,line join=round,line cap=round] (221.07, 39.76) -- (226.62, 39.76);

\path[draw=drawColor,line width= 0.4pt,line join=round,line cap=round] (223.84, 36.98) -- (223.84, 42.53);
\end{scope}
\begin{scope}
\path[clip] (  0.00,  0.00) rectangle (238.49,115.63);
\definecolor{drawColor}{RGB}{0,0,0}

\path[draw=drawColor,line width= 0.6pt,line join=round] ( 31.71, 30.69) --
	( 31.71,110.13);

\path[draw=drawColor,line width= 0.6pt,line join=round] ( 33.13,107.67) --
	( 31.71,110.13) --
	( 30.29,107.67);
\end{scope}
\begin{scope}
\path[clip] (  0.00,  0.00) rectangle (238.49,115.63);
\definecolor{drawColor}{gray}{0.30}

\node[text=drawColor,anchor=base east,inner sep=0pt, outer sep=0pt, scale=  0.88] at ( 26.76, 30.18) {0};

\node[text=drawColor,anchor=base east,inner sep=0pt, outer sep=0pt, scale=  0.88] at ( 26.76, 54.75) {20};

\node[text=drawColor,anchor=base east,inner sep=0pt, outer sep=0pt, scale=  0.88] at ( 26.76, 79.32) {40};

\node[text=drawColor,anchor=base east,inner sep=0pt, outer sep=0pt, scale=  0.88] at ( 26.76,103.89) {60};
\end{scope}
\begin{scope}
\path[clip] (  0.00,  0.00) rectangle (238.49,115.63);
\definecolor{drawColor}{gray}{0.20}

\path[draw=drawColor,line width= 0.6pt,line join=round] ( 28.96, 33.21) --
	( 31.71, 33.21);

\path[draw=drawColor,line width= 0.6pt,line join=round] ( 28.96, 57.78) --
	( 31.71, 57.78);

\path[draw=drawColor,line width= 0.6pt,line join=round] ( 28.96, 82.35) --
	( 31.71, 82.35);

\path[draw=drawColor,line width= 0.6pt,line join=round] ( 28.96,106.92) --
	( 31.71,106.92);
\end{scope}
\begin{scope}
\path[clip] (  0.00,  0.00) rectangle (238.49,115.63);
\definecolor{drawColor}{RGB}{0,0,0}

\path[draw=drawColor,line width= 0.6pt,line join=round] ( 31.71, 30.69) --
	(232.99, 30.69);

\path[draw=drawColor,line width= 0.6pt,line join=round] (230.53, 29.26) --
	(232.99, 30.69) --
	(230.53, 32.11);
\end{scope}
\begin{scope}
\path[clip] (  0.00,  0.00) rectangle (238.49,115.63);
\definecolor{drawColor}{gray}{0.20}

\path[draw=drawColor,line width= 0.6pt,line join=round] ( 37.96, 27.94) --
	( 37.96, 30.69);

\path[draw=drawColor,line width= 0.6pt,line join=round] ( 83.34, 27.94) --
	( 83.34, 30.69);

\path[draw=drawColor,line width= 0.6pt,line join=round] (128.72, 27.94) --
	(128.72, 30.69);

\path[draw=drawColor,line width= 0.6pt,line join=round] (174.10, 27.94) --
	(174.10, 30.69);

\path[draw=drawColor,line width= 0.6pt,line join=round] (219.49, 27.94) --
	(219.49, 30.69);
\end{scope}
\begin{scope}
\path[clip] (  0.00,  0.00) rectangle (238.49,115.63);
\definecolor{drawColor}{gray}{0.30}

\node[text=drawColor,anchor=base,inner sep=0pt, outer sep=0pt, scale=  0.88] at ( 37.96, 19.68) {0};

\node[text=drawColor,anchor=base,inner sep=0pt, outer sep=0pt, scale=  0.88] at ( 83.34, 19.68) {250};

\node[text=drawColor,anchor=base,inner sep=0pt, outer sep=0pt, scale=  0.88] at (128.72, 19.68) {500};

\node[text=drawColor,anchor=base,inner sep=0pt, outer sep=0pt, scale=  0.88] at (174.10, 19.68) {750};

\node[text=drawColor,anchor=base,inner sep=0pt, outer sep=0pt, scale=  0.88] at (219.49, 19.68) {1000};
\end{scope}
\begin{scope}
\path[clip] (  0.00,  0.00) rectangle (238.49,115.63);
\definecolor{drawColor}{RGB}{0,0,0}

\node[text=drawColor,anchor=base,inner sep=0pt, outer sep=0pt, scale=  1.10] at (132.35,  7.64) {$d$};
\end{scope}
\begin{scope}
\path[clip] (  0.00,  0.00) rectangle (238.49,115.63);
\definecolor{drawColor}{RGB}{0,0,0}

\node[text=drawColor,rotate= 90.00,anchor=base,inner sep=0pt, outer sep=0pt, scale=  1.10] at ( 13.08, 70.41) {$\alpha$};
\end{scope}
\begin{scope}
\path[clip] (  0.00,  0.00) rectangle (238.49,115.63);

\path[] (174.88, 50.30) rectangle (234.75,119.12);
\end{scope}
\begin{scope}
\path[clip] (  0.00,  0.00) rectangle (238.49,115.63);
\definecolor{drawColor}{RGB}{86,51,94}

\path[draw=drawColor,line width= 0.6pt,dash pattern=on 4pt off 2pt ,line join=round] (181.82,106.39) -- (193.39,106.39);
\end{scope}
\begin{scope}
\path[clip] (  0.00,  0.00) rectangle (238.49,115.63);
\definecolor{fillColor}{RGB}{86,51,94}

\path[fill=fillColor] (185.64,104.43) --
	(189.57,104.43) --
	(189.57,108.35) --
	(185.64,108.35) --
	cycle;
\end{scope}
\begin{scope}
\path[clip] (  0.00,  0.00) rectangle (238.49,115.63);
\definecolor{drawColor}{RGB}{37,122,164}

\path[draw=drawColor,line width= 0.6pt,dash pattern=on 4pt off 4pt ,line join=round] (181.82, 91.94) -- (193.39, 91.94);
\end{scope}
\begin{scope}
\path[clip] (  0.00,  0.00) rectangle (238.49,115.63);
\definecolor{drawColor}{RGB}{37,122,164}

\path[draw=drawColor,line width= 0.4pt,line join=round,line cap=round] (184.83, 91.94) -- (190.38, 91.94);

\path[draw=drawColor,line width= 0.4pt,line join=round,line cap=round] (187.61, 89.16) -- (187.61, 94.71);
\end{scope}
\begin{scope}
\path[clip] (  0.00,  0.00) rectangle (238.49,115.63);
\definecolor{drawColor}{RGB}{78,155,133}

\path[draw=drawColor,line width= 0.6pt,line join=round] (181.82, 77.48) -- (193.39, 77.48);
\end{scope}
\begin{scope}
\path[clip] (  0.00,  0.00) rectangle (238.49,115.63);
\definecolor{fillColor}{RGB}{78,155,133}

\path[fill=fillColor] (187.61, 77.48) circle (  1.96);
\end{scope}
\begin{scope}
\path[clip] (  0.00,  0.00) rectangle (238.49,115.63);
\definecolor{drawColor}{RGB}{247,192,26}

\path[draw=drawColor,line width= 0.6pt,dash pattern=on 2pt off 2pt ,line join=round] (181.82, 63.03) -- (193.39, 63.03);
\end{scope}
\begin{scope}
\path[clip] (  0.00,  0.00) rectangle (238.49,115.63);
\definecolor{fillColor}{RGB}{247,192,26}

\path[fill=fillColor] (187.61, 66.08) --
	(190.25, 61.50) --
	(184.96, 61.50) --
	cycle;
\end{scope}
\begin{scope}
\path[clip] (  0.00,  0.00) rectangle (238.49,115.63);
\definecolor{drawColor}{RGB}{0,0,0}

\node[text=drawColor,anchor=base west,inner sep=0pt, outer sep=0pt, scale=  0.80] at (200.33,103.64) {$p=0.03$};
\end{scope}
\begin{scope}
\path[clip] (  0.00,  0.00) rectangle (238.49,115.63);
\definecolor{drawColor}{RGB}{0,0,0}

\node[text=drawColor,anchor=base west,inner sep=0pt, outer sep=0pt, scale=  0.80] at (200.33, 89.18) {$p=0.05$};
\end{scope}
\begin{scope}
\path[clip] (  0.00,  0.00) rectangle (238.49,115.63);
\definecolor{drawColor}{RGB}{0,0,0}

\node[text=drawColor,anchor=base west,inner sep=0pt, outer sep=0pt, scale=  0.80] at (200.33, 74.73) {$p=0.1$};
\end{scope}
\begin{scope}
\path[clip] (  0.00,  0.00) rectangle (238.49,115.63);
\definecolor{drawColor}{RGB}{0,0,0}

\node[text=drawColor,anchor=base west,inner sep=0pt, outer sep=0pt, scale=  0.80] at (200.33, 60.27) {$p=0.15$};
\end{scope}
\end{tikzpicture}

%% file: figures/cpr_example.tex
\begin{tikzpicture}[label distance=1mm]
	%%% Input model %%%
	\node[circle, draw] (A) {$A$};
	\node[circle, draw] (B) [below = 0.5cm of A] {$B$};
	\node[circle, draw] (C) [below = 0.5cm of B] {$C$};
	\node[circle, draw] (D) [below = 0.5cm of C] {$D$};
	\factor{below right}{A}{0.25cm and 0.5cm}{270}{$\phi_1$}{f1}
	\factor{below right}{B}{0.25cm and 0.5cm}{270}{$\phi_2$}{f2}
	\factor{below right}{C}{0.25cm and 0.5cm}{270}{$\phi_3$}{f3}

	\nodecolorshift{myyellow}{A}{Acol}{-2.1mm}{1mm}
	\nodecolorshift{myyellow}{B}{Bcol}{-2.1mm}{1mm}
	\nodecolorshift{myyellow}{C}{Ccol}{-2.1mm}{1mm}
	\nodecolorshift{myyellow}{D}{Dcol}{-2.1mm}{1mm}

	\factorcolor{myblue}{f1}{f1col}
	\factorcolor{mygreen}{f2}{f2col}
	\factorcolor{myblue}{f3}{f2col}

	\draw (A) -- (f1);
	\draw (B) -- (f1);
	\draw (B) -- (f2);
	\draw (C) -- (f2);
	\draw (C) -- (f3);
	\draw (D) -- (f3);

	%%% Color Passing: Node to factor %%%
	\node[circle, draw, right = 1.75cm of A] (A1) {$A$};
	\node[circle, draw, below = 0.5cm of A1] (B1) {$B$};
	\node[circle, draw, below = 0.5cm of B1] (C1) {$C$};
	\node[circle, draw] (D1) [below = 0.5cm of C1] {$D$};
	\factor{below right}{A1}{0.25cm and 0.5cm}{270}{$\phi_1$}{f1_1}
	\factor{below right}{B1}{0.25cm and 0.5cm}{270}{$\phi_2$}{f2_1}
	\factor{below right}{C1}{0.25cm and 0.5cm}{270}{$\phi_3$}{f3_1}

	\nodecolorshift{myyellow}{A1}{A1col}{-2.1mm}{1mm}
	\nodecolorshift{myyellow}{B1}{B1col}{-2.1mm}{1mm}
	\nodecolorshift{myyellow}{C1}{C1col}{-2.1mm}{1mm}
	\nodecolorshift{myyellow}{D1}{D1col}{-2.1mm}{1mm}

	\factorcolor{myyellow}{f1_1}{f1_1col1}
	\factorcolorshift{myyellow}{f1_1}{f1_1col2}{2.1mm}{0mm}
	\factorcolorshift{myblue}{f1_1}{f1_1col3}{4.2mm}{0mm}
	\factorcolor{myyellow}{f2_1}{f2_1col1}
	\factorcolorshift{myyellow}{f2_1}{f2_1col2}{2.1mm}{0mm}
	\factorcolorshift{mygreen}{f2_1}{f2_1col3}{4.2mm}{0mm}
	\factorcolor{myyellow}{f3_1}{f3_1col1}
	\factorcolorshift{myyellow}{f3_1}{f3_1col2}{2.1mm}{0mm}
	\factorcolorshift{myblue}{f3_1}{f3_1col3}{4.2mm}{0mm}

	\coordinate[right=0.1cm of A1, yshift=-0.1cm] (CA1);
	\coordinate[above=0.2cm of f1_1, yshift=-0.1cm] (Cf1_1);
	\coordinate[right=0.1cm of B1, yshift=0.12cm] (CB1);
	\coordinate[right=0.1cm of B1, yshift=-0.1cm] (CB1_1);
	\coordinate[below=0.2cm of f1_1, yshift=0.15cm] (Cf1_1b);
	\coordinate[above=0.2cm of f2_1, yshift=-0.1cm] (Cf2_1);
	\coordinate[right=0.1cm of C1, yshift=0.12cm] (CC1);
	\coordinate[below=0.2cm of f2_1, yshift=0.15cm] (Cf2_1b);
	\coordinate[right=0.1cm of C1, yshift=-0.1cm] (CC2);
	\coordinate[above=0.2cm of f3_1, yshift=-0.1cm] (Cf3_1);
	\coordinate[right=0.1cm of D1, yshift=0.12cm] (CD1);
	\coordinate[below=0.2cm of f3_1, yshift=0.15cm] (Cf3_1b);

	\begin{pgfonlayer}{bg}
		\draw (A1) -- (f1_1);
		\draw [arc, gray] (CA1) -- (Cf1_1);
		\draw (B1) -- (f1_1);
		\draw [arc, gray] (CB1) -- (Cf1_1b);
		\draw (B1) -- (f2_1);
		\draw [arc, gray] (CB1_1) -- (Cf2_1);
		\draw (C1) -- (f2_1);
		\draw [arc, gray] (CC1) -- (Cf2_1b);
		\draw (C1) -- (f3_1);
		\draw [arc, gray] (CC2) -- (Cf3_1);
		\draw (D1) -- (f3_1);
		\draw [arc, gray] (CD1) -- (Cf3_1b);
	\end{pgfonlayer}

	%%% Color Passing: Recoloring %%%
	\node[circle, draw, right = 1.75cm of A1] (A2) {$A$};
	\node[circle, draw, below = 0.5cm of A2] (B2) {$B$};
	\node[circle, draw, below = 0.5cm of B2] (C2) {$C$};
	\node[circle, draw] (D2) [below = 0.5cm of C2] {$D$};
	\factor{below right}{A2}{0.25cm and 0.5cm}{270}{$\phi_1$}{f1_2}
	\factor{below right}{B2}{0.25cm and 0.5cm}{270}{$\phi_2$}{f2_2}
	\factor{below right}{C2}{0.25cm and 0.5cm}{270}{$\phi_3$}{f3_2}

	\nodecolorshift{myyellow}{A2}{A2col}{-2.1mm}{1mm}
	\nodecolorshift{myyellow}{B2}{B2col}{-2.1mm}{1mm}
	\nodecolorshift{myyellow}{C2}{C2col}{-2.1mm}{1mm}
	\nodecolorshift{myyellow}{D2}{D2col}{-2.1mm}{1mm}

	\factorcolor{myblue}{f1_2}{f1_2col1}
	\factorcolor{mygreen}{f2_2}{f2_2col1}
	\factorcolor{myblue}{f3_2}{f3_2col1}

	\draw (A2) -- (f1_2);
	\draw (B2) -- (f1_2);
	\draw (B2) -- (f2_2);
	\draw (C2) -- (f2_2);
	\draw (C2) -- (f3_2);
	\draw (D2) -- (f3_2);

	%%% Color Passing: Factor to node %%%
	\node[circle, draw, right = 1.75cm of A2] (A3) {$A$};
	\node[circle, draw, below = 0.5cm of A3] (B3) {$B$};
	\node[circle, draw, below = 0.5cm of B3] (C3) {$C$};
	\node[circle, draw] (D3) [below = 0.5cm of C3] {$D$};
	\factor{below right}{A3}{0.25cm and 0.5cm}{270}{$\phi_1$}{f1_3}
	\factor{below right}{B3}{0.25cm and 0.5cm}{270}{$\phi_2$}{f2_3}
	\factor{below right}{C3}{0.25cm and 0.5cm}{270}{$\phi_3$}{f3_3}

	\nodecolorshift{myblue}{A3}{A3col1}{-4.2mm}{1mm}
	\nodecolorshift{myyellow}{A3}{A3col2}{-2.1mm}{1mm}
	\nodecolorshift{myblue}{B3}{B3col1}{-6.3mm}{1mm}
	\nodecolorshift{mygreen}{B3}{B3col2}{-4.2mm}{1mm}
	\nodecolorshift{myyellow}{B3}{B3col3}{-2.1mm}{1mm}
	\nodecolorshift{myblue}{C3}{C3col1}{-6.3mm}{1mm}
	\nodecolorshift{mygreen}{C3}{C3col2}{-4.2mm}{1mm}
	\nodecolorshift{myyellow}{C3}{C3col3}{-2.1mm}{1mm}
	\nodecolorshift{myblue}{D3}{D3col1}{-4.2mm}{1mm}
	\nodecolorshift{myyellow}{D3}{D3col2}{-2.1mm}{1mm}

	\factorcolor{myblue}{f1_3}{f1_3col1}
	\factorcolor{mygreen}{f2_3}{f2_3col1}
	\factorcolor{myblue}{f3_3}{f3_3col1}

	\coordinate[right=0.1cm of A3, yshift=-0.1cm] (CA3);
	\coordinate[above=0.2cm of f1_3, yshift=-0.1cm] (Cf1_3);
	\coordinate[right=0.1cm of B3, yshift=0.12cm] (CB3);
	\coordinate[right=0.1cm of B3, yshift=-0.1cm] (CB1_3);
	\coordinate[below=0.2cm of f1_3, yshift=0.15cm] (Cf1_3b);
	\coordinate[above=0.2cm of f2_3, yshift=-0.1cm] (Cf2_3);
	\coordinate[right=0.1cm of C3, yshift=0.12cm] (CC3);
	\coordinate[below=0.2cm of f2_3, yshift=0.15cm] (Cf2_3b);
	\coordinate[above=0.2cm of f3_3, yshift=-0.1cm] (Cf3_3);
	\coordinate[right=0.1cm of C3, yshift=-0.1cm] (CC1_3);
	\coordinate[below=0.2cm of f3_3, yshift=0.15cm] (Cf3_3b);
	\coordinate[right=0.1cm of D3, yshift=0.12cm] (CD3);

	\begin{pgfonlayer}{bg}
		\draw (A3) -- (f1_3);
		\draw [arc, gray] (Cf1_3) -- (CA3);
		\draw (B3) -- (f1_3);
		\draw [arc, gray] (Cf1_3b) -- (CB3);
		\draw (B3) -- (f2_3);
		\draw [arc, gray] (Cf2_3) -- (CB1_3);
		\draw (C3) -- (f2_3);
		\draw [arc, gray] (Cf2_3b) -- (CC3);
		\draw (C3) -- (f3_3);
		\draw [arc, gray] (Cf3_3) -- (CC1_3);
		\draw (D3) -- (f3_3);
		\draw [arc, gray] (Cf3_3b) -- (CD3);
	\end{pgfonlayer}

	%%% Color Passing: Recoloring %%%
	\node[circle, draw, right = 1.75cm of A3] (A4) {$A$};
	\node[circle, draw, below = 0.5cm of A4] (B4) {$B$};
	\node[circle, draw, below = 0.5cm of B4] (C4) {$C$};
	\node[circle, draw] (D4) [below = 0.5cm of C4] {$D$};
	\factor{below right}{A4}{0.25cm and 0.5cm}{270}{$\phi_1$}{f1_4}
	\factor{below right}{B4}{0.25cm and 0.5cm}{270}{$\phi_2$}{f2_4}
	\factor{below right}{C4}{0.25cm and 0.5cm}{270}{$\phi_3$}{f3_4}

	\nodecolorshift{myyellow}{A4}{A4col}{-2.1mm}{1mm}
	\nodecolorshift{mypurple!75}{B4}{B4col}{-2.1mm}{1mm}
	\nodecolorshift{mypurple!75}{C4}{C4col}{-2.1mm}{1mm}
	\nodecolorshift{myyellow}{D4}{D4col}{-2.1mm}{1mm}

	\factorcolor{myblue}{f1_4}{f1_4col1}
	\factorcolor{mygreen}{f2_4}{f2_4col1}
	\factorcolor{myblue}{f3_4}{f3_4col1}

	\draw (A4) -- (f1_4);
	\draw (B4) -- (f1_4);
	\draw (B4) -- (f2_4);
	\draw (C4) -- (f2_4);
	\draw (C4) -- (f3_4);
	\draw (D4) -- (f3_4);

	%%% Resulting compressed model %%%
	\node[ellipse, inner sep = 1.2pt, draw, right = 2.25cm of B4, yshift=1cm] (AD) {$R(X)$};
	\node[ellipse, inner sep = 1.2pt, draw, below = 0.5cm of AD] (BC) {$S(X)$};
	\pfs{right}{AD}{0.5cm}{270}{$\phi'_1$}{f1pa}{f1p}{f1pb}
	\factor{right}{BC}{0.5cm}{270}{$\phi'_2$}{f2p}

	\begin{pgfonlayer}{bg}
		\draw (AD) -- (f1p);
		\draw (BC) -- (f1p);
		\draw (BC) -- (f2p);
	\end{pgfonlayer}

	%%% Tables %%%
	\node[below = 0.5cm of D2, xshift=-0.5cm] (tab_f2) {
		\begin{tabular}{c|c|c}
			$B$   & $C$   & $\phi_2(B,C)$ \\ \hline
			true  & true  & $\varphi_5$ \\
			true  & false & $\varphi_6$ \\
			false & true  & $\varphi_6$ \\
			false & false & $\varphi_7$ \\
		\end{tabular}
	};

	\node[left = -0.1cm of tab_f2] (tab_f1) {
		\begin{tabular}{c|c|c}
			$A$   & $B$   & $\phi_1(A,B)$ \\ \hline
			true  & true  & $\varphi_1$ \\
			true  & false & $\varphi_2$ \\
			false & true  & $\varphi_3$ \\
			false & false & $\varphi_4$ \\
		\end{tabular}
	};

	\node[right = -0.1cm of tab_f2] (tab_f3) {
		\begin{tabular}{c|c|c}
			$C$   & $D$   & $\phi_3(C,D)$ \\ \hline
			true  & true  & $\varphi_1$ \\
			true  & false & $\varphi_3$ \\
			false & true  & $\varphi_2$ \\
			false & false & $\varphi_4$ \\
		\end{tabular}
	};

	\node[right = -0.1cm of tab_f3] (tab_f12) {
		\begin{tabular}{c|c|c}
			$R(X)$ & $S(X)$ & $\phi'_1(R(X),S(X))$ \\ \hline
			true   & true   & $\varphi_1$ \\
			true   & false  & $\varphi_2$ \\
			false  & true   & $\varphi_3$ \\
			false  & false  & $\varphi_4$ \\
		\end{tabular}
	};

	\node[above = 0.5cm of tab_f12, xshift=0.5cm] (tab_f122) {
		\begin{tabular}{c|c}
			$\#_X[S(X)]$  & $\phi'_2(\#_X[S(X)])$ \\ \hline
			$[2,0]$ & $\varphi_5$ \\
			$[1,1]$ & $\varphi_6$ \\
			$[0,2]$ & $\varphi_7$ \\
		\end{tabular}
	};
\end{tikzpicture}